\documentclass{article}

\usepackage[utf8]{inputenc} % allow utf-8 input
\usepackage[T1]{fontenc}    % use 8-bit T1 fonts
\usepackage{pifont}

\usepackage{multicol,lipsum,xcolor}
\usepackage{hyperref}       % hyperlinks
\usepackage{url}            % simple URL typesetting
\usepackage{amsmath,amsthm,amssymb,bm}
\usepackage{bbold}
\usepackage{comment}
\DeclareMathOperator{\argmin}{argmin}  
\DeclareMathOperator{\argmax}{argmax}
\usepackage{algorithm,algorithmic}
\usepackage{caption}
\usepackage{graphicx}
\usepackage{tikz}
\usepackage{enumerate}
\usepackage{enumitem}
\usepackage{natbib}
\usepackage{parskip}
\usepackage{comment}
\usepackage{microtype}
\usepackage{diagbox}
\usepackage{tikz}
\usepackage[utf8]{inputenc} % allow utf-8 input
\usepackage[T1]{fontenc}    % use 8-bit T1 fonts
\usepackage{pifont}
\usepackage{accents}
\usepackage{dsfont}
\allowdisplaybreaks
\usepackage{thmtools}
\usepackage{thm-restate}
\usepackage{cleveref}
\usepackage{bbm}
\usepackage{authblk}
\usepackage{fancyhdr}

\newtheorem{thm}{Theorem}
\newtheorem*{thm*}{Theorem}
\newtheorem{corollary}{Corollary}

\newtheorem{lemma}{Lemma}

\theoremstyle{definition}

\newtheorem{assumption}{Assumption}
\newtheorem{remark}{Remark}

\newcount\Comments  % 0 suppresses notes to selves in text
\definecolor{darkred}{rgb}{0.7,0,0}
\definecolor{teal}{rgb}{0.3,0.8,0.8}
\definecolor{forestgreen}{rgb}{0.13, 0.55, 0.13}

\newcommand{\prob}[0]{\mathbb{P}}
\newcommand{\E}[0]{\mathbb{E}}

\newcommand{\F}[0]{\mathcal{F}}

\newcommand{\B}[0]{\mathcal{B}}
\newcommand{\C}[0]{\mathcal{C}}

\newcommand{\G}[0]{\mathcal{G}}

\newcommand{\R}[0]{\mathcal{R}}

\newcommand{\N}[0]{\mathcal{N}}
\newcommand{\Real}[0]{\mathbb{R}}
\newcommand{\gS}{\mathcal{S}}
\newcommand{\gA}{\mathcal{A}}
\newcommand{\Sph}[0]{\mathbb{S}}

\newcommand{\ind}[0]{\mathds{1}}
\newcommand\numberthis{\addtocounter{equation}{1}\tag{\theequation}}
\newcommand{\HBM}[0]{\hat{\textbf{M}}}
\newcommand{\HBD}[0]{\hat{\textbf{D}}}

\DeclareMathOperator{\dist}{dist}
\DeclareMathOperator{\simil}{sim}

\DeclareMathOperator{\mathspan}{span}

\DeclareMathOperator{\Mix}{Mix}
\DeclareMathOperator{\rowspan}{rowspan}

\DeclareMathOperator{\Freq}{Freq}
\usepackage{fullpage}

\title{Learning Mixtures of Markov Chains and MDPs}

\author[1]{Chinmaya Kausik}
\author[2]{Kevin Tan}
\author[2]{Ambuj Tewari}
\affil[1]{Department of Mathematics, University of Michigan, Ann Arbor, MI}
\affil[2]{Department of Statistics, University of Michigan, Ann Arbor, MI}
\date{}                     %% if you don't need date to appear
\begin{document}

\maketitle

\begin{abstract}
We present an algorithm for learning mixtures of Markov chains and Markov decision processes (MDPs) from short unlabeled trajectories. Specifically, our method handles mixtures of Markov chains with optional control input by going through a multi-step process, involving (1) a subspace estimation step, (2) spectral clustering of trajectories using "pairwise distance estimators," along with refinement using the EM algorithm, (3) a model estimation step, and (4) a classification step for predicting labels of new trajectories. We provide end-to-end performance guarantees, where we only explicitly require the length of trajectories to be linear in the number of states and the number of trajectories to be linear in a mixing time parameter. Experimental results support these guarantees, where we attain 96.6\% average accuracy on a mixture of two MDPs in gridworld, outperforming the EM algorithm with random initialization (73.2\% average accuracy).
\end{abstract}

%\kevin{Why didn't chen+poor handle control input in the same way we do? Free paper idea?}

%TODO:
%\begin{itemize}
%    \item Theoretical guarantees for model estimation and classification
%    \item Clean up the clustering and subspace guarantees to get a rough O tilde
%    \item Experiments for occupancy measure + transition structure together
%    \item Markov Chain experiments
%    \item Write the intro and related work sections using what we fleshed out in the comments inside the sections
%    \item Describe the algorithms in a way similar to Chen and Poor
%    \item Produce figures, describe the experiments
%    \item Flesh out and write the discussion section
%    \item Background and problem setup
%    \item Experiments for other confounded gridworld situations with either fewer sink states or identical sink states across confounders to make it more important to use the transition structure. \cknote{I can try to hard code such MDPs and policies}
%\end{itemize}

\section{Introduction}\label{sec:introduction}

Efficiently clustering a mixture of time series data, especially with access to only short trajectories, is a problem that pervades sequential decision making and prediction (\citet{timeseriesclusteringsurvey2005}, \citet{timeseriescoresetsneurips2021}, \citet{maharajtimeseries2000}). This is motivated by various real-world problems, ranging through psychology (\citet{psychVAR2016}), economics (\citet{mcculloch1994econtimeseries}), automobile sensing (\citet{toeplitztimeseries2017}), biology (\citet{wonglibiology2000}), neuroscience (\citet{albert1991seizure}), to name a few. One natural and important time series model is that of a mixture of $K$ MDPs, which includes the case of a mixture of $K$ Markov chains. We want to cluster from a set of short trajectories where (1) one does not know which MDP or Markov chain any trajectory comes from and (2) one does not know the transition structures of any of the $K$ MDPs or Markov chains. Previous literature like \citet{mannor2021latent} and \citet{gupta2016mix} has stated and underlined the importance of this problem, but so far, the literature on methods to solve it with theoretical guarantees and empirical results has been sparse. 

Broadly, there are three threads of literature on problems related to ours. Within reinforcement learning literature, there has been a sustained interest in frameworks very similar to mixtures of MDPs -- latent MDPs (\citet{mannor2021latent}), multi-task RL (\citet{multitask2013}), hidden model MDPs (\citet{hmMDP2021}), to name a few. However, most effort in this thread has been towards regret minimization in the online setting, where the agent interacts with an MDP from a set of unknown MDPs. The framework of latent MDPs in \citet{mannor2021latent} is equivalent to adding reward information to ours. They have shown that one can only learn latent MDPs online with number of episodes required polynomial in states and actions to the power of trajectory length (under a reachability assumption similar to our mixing time assumption). On the other hand, our method learns latent MDPs offline with number of episodes needed only linear in the number of states (in no small part due to the subspace estimation step we make). 

The other thread of literature deals with using a "subspace estimation" idea to efficiently cluster mixture models, from which we gain inspiration for our algorithm. \citet{vempala2004spectral} first introduce the idea of using subspace estimation and clustering steps, with application to learning mixtures of Gaussians. \citet{kakade2020meta} adapt these ideas to the setting of meta-learning for mixed linear regression, adding a classification step. \citet{poor2022mixdyn} bring these ideas to the time-series setting to learn mixtures of linear dynamical systems. They leave open the problems of (1) adapting the method to handle control inputs (mentioning mixtures of MDPs as an important example) and (2) handling other time series models (like autoregressive models and Markov chains), and state that the former is of great importance. There are many technical and algorithmic subtleties in adapting the ideas developed so far to MDPs and Markov Chains. The most obvious one comes from the following observation: in linear dynamical systems, the deviation from the predicted next-state value under the linear model occurs with additive i.i.d.\ noise. In MDPs and Markov chains, we are \textit{sampling} from the next-state probability simplex at each timestep, and this cannot be cast as a deterministic function of the current state with additive i.i.d.\ noise.

%\cknote{Address lower bound, mention reachability of states assumption, discuss specifics with Ambuj, tread carefully.} 

\citet{gupta2016mix} also provide a method for learning a mixture of Markov chains using only 3-trails, and compare its performance to the EM algorithm. While the requirement on trajectory length is as lax as can be, their method needs to estimate the distribution of 3-trails using all available data, incurring an estimation error in estimating $S^3A^3$ parameters, while providing no finite-sample theoretical guarantees. If the method can be shown to enjoy finite sample guarantees, the need to estimate $S^3A^3$ parameters indicates that the guarantees will scale poorly with $S$ and $A$. 

The problem that we aim to solve is the following.

\textit{Is there a method with finite-sample guarantees that can learn both mixtures of Markov chains and MDPs offline, with only data on trajectories and the number of elements in the mixture $K$?}

\subsection{Summary of Contributions}

We provide such a method, with trajectory length requirements free from an $S,A$ dependence. The method performs (1) subspace estimation, (2) spectral clustering, an optional step of using clusters to initialize the EM algorithm, (3) estimating models, and finally (4) classifying future trajectories.

\begin{thm*}[Informal]
Ignoring logarithmic terms, we can recover all labels exactly with $K^2S$ trajectories of length $K^{3/2} t_{mix}$, up to logarithmic terms and instance-dependent constants characterizing the models but not explicitly dependent on $S, A, t_{mix}$ or $K$.
\end{thm*}
Other contributions include:
\begin{itemize}
    \item This is the first method, to our knowledge, that can cluster MDPs with finite-sample guarantees where the length of trajectories does not depend explicitly on $S,A$. The length only explicitly depends linearly on the mixing time $t_{mix}$, and the number of trajectories only explicitly depends linearly on $S$.%Learning mixtures of Markov chains is a nontrivial and important problem in its own right, and our method also successfully handles the case with control input (MDPs). We extend \citet{poor2022mixdyn}'s work in time series to the case of Markov chains (and therefore autoregressive models), and introduce the first method in the \citet{vempala2004spectral} line of literature capable of handling control input with our work on MDPs. This opens the door for further advancements along these lines in reinforcement learning (RL) and control, like the linear-quadratic regulator (LQR), MDPs with infinite-state spaces, etc. %While methods for learning mixtures of Markov chains can be trivially applied, doing so requires using transitions between augmented state space of dimension SA x SA, instead of using the natural state-action-nextstate transitions. 
    %\cknote{Shorten to just mention that it's the first practical clustering result on MDPs.}
    \item We are able to provide theoretical guarantees while making no explicit demands on the policies and rewards used to collect the data, only relying on a difference in the transition structures at frequently occurring $(s,a)$ pairs.
    \item \citet{poor2022mixdyn} work under deterministic transitions with i.i.d. additive Gaussian noise, and we need to bring in non-trivial tools to analyse systems like ours, determined by transitions with non-i.i.d. additive noise. Our use of the blocking technique of \citet{bin1994mixing} opens the door for the analysis of such systems.
    %The nontrivial analysis of our method, requiring a result from \citet{bin1994mixing}, opens the door for theoretical analyses of extensions to stochastic time series models \cknote{with more complicated noise than Gaussian, like ours is martingale noise. Mention Chen and Poor's limitation to iid noise, while this opens doors to non iid noise.}, potentially with control input.
    \item Empirical results in our experiments show that our method outperforms, outperforming the EM algorithm by a significant margin (73.2\% for soft EM and 96.6\% for us on gridworld). 
    
    %\kevin{for a fair comparison with Gupta et. al., only use next state transitions}
\end{itemize}

\section{Background and Problem Setup}

%NOTATION
We work in the scenario where we have $K$ unknown models, either $K$ Markov chains or $K$ MDPs, and data of $N_{traj}$ trajectories collected offline. Throughout the rest of the paper, we work with the case of MDPs, as we can think of Markov chains as an MDP where there is only one action ($A = \{*\}$) and rewards are ignored by our algorithm anyway. 

We have a tuple $(\gS, \gA, \{\prob_k\}_{k=1}^K, \{f_k\}_{k=1}^K, p_k)$ describing our mixture. Here, $\gS, \gA$ are the state and action sets respectively. $\prob_k(s' \mid s, a)$ describes the probability of an $s,a,s'$ transition under label $k$. At the start of each trajectory, we draw $k \sim \text{Categorical}(f_1,...,f_K)$, and starting state according to $p_k$, and generate the rest of the trajectory under policies $\pi_k(a \mid s)$. We have stationary distributions on the state-action pairs $d_k(s, a)$ for $\pi_k$ interacting with $\prob_k$. We do not know (1) the parameters $\prob_k, f_k, p_k, \pi_k(\cdot \mid s)$ of each model or the policies, and (2) $k$, i.e., which model each trajectory comes from.

This coincides with the setup in \citet{gupta2016mix} in the case of Markov chains ($|\gA|=1$). It also overlaps with the setup of learning latent MDPs offline, in the case of MDPs. However, one difference is that we make no assumptions about the reward structure -- once trajectories are clustered, we can learn the models, including the rewards. It is also possible to learn the rewards with a term in the distance measure that is alike to the model separation term. However, this would require extra assumptions on reward separation that are not necessary for clustering.

\begin{assumption}[Mixing]\label{as:mixing}
The $K$ Markov chains on $\mathcal{S} \times \mathcal{A}$ induced by the behaviour policies $\pi_k$, each achieve mixing to a stationary distribution $d_{k}(s,a)$ with mixing time $t_{mix, k}$. Define the overall mixing time of the mixture of MDPs to be $t_{mix} := \max_k t_{mix, k}$.
\end{assumption}

\begin{assumption}[Model Separation]\label{as:model_sep}
There exist $\alpha, \Delta$ so that for each pair $k_1, k_2$ of hidden labels, there exists a state action pair $(s,a)$ (possibly depending on $k_1, k_2$) so that $d_{k_1}(s,a), d_{k_2}(s,a) \geq \alpha$ and $\|\prob_{k_1}(\cdot \mid s,a) - \prob_{k_2}(\cdot \mid s,a)\|_2 \geq \Delta$.
\end{assumption}

Assumption~\ref{as:model_sep} is merely saying that for any pair of labels, at least one visible state action pair witnesses a model difference $\Delta$. Call this the separating state-action pair. If no visible pair witnesses a model difference between the labels, then one certainly cannot hope to distinguish them using trajectories.

\begin{remark}\label{rem:no-policies-rewards} \textbf{Why is there no assumption about policies?}
Notice that we make no explicit assumptions about policies. The nature of our algorithm allows us to work with the transition structure directly, and so we only demand that we observe a state action pair that witnesses a difference in transition structures. The policy is implicitly involved in this assumption through the stationary distribution $d_k(s,a)$ it induces, but our results demonstrate that this is the minimal demand we need to make in relation to the policies.
\end{remark}

Additionally, Assumption~\ref{as:mixing}, which establishes the existence of a mixing time, is not a strong assumption (outside of the implicit hope that $t_{mix}$ is small). This is because any irreducible aperiodic finite state space Markov chain mixes to a unique stationary distribution. If the Markov chain is not irreducible, it mixes to a unique distribution determined by the irreducible component of the starting distribution.

The only requirement is thus aperiodicity, which is also technically superficial, as we now clarify. If the induced Markov chains were periodic with period $L$, we would have a finite set of stationary distributions $d_{u, l}(s,a)$ that the chain would cycle through over a single period, indexed by $l = 1 \to L$. One can follow the proofs to verify that the guarantees continue to hold if we modify $\alpha$ in Assumption~\ref{as:model_sep} to be a lower bound for $\min_{i,l} d_{u_i, l}(s,a)$ instead of just $\min_i d_{u_i}(s,a)$.
\section{Algorithm}

\subsection{Setup and Notation}
We have short trajectories of length $T_n$, divided into 4 segments of equal length. We call the second and fourth segment $\Omega_1$ and $\Omega_2$ respectively. We further sub-divide $\Omega_i$ into $G$ blocks, and focus only on the first state-action observation in each sub-block and its transition (discard all other observations). We often refer to these observations as "single-step sub-blocks." See Figure~\ref{fig:blockestim} for an illustration of this. Divide the set of trajectory indices into two sets and call them $\N_{sub}$ and $\N_{clust}$ (for subspace estimation and clustering). Denote their sizes by $N_{sub}$ and $N_{clust}$ respectively. Let $\N_{traj}(s,a)$ be the set of trajectory indices where $(s,a)$ is observed in both $\Omega_1$ and $\Omega_2$. Let $N_{traj}(s,a)$ be the size of this set. Denote by $N(n,i,s,a)$ the number of times $(s,a)$ is recorded in segment $i$ of trajectory $n$, and let $\textbf{N}(n,i,s,a,\cdot)$ be the vector of next-state counts. We denote by $\prob_k(\cdot \mid s,a)$ the vector of next state transition probabilities. We denote by $\Freq_\beta$ the set of all state action pairs whose occurrence frequency in our observations is higher than $\beta$.

We will call the predicted clusters returned by the clustering algorithm $\C_k$. For model estimation and classification, we do not use segments, and merely split the entire trajectory into $G$ blocks, discarding all but the last observation in each block. We call this observation the corresponding single-step sub-block. We denote the total count of $s,a$ observations in trajectory $n$ by $N(n,s,a)$ and that of $s',s,a$ triples by $N(n,s,a,s')$. 

In practice, we choose to not be wasteful and observations are not discarded while computing the transition probability estimates. To clarify, in that case $N(n,i,s,a)$ is just the count of $(s,a)$ in segment $i$ and similarly for $\textbf{N}(n,i,s,a,\cdot), N(n,s,a)$ and $\textbf{N}(n,s,a, \cdot)$. Estimators in both cases, that is both with and without discarding observations, are MLE estimates of the transition probabilities. One of them maximizes the likelihood of just the single-step sub-blocks and the other maximizes the likelihood of the entire segment. We need the latter for good finite-sample guarantees (using mixing). However, the former satisfies asymptotic normality, which is not enough for finite-sample guarantees, but it often makes it a good and less wasteful estimator in practice.
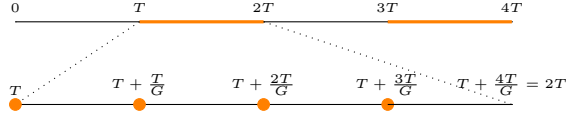
\begin{figure}[h]
    \centering
    \begin{tikzpicture}[scale = 1.1][thick]
        \draw(0,0)--(6,0);
        \foreach \x/\xtext in {0/0,1.5/$T$,3/$2T$,4.5/$3T$,6/$4T$}
          \draw(\x,0)--(\x,0) node[above] {\tiny \xtext};
        %\foreach \x in {3.5,5.5,9.5,11.5}
        %  \draw(\x,-3pt) node[below] {$c_1$};
        %\foreach \x in {4.5,6,10.5,12}
        %\draw[very thick, orange] (\x-0.75,0)--(\x,0);
        \draw[very thick, orange] (1.5,0)--(3,0);
        \draw[very thick, orange] (4.5,0)--(6,0);
        
        \draw(0,-1)--(6,-1);
        \draw[dotted](1.5,0)--(0,-1);
        \draw[dotted](3,0)--(6,-1);
        \foreach \x/\xtext in {0/$T$,1.5/$T+\frac{T}{G}$,
                                3/$T+\frac{2T}{G}$, 4.5/$T+\frac{3T}{G}$, 
                                6/$T+\frac{4T}{G}=2T$}
          \draw(\x,-1)--(\x,-1) node[above] {\tiny \xtext};
        \foreach \x in {0,1.5,3,4.5}
          \filldraw[orange](\x, -1) circle (2pt);
        % \draw(5.25,-1) node[below] {\tiny Length $1$};
        \draw(4.5,-1)--(6,-1);
        
    \end{tikzpicture}
     \caption{Breaking up a trajectory into $4$ segments and $G$ blocks per segment ($G=4$) for the single-step estimator. Observations are only recorded at the orange points.}%
     \label{fig:blockestim}
\end{figure}
\subsection{Overview}

The algorithm amounts to (1) a PCA-like subspace estimation step, (2) spectral clustering of trajectories using "thresholded pairwise distance estimates," along with an optional step of using clusters to initialize the EM algorithm, (3) estimating models (MDP transition probailities) and finally (4) classifying any trajectories not in $\N_{clust}$ (for example, $\N_{sub}$). We provide performance guarantees for each step of the algorithm in section~\ref{sec:analysis}.

\subsection{Subspace Estimation}

The aim of this algorithm is to estimate for each $(s,a)$ pair a matrix $\textbf{V}_{s,a}$ satisfying $\rowspan \textbf{V}_{s,a}^T \approx \mathspan (\prob_k(\cdot | s,a))_{k=1,..,K}$. That is, we want to obtain an orthogonal projector to the subspace spanned by the next-state distributions $\prob_k(\cdot | s,a)$ for $1 \leq k \leq K$. 

Summarizing the algorithm in natural language, we perform subspace estimation via 3 steps. We first estimate the next state distribution given state and action for each trajectory. We then obtain the outer product of the next state distributions thus estimated. These outer product matrices are averaged over trajectories, and the average is used to find the orthogonal projectors $V_{s,a}^T$ to the top K eigenvectors. 

\begin{algorithm}
\centering
	\caption{Subspace Estimation}
    \label{alg:subspace-est}
	\begin{algorithmic}[1]
		\STATE Compute $N_{traj}(s,a)$ for all $s,a$. Initialize the $S\times S$ matrix $\HBM_{s,a} \gets 0$ and the $SA \times SA$ matrix $\HBD \gets 0$.
		\STATE $\hat{\textbf{d}}_{n,1}, \hat{\textbf{d}}_{n,2} \gets \textbf{0} \in \Real^{SA}$ for all $n \in \N_{sub}$
		\FOR{$(i,s,a) \in \{1,2\} \times S \times A$}
		\STATE Compute $\textbf{N}(n,i,s,a, \cdot)$, $N(n,i,s,a),$ $\;\; \forall n \in \N_{sub}$
		\STATE $\hat{\prob}_{n,i}(\cdot \vert s,a) \gets \frac{\textbf{N}(n,i,s,a,\cdot)}{N(n,i,s,a)}\ind_{N(n,i,s,a) \neq 0},$  $\;\; \forall n$
		\STATE $[\hat{\textbf{d}}_{n,i}]_{s,a} \gets \frac{N(n,i,s,a)}{G},$ $\;\; \forall n$
		\STATE $\HBM_{s,a} \gets \HBM_{s,a} + \sum_{n \in \N_{sub}}\frac{\hat{\prob}_{n,1}(\cdot \mid s,a)\hat{\prob}_{n,2}(\cdot \mid s,a)^T}{N_{traj}(s,a)}$
		\ENDFOR
		\STATE $\HBD \gets \HBD + \sum_{n \in \N_{sub}} \frac{1}{N_{sub}}\hat{\textbf{d}}_{n,1}\hat{\textbf{d}}_{n,2}^T$
		\STATE Using SVD, return the orthogonal projectors $(\textbf{V}_{s,a}^T)_{K\times S}$ to the top $K$ eigenspaces of $\HBM_{s,a} + \HBM_{s,a}^T$ for each $(s,a)$ where $N_{traj}(s,a) \neq 0$ (set the others to $0$), along with the orthogonal projector $(\textbf{U}^T)_{K \times SA}$ to the top $K$ eigenspace of $\HBD + \HBD^T$.
	\end{algorithmic}
\end{algorithm}

\begin{remark}
\textbf{Why do we split the trajectories?}
We use two approximately independent segments $\Omega_1$ and $\Omega_2$ time separated by a multiple of the mixing time $t_{mix}$ to estimate the next state distributions. The reduced correlation between the two estimates obtained allows us to give theoretical guarantees for concentration, despite using dependent data within each trajectory $n$ in the estimation of the rank $1$ matrices $(\prob_{k_n}(\cdot |s,a))(\prob_{k_n}(\cdot |s,a))^T$. The key point is that the double estimator $\hat{\prob}_{n,1}(\cdot \mid s,a)\hat{\prob}_{n,2}(\cdot \mid s,a)$ is in expectation very close to this matrix. 

Notice that our estimator $\HBM_{s,a}$ is in expectation then given approximately by $\sum_{k=1}^K f_k (\prob_k(\cdot |s,a))(\prob_k(\cdot |s,a))^T$. The eigenspace of this matrix is clearly $\mathspan (\prob_k(\cdot | s,a))_{k=1,..,K}$. The deviation from the expectation is controlled by the total number of trajectories, while the "approximation error" separating the expectation from the desired matrix is controlled by the separation between $\Omega_1$ and $\Omega_2$.
\end{remark} 
\begin{remark} \textbf{Why is this not PCA?}
This procedure has many linear-algebraic similarities to uncentered PCA on the dataset of (trajectories, next state frequencies), but statistically has a very different target. Crucially, (centered) PCA is concerned with the variance $\E[X^TX]$, while we are interested in a decent estimate of the target $\E[X^T]\E[X]$ above and thus use a double estimator. Our theoretical analysis also has nothing to do with analyses of PCA due to this difference in the statistical target.
\end{remark}

\subsection{Clustering}

Using the subspace estimation algorithm's output, we can embed estimates from trajectories in a low dimensional subspace. For the clustering algorithm, we aim to compute the pairwise distances of these estimates from trajectories in this embedding. A double estimator is used yet again, to reduce the covariance between the two terms in the inner product used to compute such a distance.

This projection is crucial because it reduces the variance of the pairwise distance estimators from a dependence on $SA$ to a dependence on $K$. This is the intuition for how we can shift the onus of good clustering from being heavily dependent on the length of trajectories to being more dependent on the subspace estimate and thus on the number of trajectories.

There are many ways to use such "pairwise distance estimates" for clustering trajectories. In one successful example, we use a test: if the squared distances are below some threshold (details provided later), then we can conclude that they come from the same element of the mixture, and different ones otherwise. This allows us to construct (the adjacency matrix of) a graph with vertices as trajectories, and we can feed the results into a clustering algorithm like spectral clustering. Alternatively, one can use other graph partitioning methods or agglomerative methods on the distance estimates themselves.

\begin{algorithm}
\centering
	\caption{Clustering}
    \label{alg:clustering}
	\begin{algorithmic}[1]
		\STATE Compute the set $\Freq_\beta$ by picking $(s,a)$ pairs with occurrence more than $\beta$.
		\STATE $\textbf{d}_{n,1}, \textbf{d}_{n,2} \gets \textbf{0} \in \Real^{SA}$
		\FOR{$(i,s,a) \in \{1,2\} \times S \times A$}
		\STATE Compute $\textbf{N}(n,i,s,a, \cdot)$, $N(n,i,s,a),$ $\;\; \forall n \in \N_{clust}$
		\STATE $\hat{\prob}_{n,i}(\cdot \vert s,a) \gets \frac{\textbf{N}(n,i,s,a,\cdot)}{N(n,i,s,a)}\ind_{N(n,i,s,a) \neq 0},$  $\;\; \forall n$
		\STATE $[\hat{\textbf{d}}_{n,i}]_{s,a} \gets \frac{N(n,i,s,a)}{G},$ $\;\; \forall n$
		\ENDFOR
        \FOR{$(n,m) \in \N_{clust} \times \N_{clust}$}
        \FOR{$(i,s,a) \in \{1,2\} \times S \times A$}
        \STATE $\hat{\bm{\Delta}}_{i,s,a} := \textbf{V}^T_{s,a}(\hat{\prob}_{n,i}(\cdot \mid s,a) - \hat{\prob}_{m,i}(\cdot \vert s,a))$
        \ENDFOR
        \STATE $\dist_1(n,m) := \max_{(s,a) \in \Freq_\beta} \hat{\bm{\Delta}}_{1,s,a}^T \hat{\bm{\Delta}}_{2,s,a}$
        \STATE $\dist_2(n,m) := (\hat{\textbf{d}}_{n,1} - \hat{\textbf{d}}_{m,1})^T\textbf{U}\textbf{U}^T(\hat{\textbf{d}}_{n,2} - \hat{\textbf{d}}_{m,2})$
        \STATE $\dist(n,m) := \lambda\dist_1(n,m) + (1-\lambda)\dist_2(n,m)$
        \ENDFOR
        \STATE Plot a histogram of $\dist$ to determine threshold $\tau$ and cluster trajectories $\simil(n,m) := \ind_{\dist(n,m) \leq \tau}$
	\end{algorithmic}
\end{algorithm}

Choosing $\beta$, $\lambda$ and the threshold $\tau$ both involve heuristic choices, much like how choosing the threshold in \citet{poor2022mixdyn} needs heuristics, although our methods are very different. We describe our methods in more detail in Section~\ref{sec:in-practice}.

\subsubsection{Refinement using EM}\label{sssec:em-algorithm}

Our guarantees in section~\ref{sec:analysis} will show that we can recover exact clusters with high probability at the end of algorithm~\ref{alg:clustering}. However, in practice, it makes sense to refine the clusters if trajectories are not long enough for exact clustering. Remember that an instance of the EM algorithm for any model is specified by choosing the observations $Y$, the hidden variables $Z$ and the parameters $\theta$.  

If we consider observations to be next-state transitions from $(s,a) \in \Freq_\beta$, hidden variables to be the hidden labels and the parameters $\theta$ to include both next-state transition probabilities for $(s,a) \in \Freq_\beta$ and cluster weights $\hat{f}_k$, then one can now refine the clusters using the EM algorithm on this setup, which enjoys monotonicity guarantees in log-likelihood if one uses soft EM. The details of the EM algorithm are quite straightforward, described in Appendix~\ref{sec:em-algorithm-details}. 
 
We hope that this is a step towards unifying the discussion on spectral and EM methods for learning mixture models, highlighting that we need not choose between one or the other -- spectral methods can initialize the EM algorithm, in one reinterpretation of the refinement step.

Note that refinement using EM is not unique to our algorithm. The model estimation and classification steps in \citet{kakade2020meta} (under the special case of Gaussian noise) and \citet{poor2022mixdyn} (who already assume Gaussian noise) are exactly the E-step and M-step of the hard EM algorithm as well.

\subsection{Model Estimation and Classification}

Given clusters from the clustering and refinement step, $2$ tasks remain, namely those of estimating the models from them and correctly classifying any future trajectories. We can estimate the models exactly as in the M-step of hard EM. 
\begin{align*}
    \hat{\prob}_k(s' \vert s,a) &\gets \frac{\sum_{n \in \C_k} N(n,s,a,s')}{\sum_{n \in\C_k} N(n,s,a)}\\
    \hat{f}_k &\gets \frac{|\C_k|}{N_{clust}}
\end{align*}
For classification, given a set $\N_{class}$ of trajectories with size $N_{class}$ generated independently of $\N_{clust}$, we can run a process very similar to Algorithm~\ref{alg:clustering} to identify which cluster to assign each new trajectory to. It is worth noting that we can run the classification step on the subspace estimation dataset itself and recover true labels for those trajectories, since trajectories in $\N_{sub}$ and $\N_{clust}$ are independent. 

We describe the algorithm in natural language here. The algorithm is presented formally as Algorithm~\ref{alg:classification} in Appendix~\ref{sec:classification-algorithm}. We first compute an orthogonal projector $\Tilde{\textbf{V}}_{s,a}$ to the subspace spanned by the now known approximate models $\hat{\prob}_k(\cdot \mid s,a)$. For any new trajectory $n$ and label $k$, we estimate a distance $\dist(n,k)$ between the model $\hat{\prob}_{n,i}(\cdot \mid s,a)$ estimated from $n$ and the model $\hat{\prob}_k(\cdot \mid s,a)$ for $k$, after embedding both in the subspace mentioned above using $\Tilde{\textbf{V}}_{s,a}$. Again, we use a double estimator as hinted at by the use of the subscript $i$, similar to Algorithm~\ref{alg:clustering}. In practice $\dist(n,k)$ could also include occupancy measure differences. Each trajectory $n$ gets the label $k_n$ that minimizes $\dist(n,k)$. 

Previous work like \citet{poor2022mixdyn} and \citet{kakade2020meta} uses the word refinement for its model estimation and classification algorithms themselves. However, we posit that the monotonic improvement in log-likelihood offered by EM makes it well-suited for \textit{repeated application and refinement}, while in our case, the clear theoretical guarantees for the model estimation and classification algorithms make them well suited for \textit{single-step classification}. Note that we can also apply repeated refinement using EM to the labels obtained by single-step classification, which should combine the best of both worlds. 
%It could then be useful to say that clustering gives you a correctness in probability (Theorem 4) guarantee, and if you're still not confident, the rest of the method amounts to running EM with a clustering initialization. The EM component has guarantees of model estimation error given correctness (Theorem 5, we really need a 'given approximate correctness' guarantee), and label PACness (Theorem 6) guarantees conditional on model estimation error bounds of the previous step. You never really know where you are in the EM algorithm as usual, but on the off-chance you got something good in the previous step, you'll get something good in the next step and the algorithm will terminate.

\section{Analysis}\label{sec:analysis}

We have the following end-to-end guarantee for correctly classifying all data.

\begin{thm}[End-to-End Guarantee]\label{thm:endtoend}
Let both $N_{sub}$ and $N_{clust}$ be $\Omega\left( K^2S\frac{\log(1/\delta)} {f_{min}^2\alpha^3\Delta^8}\right)$ and let $T_n = \Omega\left( K^{3/2}t_{mix}\frac{\log^4((N_{clust}+N_{sub})/\delta)\log^3(1/\Delta)\log^4(1/\alpha)}{{\Delta^6\alpha^3}}\right)$. If we execute algorithms~\ref{alg:subspace-est}, ~\ref{alg:clustering} and model estimation, and then apply algorithm~\ref{alg:classification} to $\N_{sub}$ with $\lambda =1$, $\alpha/3 \leq \beta < \alpha$ and $\Delta^2/4 \leq \tau \leq \Delta^2/2$ for clustering and classification, then we can recover the true labels for the entire dataset ($\N_{clust} \cup \N_{sub}$) with probability at least $1-\delta$.
\end{thm}

\begin{proof}
This follows directly from Theorems~\ref{thm:subspace-est}, \ref{thm:clustering}, \ref{thm:model-est} and \ref{thm:classification} upon combining the conditions on $N_{sub}, N_{clust},$ and $T_n$ in both theorems. We also use the brief discussion after the statement of Theorem~\ref{thm:classification}.
\end{proof}

The dependence on model-specific parameters like $\alpha, \Delta$ and $f_{min}$ is conservative and can be easily improved upon by following the proofs carefully. We chose the form of the guarantees in this section to present a clearer message. In one example, there are versions of these theorems that depend on both $G$ and $T_n$. We choose $G = (T_n/t_{mix})^{2/3}$ to present crisper guarantees. For understanding how the guarantees would behave depending on both $G$ and $T_n$, or how to improve the dependence on model-specific parameters, the reader can follow the proofs in the appendix.

\subsection{Techniques and Proofs}
We make a few remarks on the technical novelty of our proofs. As mentioned in Section~\ref{sec:introduction}, we are dealing with two kinds of non-independence. While we borrow some ideas in our analysis from \citet{poor2022mixdyn} to deal with the temporal dependence, we crucially need new technical inputs to deal with the fact that we cannot cast the temporal evolution as a deterministic function with additive i.i.d. noise, unlike in linear dynamical systems.

We identify the blocking technique in \citet{bin1994mixing} as a general method to leverage the "near-independence" in observations made in a mixing process when they are separated by a multiple of the mixing time. Our proofs involve first showing that estimates made from a single trajectory would concentrate if the observations were independent, and then we bound the "mixing error" to account for the non-independence of the observations. We first choose a distribution (often labelled as a variant of $Q$ or $\Xi$) with desirable properties, and then bound the difference between probabilities of undesirable events under $Q$ and under the true joint distribution of observations $\chi$, using the blocking technique due to \citet{bin1994mixing}. 

There are many other technical subtleties here. In one example, the number of $(s,a)$ observations made in a single trajectory is itself a random variable and so our estimator takes a ratio of two random variables. To resolve this, we have to condition on the random set of $(s,a)$ observations recorded in a trajectory and use a conditional version of Hoeffding's inequality (different from the Azuma-Hoeffding inequality), followed by a careful argument to get unconditional concentration bounds, all under $Q$.

\subsection{Subspace Estimation} 

For subspace estimation, we have the following guarantee.
\begin{restatable}[Subspace Estimation Guarantee]{thm}{SubspaceEst}\label{thm:subspace-est}
Consider $2$ models with labels $k_1, k_2$ and a state-action pair $s,a$ with $d_{min}(s,a) \geq \alpha/3$. Consider the output $\textbf{V}_{s,a}^T$ of Algorithm~\ref{alg:subspace-est}. Let $f_{min} = \min(f_{k_1}, f_{k_2})$ be the lower of the label prevalences. Remember that each trajectory has length $T_n$.

Then given that $N_{sub} = \Omega\left( \frac{\log(1/\delta)}{\alpha^2} \right)$, $T_n =\Omega(t_{mix}\log^4(1/\alpha))$, with probability at least $1-\delta$, for $k=k_1,k_2$
$$\|\prob_k(\cdot \mid s,a) - \textbf{V}_{s,a}\textbf{V}_{s,a}^T\prob_k(\cdot \mid s,a)\|_2 \leq \epsilon_{sub}(\delta)$$
where

\begin{itemize}
    \item For $T_n = \Omega\left(  t_{mix} \log^3\left( \frac{f_{min}N_{sub} \alpha}{KS\log(1/\delta)}\right)\right)$ 

    $$\epsilon_{sub}(\delta) = O\left(\sqrt{ \frac{K}{f_{min}} \left( \sqrt{\frac{S}{N_{sub} \cdot \alpha^3} \log\left(\frac{1}{\delta}\right)} \right)}\right)$$
    
    \item While for $T_n  = O\left(  t_{mix} \log^3\left( \frac{f_{min}N_{sub} \alpha}{KS\log(1/\delta)}\right)\right)$

    $$\epsilon_{sub}(\delta) = O\left( \left(\frac{1}{2} \right)^{\frac{1}{16}\left(\frac{T_n}{t_{mix}}\right)^{1/3}}\right)$$
\end{itemize} 

Alternatively, we only need $N_{sub} = \Omega\left( \frac{K^2S\log(1/\delta)}{f_{min}^2\alpha^3\epsilon^4}\right)$ and $T_n = \Omega\left(t_{mix} \log^3(1/\epsilon) \log^4(1/\alpha) \right)$ trajectories for $\epsilon$ accuracy in subspace estimation with probability at least $1-\delta$.
\end{restatable}

\begin{remark} \textbf{Why are short trajectories enough?}
Notice that the length of trajectories only affects the bound as a multiple of $t_{mix}$ with some logarithmic terms. This is because intuitively, the onus of estimating the correct subspace lies on aggregating information across trajectories. So, as long as there are enough trajectories, each trajectory does not have to be long.
\end{remark}

\subsection{Clustering} 

Remember that $\Delta$ is the model separation and $\alpha$ is the corresponding "stationary occupancy measure" from Assumption~\ref{as:model_sep}. We give guarantees for choosing $\lambda=1$, which corresponds to using only model difference information instead of also using occupancy measure information. This is unavoidable since we have no guarantees on the separation of occupancy measures. See Section~\ref{ssec:in-practice-clustering} for a discussion. Here, we provide a high-probability guarantee for exact clustering.

\begin{restatable}[Exact Clustering Guarantee]{thm}{Clustering}\label{thm:clustering}
Pick any pair of trajectories $n,m$. Then for $\Freq_\beta$ so that it contains $(s,a)$ with $d_{min}(s,a) \geq \Omega(\alpha)$, $T_n = \Omega(t_{mix}\log^4(1/\delta)/\alpha^3)$, with probability at least $1-\delta$,
$$\left\vert \dist_1(m,n) - \left\Vert \Delta_{m,n} \right\Vert_2^2 \right\vert$$
is
$$O\left(\sqrt{\frac{K\log(1/\delta)}{\alpha}}\left( \frac{t_{mix}}{T_n} \right)^{\frac{1}{3}} \right) + 4\epsilon_{sub}(\delta/2)$$
This means that if we choose $\lambda = 1$, then if $\epsilon_{sub}(\delta) \leq \Delta^2/32$ and $T_n =  \Omega\left( K^{3/2}t_{mix}\frac{\log^4(N_{clust}/(\alpha\delta))}{\Delta^6\alpha^{3}}\right)$, 
no distance estimate attains a value between $\Delta^2/4$ and $\Delta^2/2$. So, Algorithm \ref{alg:clustering} attains exact clustering using a threshold of say $\Delta^2/3$ with probability at least $1-\delta$.
\end{restatable}

Since we already have high probability guarantees for exact clustering before refinement of the clusters, guarantees for the EM step analogous to the single-step guarantees for refinement in \citet{poor2022mixdyn} are not useful here. However, we do still present single-step guarantees for the EM algorithm in our case using a combination of Theorem~\ref{thm:model-est} for the M-step and Theorem~\ref{thm:em-guarantee} in Appendix~\ref{sec:em-guarantee}.

\subsection{Model Estimation and Classification} 

%\kevin{These guarantees are somewhat vacuous, given that we already know the correctness of the hard EM algorithm, but we'll provide them for completeness in the ICML version}

We also have guarantees for correctly estimating the relevant parts of the models and classifying sets of trajectories different from $\N_{clust}$.

\begin{restatable}[Model Estimation Guarantee]{thm}{ModelEstimation}\label{thm:model-est}
For any state action pair $(s,a)$ with $d_{min}(s,a) \geq \alpha/3$, and for $GN_{clust} \geq \Omega\left( \frac{\log(1/\delta)}{f_{min}^2\alpha^2}\right)$ and $T_n \geq \Omega(Gt_{mix}\log(G/\delta))$, with probability greater than $1-\delta$, 
$$\|\hat{\prob}_{k}(\cdot \mid s,a) - \prob_k(\cdot \mid s,a)\|_1$$
is bounded above by
$$O\left(\left( \frac{t_{mix}}{T_n}\right)^{1/3}\sqrt{\frac{1}{N_{clust}f_{min}\alpha} (S+\log(\frac{1}{\delta}))}\right)$$
\end{restatable}

Note that since the $1$-norm is greater than the $2$-norm, the same bound holds in the $2$-norm as well. Also notice that since our assumptions do not say anything about observing all $(s,a)$ pairs often enough, we can only given guarantees in terms of the occurrence frequency of $(s,a)$ pairs.

\begin{restatable}[Classification Guarantee]{thm}{Classification}\label{thm:classification}
Let $\epsilon_{mod}(\delta)$ be a high probability bound on the model estimation error $\|\hat{\prob}_{k}(\cdot \mid s,a) - \prob_k(\cdot \mid s,a)\|_2$. Then there is a universal constant $C_3$ so that Algorithm~\ref{alg:classification} can identify the true labels for trajectories in $\N_{class}$ with probability at least $1-\delta$ for $T_n =  \Omega\left( K^{3/2}t_{mix}\frac{\log^4(N_{class}/(\alpha\delta))}{\Delta^6\alpha^{3}}\right)$, whenever $\epsilon_{mod}(\delta/2) \leq \frac{C_3\Delta^4f_{min}\alpha}{K}$ and $N_{clust} \geq \Omega\left( \frac{\log(1/\delta)}{f_{min}^2\alpha^2}\right)$.
\end{restatable}

Note that by Theorem~\ref{thm:model-est}, a sufficient condition for $\epsilon_{mod}(\delta/2) \leq \frac{C_3\Delta^4f_{min}\alpha}{K}$ is $N_{clust}T_n^{2/3} \geq \Omega\left(K^2t_{mix}^{2/3}S \frac{\log(1/\delta)}{\Delta^8f_{min}^3\alpha^3}\right)$. Under the conditions on $T_n$ in Theorem~\ref{thm:classification}, a suboptimal but sufficient condition on $N_{clust}$ is $N_{clust} = \Omega\left( K^2S\frac{\log(1/\delta)} {f_{min}^2\alpha^3\Delta^8}\right)$, which matches that for $N_{sub}$.

\section{Practical Considerations}\label{sec:in-practice}

\subsection{Subspace Estimation}\label{ssec:in-practice-subspace}

\textbf{Heuristics for choosing K:} One often does not know $K$ beforehand and often wants temporal features to guide the process of determining $K$, for example in identifying the number of groups of similar people represented in a medical study. We suggest a heuristic for this. One can examine how many large eigenvalues there are in the decomposition, via (1) ordering the eigenvalues of $\HBM_{sa} \; \forall s,a$ by magnitude, (2) taking the square of each to obtain the eigenvalue energy, (3) taking the mean or average over states and actions, and (4) plotting a histogram. See Figure~\ref{fig:eigenergy} in the appendix.

% \begin{figure}[h]
%     \centering
%     \includegraphics[width=8cm, height=5cm]{images/eigours.png}
%     \caption{Eigenvalues of $\HBM_{0,1} + \HBM_{0,1}^T$. Note how there are only $2=K$ large eigenvalues.}
%     \label{fig:eigours}
% \end{figure}

One can also consider running the whole process with different values of $K$ and choose the value of $K$ that maximises the likelihood or the AIC of the data (if one wishes the mixture to be sparse). However, \citet{fitzpatrick2022asymp} points out that such likelihood-based methods can lead to incorrect predictions for $K$ even with infinite data.

% \textbf{Note on software:} $V_{s,a}^T$ is the matrix of the $K$ eigenvectors for the top $K$ eigenvalues, but if some eigenvalues repeat, one may need to replace the set of the corresponding eigenvectors with orthonormal ones, depending on the software used. The user should check the eigenvalues by hand.

\subsection{Clustering}\label{ssec:in-practice-clustering}

\textbf{Picking $\beta$:} Choosing $\beta$ involves heuristically picking state-action pairs that have high frequency and "witness" enough model separation. We propose one method for this. For each $(s,a)$ pair, one first executes subspace estimation and then averages the value of $\dist_1(m,n)$ across all pairs of trajectories. Call this estimate $\Delta_{s,a}$, since it is a measure of how much model separation $(s,a)$ can "witness". We then compute the occupancy measure value $d(s,a)$ of $(s,a)$ in the entire set of observations. Making a scatter-plot of $\Delta_{s,a}$ against $d(s,a)$, we want a value of $\beta$ so that there are enough pairs from $\Freq_\beta$ in the top right. 

\textbf{Picking thresholds $\tau$:} The histogram of $\dist$ plotted will have many modes. The one at $0$ reflects distance estimates between trajectories belonging to the same hidden label, while all the other modes reflect distance between trajectories coming from various pairs of hidden labels. The threshold should thus be chosen between the first two modes. See Figure~\ref{fig:histacc} in the appendix.

\textbf{Picking $\lambda$:} In general, occupancy measures are different for generic policies interacting with MDPs and should be included in the implementation by choosing $\lambda<1$. The histogram for $\dist_2$ should indicate whether or not occupancy measures allow for better clustering (if they have the right number of well-separated modes). 
%import text from history

\textbf{Versions of the EM algorithm:} In our description of the EM algorithm, we only use next-state transitions as observations instead of the whole trajectory. So, we do not learn other parameters like the policy and the starting state's distribution for the EM algorithm. This makes sense in principle, because our minimal assumptions only talk about separation in next-state transition probabilities, and there is no guarantee that other information will help with classification. In practice, one should make a domain-specific decision on whether or not to include them.

\textbf{Initializing soft EM with cluster labels:} We also recommend that when one initializes the soft EM algorithm with results from the clustering step, one introduces some degree of uncertainty instead of directly feeding in the 1-0 clustering labels. That is, for trajectory $m$, instead of assigning $\mathbbm{1}(i=k_m)$ to be the responsibilities, make them say $0.8\cdot\mathbbm{1}(i\in \C_k) + 0.2/K$ instead. We find that this can aid convergence to the global maximum, and do so in our experiments.

%Terminate upon clustering step
%Separate subspace and clustering datasets

\section{Experiments}\label{sec:experiments}

We perform our experiments for MDPs on an 8x8 gridworld with $K=2$ elements in the mixture (from \cite{bruns2021model}). Unlike \citet{bruns2021model}, the behavior policy here is the same across both elements of the mixture to eliminate any favorable effects that a different behavior policy might have on clustering, so that we evaluate the algorithm on fair grounds. The first element is the "normal" gridworld, while the second is adversarial -- transitions are tweaked towards having a higher probability of ending up in the lowest-value adjacent state. The value is only used to adjust the transition structure in the second MDP, and has no other role in our experiments. The mixing time of this system is roughly $t_{mix} \approx 25$. We only use $\dist_1$ for the clustering, omitting the occupancy measures to parallel the theoretical guarantees. Including them would likely improve performance. We chose to perform the experiments with 1000 trajectories, given the difficulty of obtaining large numbers of trajectories in important real-life scenarios that often arise in areas like healthcare. 

\begin{figure}[h]
    \centering
    \includegraphics[width=8cm, height=6cm]{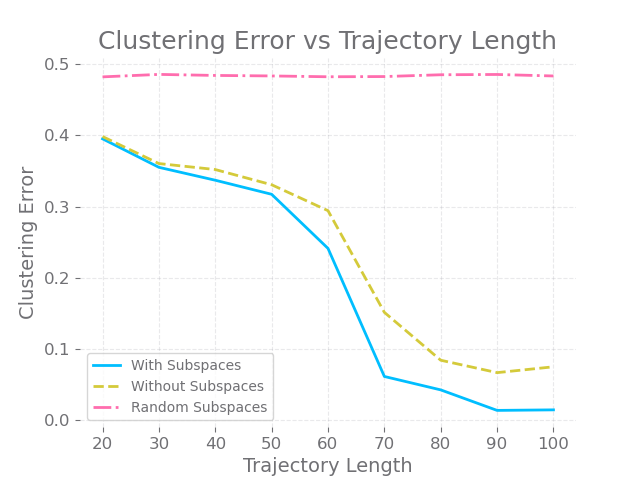}
    \caption{Clustering error v.s. trajectory length on 1000 trajectories, with a comparison between using $\textbf{V}_{s,a}^T$, $I_{S\times S}$ or a random projector to a $K$-dimensional subspace in Algorithm \ref{alg:clustering}. The same threshold was used for each trajectory length. Results averaged over 30 trials. The mixing time of this system is roughly $t_{mix} \approx 25$.}
    \label{fig:clust}
\end{figure}

Figure \ref{fig:clust} plots the error at the end of Algorithm~\ref{alg:clustering} (before refinement) while either using the projectors $\textbf{V}_{s,a}^T$ determined in Algorithm \ref{alg:subspace-est} ("With Subspaces"), replacing them with a random projector ("Random Subspaces") or with the identity matrix ("Without Subspaces"). The difference in performance demonstrates the importance of our structured subspace estimation step. Also note that past a certain point, between $T_n=60$ and $T_n=70 \sim 3t_{mix}$, the performance of our method drastically improves, showing that the dependence of our theoretical guarantees on the mixing time is reflected in practice as well. We briefly discuss the poor performance of choosing a random subspace in Appendix~\ref{sec:random-subspaces}.
\begin{figure}[h]
    \centering
    \includegraphics[width=8cm, height=6cm]{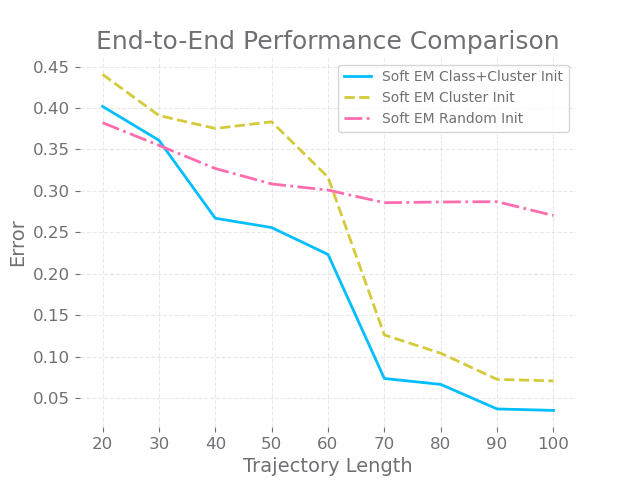}
    \caption{End-to-end error v.s. trajectory length on 1000 trajectories, comparing initializations of the soft EM algorithm using (1) random initializations, (2) models from $\N_{clust}$, and (3) classification and clustering labels from $\N_{clust}$ and $\N_{sub}$. Results averaged over 30 trials, with 30 random initializations for randomly-initialized EM within each trial.}
    \label{fig:emClust}
\end{figure}

In Figure~\ref{fig:emacc}, we benchmark our method's end-to-end performance against the most natural benchmark, the randomly initialized EM algorithm. We use the version of the soft EM algorithm that considers the entire trajectory to be our observation, and thus also includes policies and starting state distributions. So, we are comparing our method against the full power of the EM algorithm. We have three different plots, corresponding to (1) soft EM with random initialization, (2) Refining models obtained from the model estimation step applied to $\N_{clust}$ using soft EM on $\N_{clust} \cup \N_{sub}$, and (3) Refining labels for $\N_{clust}$ and $\N_{sub}$ using soft EM (the latter obtained from applying Algorithm~\ref{alg:classification} to $\N_{sub}$). We report the final label accuracies over the entire dataset, $\N_{clust} \cup \N_{sub}$. Remember that we can view refinement using soft EM as initializing soft EM with the outputs of our algorithms. Note that the plot for (3), which reflects the true end-to-end version of our algorithm, almost always outperforms randomly initialized soft EM. Also, for $T_n > 60$, both variants of our method outperform randomly initialized soft EM. We present a variant of Figure~\ref{fig:emClust} with hard EM included as Figure~\ref{fig:em-soft-hard} in the appendix.

\section{Discussion}

We have shown that we can recover the true trajectory labels with (1) the number of trajectories having only a linear dependence in the size of the state space, and (2) the length of the trajectories depending only linearly in the mixing time -- even before initializing the EM algorithm with these clusters (which would further improve the log-likelihood, and potentially cluster accuracy). End-to-end performance guarantees are provided in Theorem \ref{thm:endtoend}, and experimental results are both promising and in line with the theory. 

%\subsection{Online Classification and Guarantees}
%\kevin{chinmaya to finish}
%\cknote{summarize, other controlled processes. This plus online, separation is not required then.}

\subsection{Future Work}
\textbf{Matrix sketching:} The computation of $\dist_1(m,n)$ is computationally intensive, amounting to computing about $S \times A$ distance matrices. We could alternatively approximate the thresholded version of the matrix $\dist(m,n)$ (which in the ideal case is a rank-$K$ binary matrix) with ideas from \citet{musco2016kernel}.

\textbf{Function approximation:} The question of the right extension of our ideas to Markov chains and MDPs with large, infinite, or uncountable state spaces is very much open (at least, those whose transition kernel is not described by a linear dynamical systems). This is important, as many applications often rely on continuous state spaces.

\textbf{Other controlled processes:} \citet{poor2022mixdyn} learn a mixture of linear dynamical systems without control input. An extension to the case with control input will be very valuable. We believe that the techniques used in our work may prove useful in this, as well as for extensions to other controlled processes that may neither be linear nor Gaussian. 

%\textbf{Non-mixing case:} Our work heavily relies on mixing to obtain guarantees. We invite readers to explore the non-mixing case, to see if any theoretical results can be obtained or to see if the method still succeeds in practice even without meeting the mixing time. 

\bibliographystyle{apalike}
\bibliography{ref}

\newpage
\appendix
\onecolumn

\section{Additional Figures}\label{sec:additional-figures}

\subsection{Determining $K$}

See Figure~\ref{fig:eigenergy} below, following the discussion in section~\ref{ssec:in-practice-subspace}.
\begin{figure}[h]
    \centering
    \includegraphics[width=8cm, height=6cm]{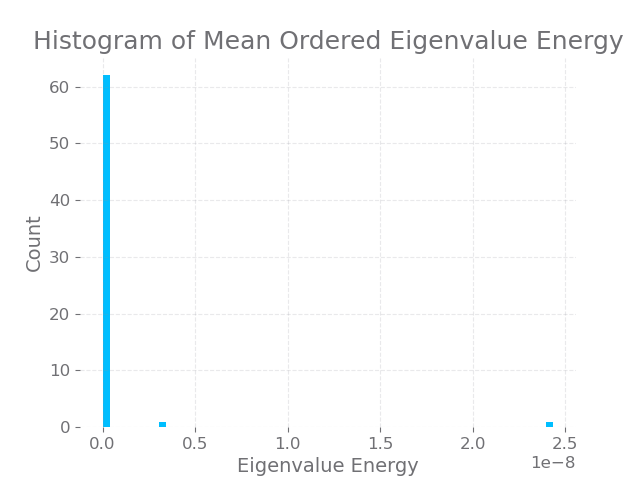}
    \caption{Histogram of the average ordered eigenvalue energy (the square of the eigenvalue) where the mean is taken over states and actions. There are two large eigenvalues, corresponding to $K=2$.}
    \label{fig:eigenergy}
\end{figure}

\subsection{Block Matrix of Raw Distance Estimates}
See Figure~\ref{fig:blocks} below, which presents the raw distance matrix before thresholding, to provide a sense of the quality of the pairwise distance estimates themselves. These could also be used for agglomerative clustering, for example.
\begin{figure}[h]
    \centering
    \includegraphics[width=8cm, height=6cm]{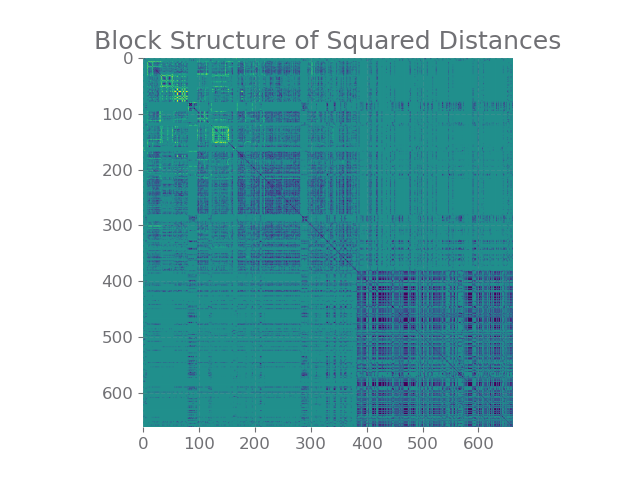}
    \caption{Block structure of the matrix of squared pairwise distance estimates (after sorting).}
    \label{fig:blocks}
\end{figure}

\subsection{Determining The Threshold $\tau$}
See Figure~\ref{fig:histacc} below, following the discussion in section~\ref{ssec:in-practice-clustering}.
\begin{figure}[H]
    \centering
    \includegraphics[width=7cm, height=5cm]{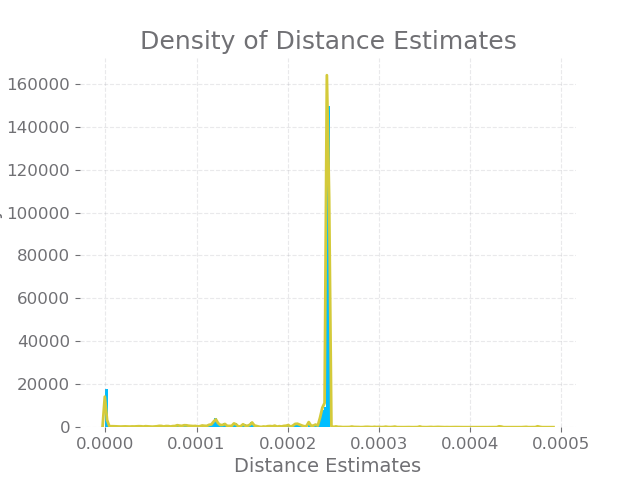}
    \includegraphics[width=8cm, height=6cm]{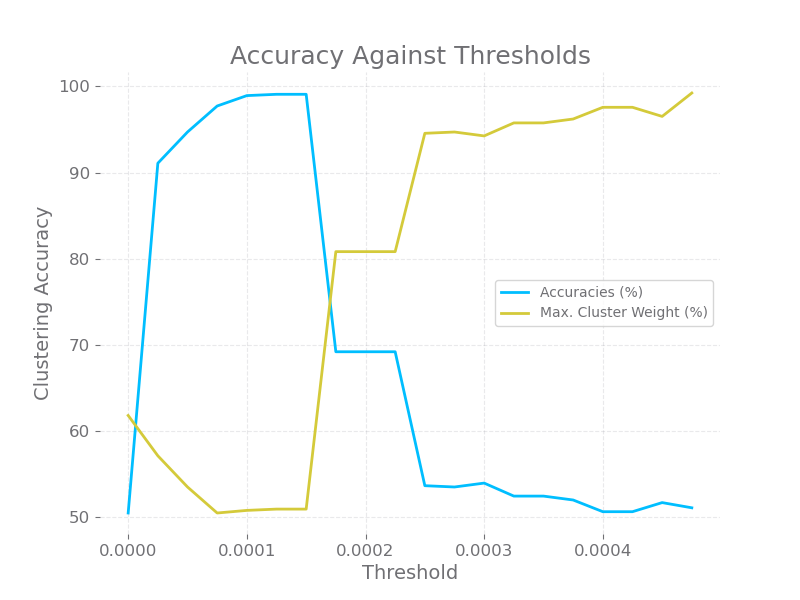}
    \caption{Histogram (and KDE) of pairwise squared distance estimates in projected subspace above, and accuracy against thresholds below. Note how there is a spurious mode around the 0.00015 mark, and picking any threshold past it yields a significant drop in accuracy.}
    \label{fig:histacc}
\end{figure}

\subsection{Local Extrema in EM}
See Figure~\ref{fig:emacc} below, illustrating how EM often gets stuck in suboptimal local extrema, given by the low final log-likelihood values recorded in the scatterplot.
\begin{figure}[H]
    \centering
    \includegraphics[width=8cm, height=6cm]{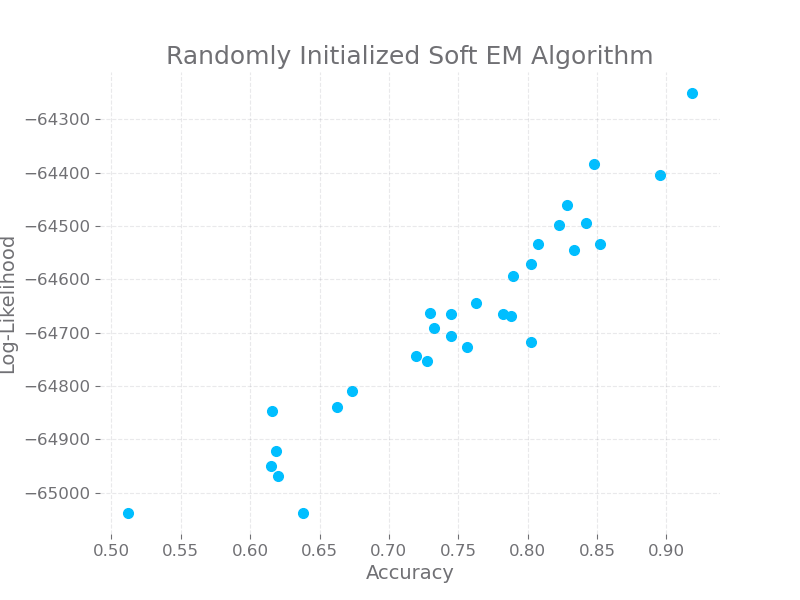}
    \caption{Scatter-plot of likelihoods v.s. clustering accuracy achieved by the randomly-initialized soft EM algorithm over 30 trials. Randomly-initialized soft EM does not achieve the global maximum all of the time.}
    \label{fig:emacc}
\end{figure}

\subsection{Comparing End-To-End Performance Using Soft and Hard EM} 

We compare various initializations of EM -- (1) random initializations, (2) models from $\N_{clust}$, and (3) classification and clustering labels from $\N_{clust}$ and $\N_{sub}$ -- this time using both soft and hard EM.

\begin{figure}[H]
    \centering
    \includegraphics[width=8cm, height=6cm]{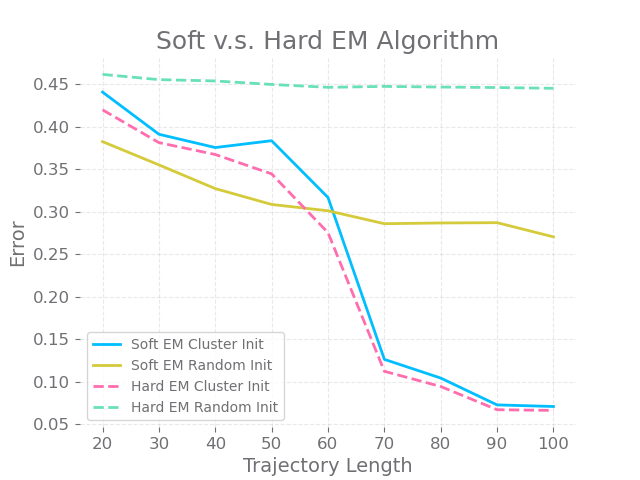}
    \caption{End-to-end error v.s. trajectory length on 1000 trajectories, comparing various initializations of the soft and the hard EM algorithm. Results averaged over 30 trials, with 30 random initializations for randomly-initialized EM within each trial.}
    \label{fig:em-soft-hard}
\end{figure}

\newpage
\newpage
\section{Discussion on Using Random Projections}\label{sec:random-subspaces}

We note that those familiar with the intuition behind the Johnson-Lindenstrauss lemma would guess that a projection to a random $n$-dimensional subspace for low $n$ would preserve distances with good accuracy. However, note that the bound on the dimension $n$ needed to preserve distances between our $N_{clust}$ estimators up to a multiplicative distortion of $1 \pm \epsilon$ is $\frac{\log(N_{clust})}{\epsilon^2}$. This bound is known to be tight, see for example \citet{JLlemmatight2017}. Upon thought, this shows that to get good distortion bounds (which will contribute to the deviation between distance estimates and the thresholds), we need a large dimension, interpreted as being affected by the $1/\epsilon^2$. In fact, as soon as $\log(N_{clust})$ exceeds $1$, we will need a dimension of order $1/\Delta^2$, while $K$ can be arbitrarily small compared to this. 

In our case, $K=2$, and we see that we don't get good performance using a random subspace until we hit dimension 50, where the maximum dimension is $S=64$. Clearly, the $1/\epsilon^2$ term in the Johnson-Lindenstrauss lemma drastically affects the performance of using random subspaces. Using a random subspace of dimension $50$ for $S=64$ is much closer to not projecting at all than to using a subspace of dimension $2$.

\begin{figure}[H]
    \centering
    \includegraphics[width = 8cm, height = 6cm]{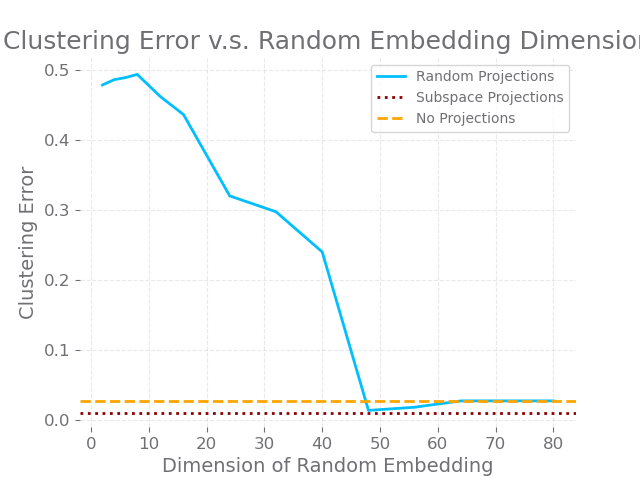}
    \caption{Clustering error using random projections of varying dimension for a trajectory length of 100, benchmarked against the performance of the "with subspace" and "without subspace" versions.}
    \label{fig:hl-random-embedding}
\end{figure}
\newpage
\section{Details of the EM Algorithm}\label{sec:em-algorithm-details}

We describe the E and M steps for hard EM below first, for simplicity. 

\textbf{M-step:} Given the cluster labels, we can estimate each model with the MLE as:
\begin{align*}
    \hat{\prob}_k(s' \vert s,a) &\gets \frac{\sum_{n \in \N_{clust}} \ind_{n \in \C_k}N(n,s,a,s')}{\sum_{n \in \N_{clust}} \ind_{n \in \C_k}N(n,s,a)}\\
    \hat{f}_k &\gets \frac{\sum_{n \in \N_{clust}} \ind_{n \in \C_k}}{N_{clust}} = \frac{|\C_k|}{N_{clust}}
\end{align*}

Readers can convince themselves that this is truly the MLE estimate by making the following observation. We can write the log-likelihood of the predicted clusters $\C_k$ and estimated models as $\sum_{k=1}^K \sum_{n\in \N_{clust}} \ind_{n \in \C_k} \ell(\hat{\prob}_k, \hat{f}_k, n)$, where $\ell(\hat{\prob}_k, \hat{f}_k, n) = \log\left(f_k \prod_{s,s',a} (\hat{\prob}_k(s'\mid s,a))^{N(n,s,a,s')}\right)$. The rest of the derivation mimics the well-known and straightforward computation for Markov chains, using Lagrange multipliers to constrain the estimates to probability distributions.
% We can similarly estimate the policies under each mixture element $\hat{\pi}_k(a|s) \gets \frac{\sum_{n \in \C_k} N(n,s,a)}{\sum_{n \in \C_k} N(n, s)}$ and starting state distributions $\hat{p}_{k,s_{0,n}} \gets \frac{1}{|\C_k|}\sum_{n \in \C_k} N(n, s_0)$.

% Call the starting state of each trajectory $s_{0,n}$ and denote the probability of starting at a state $s$ in label $k$ by $p_{k,s}$ and call its estimate $\hat{p}_{k,s}$. We can write the log-likelihood of the predicted clusters $\C_k$ and estimated models as $\sum_{k=1}^K \ell(\hat{\prob}_k, \hat{\pi}_k, \hat{p}_{k,s_{0,n}})$, where $\ell(\hat{\prob}_k, \hat{\pi}_k, \hat{p}_{k,s_{0,n}}) = \sum_{n \in \C_k} \log\left(\hat{p}^{(k)}_{s_{0,n}} \prod_{s,s',a} (\hat{\prob}_k(s'\mid s,a)\hat{\pi}_k(a|s))^{N(n,s,a,s')}\right)$

\textbf{E-step:} On new or unseen data, assign cluster membership according to the following rule:
\begin{equation}
     k_m \gets \argmax_k \ell(\hat{\prob}_k, \hat{f}_k, m) + \log(\hat{f_i})
\end{equation}
where $\ell(\hat{\prob}_k, m)$ is as above.

Note that for soft EM, we can replace every occurrence of $\ind_{n \in \C_k}$ in the M-step with $p_n(k)$, where $p_n(\cdot)$ is the posterior for trajectory $n$ having label $k$, which is constantly updated during soft EM. For the E-step, we replace the argmax computation by a computation of $p_n(k) = \prob(k_n = k \mid \hat{\prob}_k, \hat{f}_k, 1 \leq k \leq K)$. Intuitively described, in hard EM, we recompute the values of $\ind_{n \in \C_k}$ using the argmax during the E-step, while in soft EM, we recompute the values of $p_n(k)$.
\newpage
\section{The Classification Algorithm}\label{sec:classification-algorithm}

Note that we define a new quantity, $\hat{f}_{k,s,a}$, which is the proportion of trajectories with label $k$ among all trajectories in $\N_{clust}$ where $s,a$ is observed.

\begin{algorithm}
\centering
	\caption{Classification}
    \label{alg:classification}
	\begin{algorithmic}[1]
	    \STATE \textbf{Input:} Clusters $\C_k \subset \N_{clust}$, models $\hat{\prob}_k(\cdot \mid s,a)$ estimated from $\C_k$, and a set $\N_{class}$ of trajectories to classify.
	    \STATE Compute $\hat{f}_{k,s,a}$ for all $k,s,a$.
		\STATE Compute $\Tilde{\textbf{M}}_{s,a} = \sum_{k=1}^K \hat{f}_{k,s,a}\hat{\prob}_{k}(\cdot \vert s,a)\hat{\prob}_{k}(\cdot \vert s,a)^T$ and store the orthogonal projector $\Tilde{\textbf{V}}_{s,a}^T$ to its top-K eigenspace, for each $(s,a)$.
		\STATE Compute $\hat{\textbf{d}}_k = \frac{1}{|\C_k|}\sum_{n \in \C_k} \frac{N(n,s,a)}{G}$ for all $k$.
		\STATE Compute $\Tilde{D} = \sum_{k=1}^K \hat{\textbf{d}}_k\hat{\textbf{d}}_{k}^T$ and store the orthogonal projector $\Tilde{\textbf{U}}^T$ to its top-K eigenspace.
		\STATE Compute the set $SA_\beta$ by picking $(s,a)$ pairs with occurrence more than $\beta$
		\STATE $\textbf{d}_{n,1}, \textbf{d}_{n,2} \gets \textbf{0} \in \Real^{SA}$
		\FOR{$(i,s,a) \in \{1,2\} \times S \times A$}
		\STATE Compute $\textbf{N}(n,i,s,a, \cdot)$, $N(n,i,s,a),$ $\;\; \forall n$
		\STATE $\hat{\prob}_{n,i}(\cdot \vert s,a) \gets \frac{\textbf{N}(n,i,s,a,\cdot)}{N(n,i,s,a)}\ind_{N(n,i,s,a) \neq 0},$  $\;\; \forall n$
		\STATE $[\hat{\textbf{d}}_{n,i}]_{s,a} \gets \frac{N(n,i,s,a)}{G},$ $\;\; \forall n$
		\ENDFOR
        \FOR{$(n,k) \in \N_{clust} \times \{1, 2, \dots K\}$}
        \FOR{$(i,s,a) \in \{1,2\} \times S \times A$}
        \STATE $\hat{\bm{\Delta}}_{i,s,a} := (\hat{\prob}_{n,i}(\cdot \mid s,a) - \hat{\prob}_{k}(\cdot \vert s,a))\Tilde{\textbf{V}}^T_{s,a}$
        \ENDFOR
            \STATE $\dist_1(n,k) := \max_{s,a} \hat{\bm{\Delta}}_{1,s,a}^T \hat{\bm{\Delta}}_{2,s,a}$
        \STATE $\dist_2(n,k) := (\hat{\textbf{d}}_{n,1} - \hat{\textbf{d}}_{k})^T\textbf{U}\textbf{U}^T(\hat{\textbf{d}}_{n,2} - \hat{\textbf{d}}_{k})$
        \STATE $\dist(n,k) := \lambda\dist_1(n,k) + (1-\lambda)\dist_2(n,k)$
        \ENDFOR
        \STATE Assign $k_n \gets \argmin_k \dist(n,k)$ for each $n$.
	\end{algorithmic}
\end{algorithm}
\newpage
\section{Proof of Theorem~\ref{thm:subspace-est}}

\subsection{Proof of the theorem}
We recall the theorem here.

\SubspaceEst*

\begin{remark}
We can convert the $\alpha^3$ in the denominator to an $\alpha$ at the cost of making $T_n$ more heavily dependent on $\alpha$ (more than just $\log(1/\alpha)$). Intuitively, $\alpha$ accounts for the probability of not observing $s,a$, so this is just saying that we can shift the onus for that from the number of trajectories to their length. We chose not to do that since we are trying to minimize the length of trajectories needed, and assume that we have access to many trajectories.
\end{remark}

\begin{proof}

The main input is the proposition below, proved in the next section.

\begin{restatable}{prop}{SubspaceEstProp}\label{prop:subspace-est}
Consider $L<K$ models with labels $j_l$, $1 \leq l \leq L$, with $d_{min}(s,a):= \min_l d_{j_l}(s,a)$. Consider the output $\textbf{V}_{s,a}^T$ of Algorithm~\ref{alg:subspace-est}. Let $f_{min} = \min_l f_{j_l}$ be the minimum frequency across these models in the mixture. Remember that each trajectory has length $T_n$. Then we have the guarantee that with probability at least $1-\delta$

$$\|\prob_j(\cdot \mid s,a) - \textbf{V}_{s,a}\textbf{V}_{s,a}^T\prob_j(\cdot \mid s,a)\|_2 $$ 
is bounded above by 
$$\sqrt{\frac{4K}{f_{min}d_{min}(s,a)} \left( \sqrt{\frac{128}{N_{sub} \cdot d_{min}(s,a)} (2S\log(12)+\log(4/\delta))} +  \left(\frac{1}{2} \right)^{\frac{T_n}{8Gt_{mix}}}\right)}$$

for all $j \in \{j_l \mid 1 \leq l \leq L\}$, when $N_{sub} \geq \frac{32}{d_{min}(s,a)^2}\log\left(\frac{1}{\delta}\right)$ and $\frac{T_n}{8t_{mix}} > \frac{G\log (48G/ d_{min}(s,a))}{\log 2}$.
\end{restatable}

For a state-action pair with $d_{min}(s,a) \geq \alpha/3$, the conditions simplify to $N_{sub} \geq \Omega\left( \frac{\log(1/\delta)}{\alpha^2}\right)$ and $T_n \geq \Omega(Gt_{mix}\log(G/\alpha))$. We set $G = \left( \frac{T_n}{t_{mix}}\right)^{\frac{2}{3}}$ to get bounds that only depend on $T_n$. Note that this means a sufficient condition on $T_n$ is $T_n \geq \Omega(t_{mix}\log^4(1/\alpha))$ (one can show this with an elementary computation). Also note that
$$\sqrt{\frac{S+\log(1/\delta)}{N_{sub} \cdot \alpha}} \leq \sqrt{\frac{S\log(1/\delta)}{N_{sub} \cdot \alpha}} $$
Then with probability at least $1-\delta$, the following bound holds for any label $j=j_l$ for some $l$.
$$\|\prob_j(\cdot \mid s,a) - \textbf{V}_{s,a}\textbf{V}_{s,a}^T\prob_j(\cdot \mid s,a)\|_2 \leq O\left(\sqrt{\frac{K}{f_{min}\alpha} \left( \sqrt{\frac{S\log(1/\delta)}{N_{sub} \cdot \alpha} } +  \left(\frac{1}{2} \right)^{\frac{1}{8}\left(\frac{T_n}{t_{mix}}\right)^{1/3}}\right)}\right)$$

So, there is a universal constant $C_2$ so that for $T_n > C_2 t_{mix} \log^3\left( \frac{f_{min}N_{sub} \alpha}{KS\log(1/\delta)}\right)$, 

$$\left(\frac{1}{2} \right)^{\frac{1}{8}\left(\frac{T_n}{t_{mix}}\right)^{1/3}} \leq C'\frac{K}{f_{min}} \left( \sqrt{\frac{S}{N_{sub} \cdot \alpha^3} \log\left(\frac{1}{\delta}\right)} \right)$$

While for $T_n = O\left( t_{mix} \log^3\left( \frac{f_{min}N_{sub} \alpha}{KS\log(1/\delta)}\right)\right)$,

$$\frac{K}{f_{min}} \left( \sqrt{\frac{S}{N_{sub} \cdot \alpha^3} \log\left(\frac{1}{\delta}\right)} \right) \leq O\left(\left(\frac{1}{2} \right)^{\frac{1}{8}\left(\frac{T_n}{t_{mix}}\right)^{1/3}}\right)$$

So, combining all these, for $N_{sub} = \Omega\left( \frac{\log(1/\delta)}{\alpha^2} \right)$, $T_n =\Omega(t_{mix}\log^4(1/\alpha))$\\

\begin{itemize}
    \item For $T_n = \Omega\left(  t_{mix} \log^3\left( \frac{f_{min}N_{sub} \alpha}{KS\log(1/\delta)}\right)\right)$ 

    $$\epsilon_{sub}(\delta) = O\left(\sqrt{ \frac{K}{f_{min}} \left( \sqrt{\frac{S}{N_{sub} \cdot \alpha^3} \log\left(\frac{1}{\delta}\right)} \right)}\right)$$
    
    \item While for $T_n  = O\left(  t_{mix} \log^3\left( \frac{f_{min}N_{sub} \alpha}{KS\log(1/\delta)}\right)\right)$

    $$\epsilon_{sub}(\delta) = O\left( \left(\frac{1}{2} \right)^{\frac{1}{16}\left(\frac{T_n}{t_{mix}}\right)^{1/3}}\right)$$
\end{itemize}

\end{proof}

\subsection{Proof of the Proposition~\ref{prop:subspace-est}}

We recall the proposition here.

\SubspaceEstProp*

\begin{remark}
We should point out that we will only need $L=2$ for subsequent theorems. Also, remember that only $s,a$ with $d_{min}(s,a) > \alpha$ will be relevant in subsequent theorems, with $\alpha$ as in our assumption.
\end{remark}

\begin{proof}
For brevity of notation, we will denote $c_{n,i} := N(n,i,s,a)$, $\textbf{w}_{n,i} := N(n,i,s,a, \cdot)$ and suppress the $(s,a)$ dependence. We will first need the following lemma which guarantees that we can get past mixing and concentration hurdles with our estimator, modulo actually observing $s,a$ in both segments.

\begin{restatable}{lemma}{HatVsTilde}\label{lem:h-hat-h-tilde}
Let $\B_n$ be the event given by $n \in \N_{traj}(s,a)$, which is the same as $c_{n,1}c_{n,2} \neq 0$ and let

$$\textbf{M}_{s,a} = \sum_{j=1}^K \prob(k_n=j \mid \B_n)\prob_j(\cdot \mid s,a) \prob_j(\cdot \mid s,a)^T$$.

Call our estimator $\hat{\textbf{M}}_{s,a}$. Then we know that
$$\hat{\textbf{M}}_{s,a} = \frac{1}{N_{traj}(s,a)} \sum_n \hat{\prob}_{n,1}(\cdot \mid s,a) \hat{\prob}_{n,2}(\cdot \mid s,a)^T$$
and we have

$$\|\hat{\textbf{M}}_{s,a} - \textbf{M}_{s,a}\| < \sqrt{\frac{32}{N_{traj}(s,a)} (2S\log(12)+\log(\frac{2}{\delta}))} + \frac{48G}{d_{min}(s,a)}\left(\frac{1}{4}\right)^{\frac{T_n}{8Gt_{mix}}}$$

\end{restatable}

\begin{remark}
Note that since all trajectories are generated independently of each other and the process that generates them is identical, $\prob(k_n = j \cap \B_n)$ is the same for all $n$. A similar observation holds for many conditional/unconditional probabilities and conditional/unconditional expectations in this proof, and will not be stated again. 
\end{remark}
Assume the lemma for now. The proof is delayed to after the proof of the theorem. We will combine this lemma with Lemma 3 from \citet{poor2022mixdyn}. In the context of their lemma, $p^{(j)} = \prob(k_n = j \mid \B_n)$, $\textbf{y}^{(j)} = \prob_j(\cdot \mid s,a)$. Now, we can use the first term on the right-hand side of the bound in Lemma 3 of \citet{poor2022mixdyn} to get that for any $1 \leq l \leq L$

\begin{equation}\label{eqn:bound-proj-error-rough}
    \|\prob_{j_l}(\cdot \mid s,a) - \textbf{V}_{s,a}\textbf{V}_{s,a}^T\prob_{j_l}(\cdot \mid s,a)\|_2 \leq \sqrt{\frac{2K}{\min_l (\prob(k_n = j_l \mid \B_n))}\|\hat{\textbf{M}}_{s,a} - {\textbf{M}}_{s,a}\|}
\end{equation}

\subsubsection{\textbf{Lower Bounding} \texorpdfstring{$\prob(k_n = j_l \mid \B_n)$}{blah}}

Note that 
$$\prob(k_n = j_l \mid \B_n) = \frac{\prob(k_n = j_l) \prob(\B_n \mid k_n = j_l)}{\prob(\B_n)} \geq f_{j_l} \prob(\B_n \mid k_n = j_l)$$

So, we need only lower bound $\prob(\B_n \mid k_n = j_l)$, for which we will need a lemma. We will use the following crucial lemma several times in our proofs. This is where we use \cite{bin1994mixing}'s blocking technique.

\begin{lemma}\label{lem:blocking-technique}
Consider a function $h$ on segments of a Markov chain with mixing time $t_{mix} = t_{mix}(\frac{1}{4})$ with $C = \sup h$. Consider the joint distribution $\chi$ over the product of the $\sigma$-algebras of $n$ such segments, with marginals $\chi_i$. Let the product distribution of the marginals $\chi_i$ be called $\Xi$. Then for $\lambda = (\frac{1}{4})^{\frac{1}{t_{mix}}}$ and for the minimum distance between consecutive segments being $a_n$, we have
$$|\E_\chi h - \E_\Xi h| \leq 4C(n-1)\lambda^{a_n}$$
\end{lemma}

\begin{proof}
Remember that each of our Markov processes is mixing, so there exists $t_{mix, j} = t_{mix, j}(\frac{1}{4})$ and a stationary distribution $d_j$ so that $TV(P_j^n(x, \cdot), d_j) < \frac{1}{4}$ for $n\geq t_{mix, j}$. Let $t_{mix} = \max_j t_{mix, j}$. Since the decay in total variation distance is multiplicative, $TV(P_j^n(x, \cdot), d_j) < (\frac{1}{4})^{c}$ for all $j$ and $n\geq ct_{mix}$. This implies that

$$\max_j TV(P_j^n(x, \cdot), d_j) < \left( \frac{1}{4} \right)^{\frac{T_n}{4t_{mix}}-1} = 4\lambda^n$$
where $\lambda = (\frac{1}{4})^{\frac{1}{t_{mix}}}$

This means that we satisfy the definition of $V$-geometric ergodicity from \citet{vidyasagar2010betamixing}, with $V$ being the constant function with value $1$, $\mu = 4$ and $\lambda$ as above. That means that any of our processes is beta-mixing by (the proof of) Theorem 3.10 from the text and 
$$\beta_n \leq \mu \lambda^n = 4\lambda^n$$

we employ an argument analogous to the setup and argument used to prove Lemma 4.1 of \citet{bin1994mixing}, merely with $H_i$'s replaced by the segments of arbitrary length instead of $a_n$-sized blocks while $T_i$'s stay at $a_n$ sized blocks. Then, $Q$ from Corollary 2.7 is the probability distribution of the segments here, $\Omega_i$ from Corollary 2.7 is the real vector space of the same dimension as the length of the $i^{th}$ segment, $\Sigma_i$ is the product Borel field on this vector space and $m$ in the theorem is the number of segments $n$ here (note that $n$ is called $\mu_n$ in Lemma 4.1). $\Tilde{Q}$ is the product distribution over the marginals of $Q$, as in the theorem. Note that $\beta(Q)$ from Corollary 2.7 used in the proof remains less than $\beta_{a_n}$. Now we can directly quote Corollary 2.7 to conclude that
$$|\E_\chi h - \E_\Xi h| \leq C(n-1)\beta_{a_n} \leq 4C(n-1)\lambda^{a_n}$$

\end{proof}

Define
$$h = \ind_{(c_{n,1}c_{n,2}=0)}$$

We are now ready to bound $P(\B_n \mid k_n = j) = P(c_{n,1}c_{n,2}=0 \mid k_n = j)$. Consider the joint distribution over the segments $\Omega_1$ and $\Omega_2$ of a trajectory sampled from hidden label $j$. Call this $\chi$ and let its marginals on $\Omega_i$ be $\chi_i$. Let the product distribution of its marginals be $\Xi := \chi_1\times \chi_2$. Notice that then $$\E_{\Xi} h = P(c_{n,1} = 0 \mid k_n = j)P(c_{n,2} = 0 \mid k_n = j)$$ by definition of $\Xi$. Also, clearly we have

$$\E_\chi h = P(c_{n,1}c_{n,2}=0 \mid k_n = j)$$

Now, using Lemma~\ref{lem:blocking-technique}, we get that for $C = \sup h = 1$ and $n = 2$, we have the following inequality.
\begin{align*}
    |P(c_{n,1}c_{n,2}=0 \mid k_n = j) -P(c_{n,1} = 0 \mid k_n = j)P(c_{n,2} = 0 \mid k_n = j)| = |\E_\chi h -  \E_\Xi h| \leq 4\lambda^{\frac{T_n}{4}} \numberthis \label{eqn:prob-cn1-cn2-bound}
\end{align*}

Additionally, for $i=1,2$, if $d_{t,j}(s,a)$ is the distribution at time $t$, the following is obtained by the definition of mixing times.
\begin{align*}
    \prob(c_{n,i} \neq 0 \mid k_n=j) &\leq (1-d_{(2i-1)T,j}(s,a))\\
    &\leq (1-d_{j}(s,a) + TV(d_{(2i-1)T,j}, \pi))\\
    &\leq (1-d_{min}(s,a)+4\lambda^{\frac{T_n}{4}})\\
    &\leq \left(1 - \frac{d_{min}(s,a)}{2} \right) \numberthis \label{eqn:cn-i-bound}
\end{align*}
where the last inequality holds for $T_n > 4t_{mix}\frac{\log (8/ d_{min}(s,a))}{\log 4}$. This allows us to use inequality~\ref{eqn:prob-cn1-cn2-bound} and 

\begin{align*}
    P(c_{n,1}c_{n,2}=0 \mid k_n = j) &\leq 4\lambda^{\frac{T_n}{4}} + P(c_{n,1} = 0 \mid k_n = j)P(c_{n,2} = 0 \mid k_n = j)\\
    &\leq 4\lambda^{\frac{T_n}{4}} + \left(1 - \frac{d_{min}(s,a)}{2} \right)^2\\
    &\leq 1-d_{min}(s,a) +\frac{d_{min}(s,a)^2}{4} + 4\lambda^{\frac{T_n}{4}}\\
    &\leq 1-d_{min}(s,a) +\frac{d_{min}(s,a)}{4} + 4\lambda^{\frac{T_n}{4}}\\
    &\leq 1 - \frac{d_{min}(s,a)}{2} \numberthis \label{eqn:cm1-cm2-bound}
\end{align*}

where the last inequality holds for $T_n> 4t_{mix}\frac{\log (16/ d_{min}(s,a))}{\log 4}$. We conclude that for $T_n> 4t_{mix}\frac{\log (16/ d_{min}(s,a))}{\log 4}$, and $j=j_l$ for some $l$,
$$P(\B_n \mid k_n = j) \geq \frac{d_{min}(s,a)}{2}$$

And so,
\begin{align*}
    \min_l f_{j_l} (\prob(k_n = j_l \mid \B_n)) &\geq \min_l f_{j_l} (\prob(k_n = j_l \cap \B_n))\\
    &\geq \min_l f_{j_l} (\prob(\B_n \mid k_n=j) \prob(k_n=j))\\
    &\geq \frac{f_{min}d_{min}(s,a)}{2}
\end{align*}

We can thus conclude that for $T_n> 4t_{mix}\frac{\log (16/ d_{min}(s,a))}{\log 4}$,

\begin{equation}\label{eqn:bound-proj-error}
\|\prob_{j_l}(\cdot \mid s,a) - \textbf{V}_{s,a}\textbf{V}_{s,a}^T\prob_{j_l}(\cdot \mid s,a)\|_2 \leq \sqrt{\frac{4K}{f_{min}d_{min}(s,a)}\|\hat{\textbf{M}}_{s,a} - {\textbf{M}}_{s,a}\|}
\end{equation}

\subsubsection{\textbf{Absorbing the extra terms into the exponent of $1/4$}}

Now remember from Lemma~\ref{lem:h-hat-h-tilde} that
$$\|\hat{\textbf{M}}_{s,a} - \textbf{M}_{s,a}\| < \sqrt{\frac{32}{N_{traj}(s,a)} (2S\log(12)+\log(\frac{2}{\delta}))} + \frac{48G}{d_{min}(s,a)}\left(\frac{1}{4}\right)^{\frac{T_n}{8Gt_{mix}}}$$

Notice that for $\frac{T_n}{8t_{mix}} >\frac{G\log (48G/ d_{min}(s,a))}{\log 2} > \frac{\log (16/ d_{min}(s,a))}{2\log 4}$, we have that 

\begin{align*}
    \frac{48G}{d_{min}(s,a)}\left(\frac{1}{4}\right)^{\frac{T_n}{8Gt_{mix}}} &= \frac{48G}{d_{min}(s,a)}\left(\frac{1}{4}\right)^{\frac{T_n}{16Gt_{mix}}}\left(\frac{1}{4}\right)^{\frac{T_n}{16Gt_{mix}}}\\
    &= \frac{48G}{d_{min}(s,a)}\left(\frac{1}{2}\right)^{\frac{T_n}{8Gt_{mix}}}\left(\frac{1}{2}\right)^{\frac{T_n}{8Gt_{mix}}}\\
    &\leq \left(\frac{1}{2}\right)^{\frac{T_n}{8Gt_{mix}}}
\end{align*}

\subsubsection{\textbf{Bounding the concentration term}}

We finally need to bound $N_{traj}(s,a)$ from below to bound the first term in this sum. Note that $\E[N_{traj}(s,a)] \geq N_{sub}(1-P(c_{n,1}c_{n,2} = 0)) \geq N_{sub}\frac{d_{min}(s,a)}{2}$ from equation~\ref{eqn:cm1-cm2-bound} above. Now, by Hoeffding's inequality, we have

\begin{align*}
    \prob\left(N_{traj}(s,a) < N_{sub}\frac{d_{min}(s,a)}{4}\right) &= \prob\left(N_{traj}(s,a) < N_{sub}\frac{d_{min}(s,a)}{2} - N_{sub}\frac{d_{min}(s,a)}{4}\right)\\
    &\leq \prob\left(N_{traj}(s,a) < \E[N_{traj}(s,a)] - N_{sub}\frac{d_{min}(s,a)}{4}\right)\\
    &= \prob
    \left(\sum_{n \in \N_{sub}} \ind_{c_{n,1}c_{n,2} \neq 0} < N_{sub}\E[\ind_{c_{n,1}c_{n,2} \neq 0}] - N_{sub}\frac{d_{min}(s,a)}{4}\right)\\
    &\leq \exp\left(-\frac{d_{min}(s,a)^2N_{sub}}{8}\right)
\end{align*}

This is less than $\delta$ for $N_{sub} \geq \frac{8}{d_{min}(s,a)^2}\log\left(\frac{1}{\delta}\right)$. 

Combining this with equation~\ref{eqn:bound-proj-error} and splitting the two $\delta$, we have our result that
$$\|\prob_j(\cdot \mid s,a) - \textbf{V}_{s,a}\textbf{V}_{s,a}^T\prob_j(\cdot \mid s,a)\|_2 $$ 
is bounded above by 
$$\sqrt{\frac{4K}{f_{min}d_{min}(s,a)} \left( \sqrt{\frac{128}{N_{sub} \cdot d_{min}(s,a)} (2S\log(12)+\log(4/\delta))} +  \left(\frac{1}{2} \right)^{\frac{T_n}{8Gt_{mix}}}\right)}$$

for $N_{sub} \geq \frac{8}{d_{min}(s,a)^2}\log\left(\frac{1}{\delta}\right)$ and $\frac{T_n}{8t_{mix}} > \frac{G\log (48G/ d_{min}(s,a))}{\log 2}$.
\end{proof}

\subsection{Proof of Lemma~\ref{lem:h-hat-h-tilde}}

We recall Lemma \ref{lem:h-hat-h-tilde}.
\HatVsTilde*

\begin{proof} We divide the proof into subsections. We first remind ourselves that the estimator $\hat{\textbf{M}}_{s,a}$ is given by the matrix

\begin{align*}
\hat{\textbf{M}}_{s,a} = \frac{1}{N_{traj}(s,a)} \sum_{n \in \N_{traj}(s,a)} \left( \hat{\mathbb{P}}_{n,1}(\cdot |s,a) \hat{\mathbb{P}}_{n,2}(\cdot |s,a)^T \right)
 \end{align*}

\subsubsection{\textbf{Estimating} \texorpdfstring{$\E[\hat{\textbf{M}}_{s,a}]$}{blah}}

We will split the expectation into the desired term and the error coming from correlation between the two segments $\Omega_1$ and $\Omega_2$. Remember that for brevity of notation, let $c_{n,i} := N(n,i,s,a)$, $\textbf{w}_{n,i} := N(n,i,s,a, \cdot)$. Call the estimate from each trajectory a random variable $\hat{\textbf{M}}_{n,s,a}$, that is
$$\hat{\textbf{M}}_{n,s,a} = \frac{\textbf{w}_{n,1}\textbf{w}_{n,2}^T}{c_{n,1}c_{n,2}} $$

Now $$\hat{\textbf{M}}_{s,a} = \sum_{n \in \N_{traj}(s,a)} \frac{1}{N_{traj}(s,a)} \hat{\textbf{M}}_{n,s,a}$$

Remember that $$\hat{\prob}_{n,i}(\cdot \mid s,a) := \frac{\textbf{w}_{n,i}}{c_{n,i}} \ind_{c_{n,i} \neq 0}$$
Let $k_n$ be the hidden label for trajectory $n$, as usual. Define the event $\B_n$ to be $n \in \N_{traj}(s,a)$, which is the same as $c_{n,1}c_{n,2} \neq 0$. We establish the following equality, essentially just defining the quantity $\Mix(j)$.
\begin{align*}
    \E[\hat{\textbf{M}}_{n,s,a} \mid \B_n]&= \E \left[\frac{\textbf{w}_{n,1}\textbf{w}_{n,2}^T}{c_{n,1}c_{n,2}}  \middle\vert \B_n\right]
    \\
    &= \sum_{j=1}^K \prob(k_n = j \mid \B_n) \E \left[  \frac{\textbf{w}_{n,1}\textbf{w}_{n,2}^T}{c_{n,1}c_{n,2}} \  \middle\vert k_n=j, \B_n\right]
    \\
    &=  \sum_{j=1}^K \prob(k_n = j \mid \B_n) \prob_j(\cdot \mid s,a)\prob_j(\cdot \mid s,a)^T + \sum_{j=1}^K \prob(k_n = j \mid \B_n) \Mix(j)
    \numberthis \label{eqn:est-h-hat-m-sa}
\end{align*}

where $\Mix(j) =  \E \left[  \frac{\textbf{w}_{n,1}\textbf{w}_{n,2}^T}{c_{n,1}c_{n,2}} \  \middle\vert k_n=j, \B_n\right] - \prob_j(\cdot \mid s,a)\prob_j(\cdot \mid s,a)^T$. Notice that this has connotations of covariance. Now note the following chain of equations.
\begin{align*}
    \E[\hat{\textbf{M}}_{s,a} \mid \N_{traj}(s,a)] &= \E\left[ \sum_{n \in \N_{traj}(s,a)} \frac{1 }{N_{traj}(s,a)}\hat{\textbf{M}}_{n,s,a} \middle\vert \N_{traj}(s,a) \right]  \\
    &= \sum_{n \in \N_{traj}(s,a)} \frac{1}{N_{traj}(s,a)} \E\left[\hat{\textbf{M}}_{n,s,a} \middle\vert  \N_{traj}(s,a) \right]\\
    &= \sum_{n \in \N_{traj}(s,a)} \frac{1}{N_{traj}(s,a)} \E\left[\hat{\textbf{M}}_{n,s,a} \middle\vert \ind_{\B_1}, \ind_{\B_2},\dots \ind_{\B_{N_{sub}}} \right]\\
    &= \sum_{n \in \N_{traj}(s,a)} \frac{1}{N_{traj}(s,a)} \E\left[\hat{\textbf{M}}_{n,s,a} \middle\vert \B_n \right]\\
    &= \E\left[\hat{\textbf{M}}_{n,s,a} \middle\vert \B_n \right]\\
    &= \sum_{j=1}^K \prob(k_n = j \mid \B_n) \prob_j(\cdot \mid s,a)\prob_j(\cdot \mid s,a)^T + \sum_{j=1}^K \prob(k_n = j \mid \B_n) \Mix(j)\\
    &= \textbf{M}_{s,a} + \sum_{j=1}^K \prob(k_n = j \mid \B_n) \Mix(j) \numberthis \label{eqn:est-m-hat-sa}
\end{align*}

Here, the third equality is because the set $\N_{traj}(s,a)$ is exactly described by the indicators listed, and they generate the same $\sigma$-algebra, The fourth equality holds since all trajectories are independent and so conditioning on events in other trajectories doesn't affect the expectation of $\hat{\textbf{M}}_{n,s,a}$. The fifth equality is because $\E[\hat{\textbf{M}}_{n,s,a} \mid \B_n]$ is the same for all $n$ as determined above (in fact, we have shown that it is a constant random variable).

\subsubsection{\textbf{Setup for the main bound}}

We have that $$\textbf{M}_{s,a} = \sum_{j=1}^K \prob(k_n = j \mid \B_n) \prob_j(\cdot \mid s,a) \prob_j(\cdot \mid s,a)^T$$

By equation~\ref{eqn:est-m-hat-sa},
\begin{align*}
    \|\hat{\textbf{M}}_{s,a} - \textbf{M}_{s,a}\| &\leq \|\hat{\textbf{M}}_{s,a} - \E[\hat{\textbf{M}}_{s,a} \mid \N_{traj}(s,a)] \| + \sum_{j=1}^K \prob(k_n=j \mid \B_n) \left\Vert \Mix(j) \right\Vert\\
    &\leq \|\hat{\textbf{M}}_{s,a} - \E[\hat{\textbf{M}}_{s,a}  \mid \N_{traj}(s,a)] \| + 
    \left( \sum_{j=1}^K \prob(k_n=j \mid \B_n) \right)\max_j \left\Vert \Mix(j) \right\Vert\\
    &= \|\hat{\textbf{M}}_{s,a} - \E[\hat{\textbf{M}}_{s,a}  \mid \N_{traj}(s,a)] \| + \max_j \left\Vert \Mix(j) \right\Vert
\end{align*}

The first term represents the error in concentration across trajectories and the second term represents the correlation between the two segments $\Omega_1$ and $\Omega_2$ in the same trajectory. We bound the first using a covering argument and use Bin Yu's work to bound the other.

\subsubsection{\textbf{Covering argument to bound} \texorpdfstring{$\hat{\textbf{M}}_{s,a} - \E[\hat{\textbf{M}}_{s,a}]$}{blah}}\label{ssec:covering-arg-subspace-est}

We will need this conditional version of Hoeffding's inequality for this section. Note that this is not quite the Azuma-Hoeffding inequality with a constant filtration due to the conditional probability involved, as well as due to the conditional independence needed.
\begin{lemma}\label{lem:conditional-hoeffdings}
Consider a $\sigma$-algebra $\F$ and let $A_i \leq B_i$ be random variables measurable over it. If random variables $X_i$ are almost surely bounded in $[A_i,B_i]$ and are conditionally independent over some $\sigma$-algebra $\F$, then the following inequalities hold for $S_n = \sum_{i=1}^n X_i$

\begin{align*}
    \prob \left(S_n - \E[S_n \mid \F] > \epsilon \middle\vert \F \right) &\leq \exp\left( -\frac{2\epsilon}{\sum_i (B_i-A_i)^2}\right)\\
    \prob \left(S_n - \E[S_n \mid \F] < -\epsilon \middle\vert \F \right) &\leq \exp\left( -\frac{2\epsilon^2}{\sum_i (B_i-A_i)^2}\right)
\end{align*}
\end{lemma}
\begin{proof}
The proof is essentially a repeat of one of the standard proofs of Hoeffding's inequality. Note that we have the conditional Markov inequality $\prob(X \geq a \mid \F) \leq \frac{1}{a}\E[X \geq a \mid \F]$, shown exactly the way Markov's inequality is shown. Now, we have the following chain of inequalities.
\begin{align*}
    \prob(\left(S_n - \E[S_n \mid \F] > \epsilon \middle\vert \F \right) &= e^{-s\epsilon}\E[e^{S_n-\E[S_n \mid \F]} \mid \F]\\
    &= e^{-s\epsilon} \prod_{i=1}^n \E[e^{X_i - \E[X_i \mid \F]} \mid \F]
\end{align*}

We now show a conditional expectation version of Hoeffding's lemma by repeating the steps for a standard proof to show that $\E[e^{X-\E[X\mid \F]}\mid \F] \leq \frac{\lambda^2 (B-A)^2}{8}$ for random variables $A \leq B$ measurable over $\F$ and $X \in [A,B]$ almost surely. Note that by convexity of $e^{\lambda x}$, we have the following for $x \in [A,B]$ at any value of $A$ and $B$.
\begin{align*}
    e^{\lambda x} \leq \frac{B-x}{B-A}e^{\lambda B} + \frac{x-A}{B-A} e^{\lambda B}
\end{align*}
WLOG, we can replace $X$ by $X - \E[X \mid \F]$ and assume $\E[X\mid \F] = 0$. In that case, we note the following inequality, where we define for any fixed value of $A$ and $B$ the function $L(y):= \frac{yA}{B-A} + \log\left(1 + \frac{A-e^yB}{B-A} \right)$.
\begin{align*}
    \E[e^{\lambda X} \mid \F] &\leq \frac{B-\E[X \mid \F]}{B-A}e^{\lambda A} + \frac{\E[X\mid \F]-A}{B-A} e^{\lambda B}\\
    &= \frac{B}{B-A}e^{\lambda A} + \frac{-A}{B-A} e^{\lambda B}\\
    &= e^{L(\lambda(B-A))}
\end{align*}
Basic computations involving Taylor's theorem from a standard proof of Hoeffding's inequality show that $L(y) \leq \frac{y^2}{8}$ for any value of $A, B$. This gives us the condition version of Hoeffding's lemma, $\E[e^{X-\E[X\mid \F]}\mid \F] \leq \frac{\lambda^2 (B-A)^2}{8}$. This allows us to establish the following chain of inequalities.
\begin{align*}
    \prob(\left(S_n - \E[S_n \mid \F] > \epsilon \middle\vert \F \right) &= e^{-s\epsilon} \prod_{i=1}^n \E[e^{X_i - \E[X_i \mid \F]} \mid \F]\\
    &\leq \exp(-s\epsilon)\prod_{i=1}^n\exp\left(\frac{s^2(B_i-A_i)^2}{8}\right)\\
    &= \exp\left( -s\epsilon + \sum_{i=1}^n \frac{s^2(B_i-A_i)^2}{8}\right)\\
\end{align*}

Since $s$ is arbitrary, we can pick $s = \frac{4\epsilon}{\sum_i (B_i-A_i)^2}$ above to get an upper bound of $\exp\left(-\frac{2\epsilon^2}{\sum_{i=1}^n (B_i-A_i)^2}\right) $. The other inequality is proved analogously.
\end{proof}

We now show that the first term from the previous section concentrates. Pick $\textbf{u}, \textbf{v} \in \Sph^{S-1}$, that is they lie in the unit Euclidean norm sphere in $\Real^S$. We need only bound this term when $N_{traj}(s,a) \neq 0$, as otherwise the lemma holds vacuously.

Note that 
\begin{align*}
    \hat{\textbf{M}}_{s,a} - \E[\hat{\textbf{M}}_{s,a} \mid \N_{traj}(s,a)] &= \sum_{n \in \mathcal{N}_{traj}(s,a)} \frac{\hat{\textbf{M}}_{n,s,a}- \E[\hat{\textbf{M}}_{n,s,a} \mid \N_{traj}(s,a)]}{N_{traj}(s,a)}
\end{align*}

Now we set up our covering argument. Consider a covering of $\Sph^{S-1}$ by balls of radius $\frac{1}{4}$. We will need at most $12^S$ such balls and if $C$ is the set of their centers, then for any matrix $X$, the following holds in regard to its norm.
\begin{align*}
    \|X\| = \sup_{\textbf{u,v}\in \Sph^{S-1}} |\textbf{u}^TX\textbf{v}| \leq 2 \max_{\textbf{u,v} \in C} |\textbf{u}^TX\textbf{v}| \leq 2 \|X\| \numberthis \label{eqn:covering-norm-bound}
\end{align*}

For any pair $\textbf{u,v} \in C$, note that
\begin{align*}
    |\textbf{u}^T\hat{\textbf{M}}_{n,s,a}\textbf{v}| &= \left\vert\textbf{u}^T\frac{\textbf{w}_{n,1}}{c_{n,1}}\right\vert
    \left\vert\frac{\textbf{w}_{n,2}^T}{c_{n,2}}\textbf{v}\right\vert\ind_{c_{n,1}c_{n,2}\neq 0}\\
    &\leq \|\textbf{u}\|_2
\|\textbf{v}\|_2
\left\Vert \frac{\textbf{w}_{n,1}}{c_{n,1}}\right\Vert_2
\left\Vert \frac{\textbf{w}_{n,2}}{c_{n,2}}\right\Vert_2\\
&\leq \left\Vert \frac{\textbf{w}_{n,1}}{c_{n,1}}\right\Vert_1 \left\Vert \frac{\textbf{w}_{n,2}}{c_{n,2}}\right\Vert_1\\
&\leq 1
\end{align*}

and so $|\textbf{u}^T\E[\hat{\textbf{M}}_{n,s,a}]\textbf{v}| \leq \E[|\textbf{u}^T\hat{\textbf{M}}_{n,s,a}\textbf{v}|] \leq 1$. A little thought shows that the estimates $\hat{\textbf{M}}_{n,s,a}$ are independent for $n \in \N_{traj}(s,a)$ when conditioned on the $\N_{traj}(s,a)$. 

$$\prob\left( \left\vert \sum_{n \in \mathcal{N}_{traj}(s,a)} \frac{1}{N_{traj}(s,a)} \textbf{u}^T(\hat{\textbf{M}}_{n,s,a} - \E[\hat{\textbf{M}}_{n,s,a} \mid \N_{traj}(s,a)])\textbf{v} \right\vert > \frac{\epsilon}{4} \middle\vert \N_{traj}(s,a) \right) < 2e^{-\frac{\epsilon^2N_{traj}(s,a)}{32}}$$

Doing this for all $12^{2S}$ pairs $\textbf{u,v}$, we use inequality~\ref{eqn:covering-norm-bound} to get that the conditionalprobability given by

$$\prob\left( \left\Vert \sum_{n \in \N_{traj}(s,a)} \frac{1}{N_{traj}(s,a)} \hat{\textbf{M}}_{n,s,a} - \E[\hat{\textbf{M}}_{n,s,a} \mid \N_{traj}(s,a)] \right\Vert > \frac{\epsilon}{2} \middle\vert \N_{traj}(s,a) \right) $$
 is bounded above by the following expression.
\begin{align*}
    &\prob\left(\exists \textbf{u,v} \in C; \left\vert \sum_{n \in \N_{traj}(s,a)} \frac{1}{N_{traj}(s,a)} \textbf{u}^T(\hat{\textbf{M}}_{n,s,a} - \E[\hat{\textbf{M}}_{n,s,a} \mid \N_{traj}(s,a)])\textbf{v} \right\vert > \frac{\epsilon}{4} \middle\vert \N_{traj}(s,a) \right)\\
    &\leq \sum_{\textbf{u,v} \in C} \prob\left( \left\vert \sum_{n \in \N_{traj}(s,a)} \frac{1}{N_{traj}(s,a)} \textbf{u}^T(\hat{\textbf{M}}_{n,s,a} - \E[\hat{\textbf{M}}_{n,s,a} \mid \N_{traj}(s,a)])\textbf{v} \right\vert > \frac{\epsilon}{4} \middle\vert \N_{traj}(s,a) \right)\\
    &< 2*12^{2S}*e^{-\frac{\epsilon^2N_{traj}(s,a)}{32}}
\end{align*}

This is less than $\delta$ for $N_{traj}(s,a) > \frac{32}{\epsilon^2} (2S\log(12)+\log(\frac{2}{\delta}))$. Since this holds for such values of $N_{traj}(s,a)$ irrespective of $\N_{traj}(s,a)$, we can conclude that for $N_{traj}(s,a) > \frac{32}{\epsilon^2} (2S\log(12)+\log(\frac{2}{\delta}))$, with probability universally greater than $1-\delta$, 

$$\|\hat{\textbf{M}}_{s,a} - \textbf{M}_{s,a}\| < \frac{\epsilon}{2} + \max_j \left\Vert \Mix(j) \right\Vert$$

Alternatively, this establishes that with probability greater than $1-\delta$, we have the following inequality involving the random variables $\hat{\textbf{M}}_{s,a}$ and $N_{traj}(s,a)$.

$$\|\hat{\textbf{M}}_{s,a} - \textbf{M}_{s,a}\| < \sqrt{\frac{32}{N_{traj}(s,a)} (2S\log(12)+\log(\frac{2}{\delta}))} + \max_j \left\Vert \Mix(j) \right\Vert$$

\subsubsection{\textbf{Bounding the mixing term}}\label{sssec:mix-bound-subspace-est}

We now resolve the last remaining thread, which is that of bounding the mixing term. Let's fix a $j$ for this section, since proving our upper bounds for arbitrary $j$ is sufficient. Let the joint distribution of the observations under label $j$ be $\chi$. Let its marginal on the segment $\Omega_i$ be $\chi_i$. Let the marginals on each of the $G$ single-step sub-blocks be $\chi_{i,g}$. Denote the product distribution $\prod_g \chi_{i,g}$ by $Q_i$. 
\begin{align*}
    \|\Mix(j)\| &=  \left\Vert \E \left[  \frac{\textbf{w}_{n,1}\textbf{w}_{n,2}^T}{c_{n,1}c_{n,2}} \  \middle\vert k_n=j, \B_n\right] -\prob_j(\cdot \mid s,a)\prob_j(\cdot \mid s,a)^T\right\Vert\\
    \
    &= \left\Vert \frac{1}{\prob(\B_n)} \E_\chi \left[  \frac{\textbf{w}_{n,1}\textbf{w}_{n,2}^T}{c_{n,1}c_{n,2}} \ind_{\B_n}\right] - \prob_j(\cdot \mid s,a)\prob_j(\cdot \mid s,a)^T \right\Vert\\
    \
    &\leq  \frac{1}{\prob(\B_n)}\left\Vert \E_\chi \left[  \frac{\textbf{w}_{n,1}\textbf{w}_{n,2}^T}{c_{n,1}c_{n,2}} \ind_{\B_n}\right] - \E_{\chi_1}\left[ \frac{\textbf{w}_{n,1}}{c_{n,1}} \ind_{c_{n,1} \neq 0}\right]\E_{\chi_2}\left[ \frac{\textbf{w}_{n,2}^T}{c_{n,2}} \ind_{c_{n,2} \neq 0}\right]  \right\Vert\\
    &\hspace{5ex}+  \frac{1}{\prob(\B_n)}\left\Vert \E_{\chi_1}\left[ \frac{\textbf{w}_{n,1}}{c_{n,1}} \ind_{c_{n,1} \neq 0}\right]\E_{\chi_2}\left[ \frac{\textbf{w}_{n,2}^T}{c_{n,2}} \ind_{c_{n,2} \neq 0}\right] - \prob_{Q_1}(c_{n,1}\neq 0)\prob_{Q_2}(c_{n,2}\neq 0) \prob_j(\cdot \mid s,a)\prob_j(\cdot \mid s,a)^T\right\Vert\\
    &\hspace{5ex} +  \frac{1}{\prob(\B_n)} \left\Vert(\prob_{Q_1}(c_{n,1}\neq 0)\prob_{Q_2}(c_{n,2}\neq 0) - \prob(c_{n,1}\neq 0)\prob(c_{n,2}\neq 0))\prob_j(\cdot \mid s,a)\prob_j(\cdot \mid s,a)^T (\right\Vert\\
    &\hspace{5ex} +  \frac{1}{\prob(\B_n)} \left\Vert\left( \prob(c_{n,1}\neq 0)\prob(c_{n,2} \neq 0) - \prob(\B_n)\right) \prob_j(\cdot \mid s,a)\prob_j(\cdot \mid s,a)^T\right\Vert\\
    \
    &\leq \frac{1}{\prob(\B_n)}\left\Vert \E_\chi \left[  \frac{\textbf{w}_{n,1}\textbf{w}_{n,2}^T}{c_{n,1}c_{n,2}} \ind_{\B_n}\right] - \E_{\chi_1}\left[ \frac{\textbf{w}_{n,1}}{c_{n,1}} \ind_{c_{n,1} \neq 0}\right]\E_{\chi_2}\left[ \frac{\textbf{w}_{n,2}^T}{c_{n,2}} \ind_{c_{n,2} \neq 0}\right]  \right\Vert\\
    &\hspace{5ex}+  \frac{1}{\prob(\B_n)}\left\Vert \E_{\chi_1}\left[ \frac{\textbf{w}_{n,1}}{c_{n,1}} \ind_{c_{n,1} \neq 0}\right]\E_{\chi_2}\left[ \frac{\textbf{w}_{n,2}^T}{c_{n,2}} \ind_{c_{n,2} \neq 0}\right] - \prob_{Q_1}(c_{n,1}\neq 0)\prob_{Q_2}(c_{n,2}\neq 0) \prob_j(\cdot \mid s,a)\prob_j(\cdot \mid s,a)^T\right\Vert\\
    &\hspace{5ex} +  \frac{1}{\prob(\B_n)} \left\vert\prob_{Q_1}(c_{n,1}\neq 0)\prob_{Q_2}(c_{n,2}\neq 0) - \prob_{\chi_1}(c_{n,1}\neq 0)\prob_{\chi_2}(c_{n,2}\neq 0)\right\vert\\
    &\hspace{5ex} +  \frac{1}{\prob(\B_n)} \left\vert\prob_{\chi_1}(c_{n,1}\neq 0)\prob_{\chi_2}(c_{n,2} \neq 0) - \prob(\B_n)\right\vert \numberthis \label{eqn:mixing-term-inequality}
\end{align*}

Here, in the last inequality, we used the fact that $\|\prob_j(\cdot \mid s,a)\|_2 \leq \|\prob_j(\cdot \mid s,a)\|_1 = 1$ and $\|\textbf{a}\textbf{a}^T\| \leq \|\textbf{a}\|_2\|\textbf{a}\|_2$. Also note that $\prob_{\chi_i}(c_{n,i} \neq 0) = \prob_\chi(c_{n,i} \neq 0) = \prob(c_{n,i} \neq 0)$.

Intuitively, the first term represents mixing of the expectation across the two segments, the second term represents mixing of the expectations across the single-step sub-blocks inside segments, the third term represents mixing of the observation probabilities across the single-step sub-blocks inside segments, and the fourth term represents mixing of the observation probabilities across the two segments. In short, the first and fourth represent segment-level mixing while the second and third represent sub-block-level mixing.

\textbf{Bounding the first term (segment-level mixing)}

We will use \cite{bin1994mixing}'s blocking technique again, invoking Lemma~\ref{lem:blocking-technique}. Pick an arbitrary $\textbf{u,v}\in \Sph^{S-1}$. Recall that
$$\hat{\prob}_{n,i}(\cdot \mid s,a) := \frac{\textbf{N}(n,i,s,a, \cdot)}{N(n,i,s,a)} \ind_{N(n,i,s,a) \neq 0} = \frac{\textbf{w}_{n,i}}{c_{n,i}} \ind_{c_{n,i} \neq 0}$$
Consider the real-valued random variable 
\begin{align*}
    h_{\textbf{u,v}} &:= \textbf{u}^T \left( \frac{\textbf{w}_{n,1}\textbf{w}_{n,2}^T}{c_{n,1}c_{n,2}} \ind_{\B_n} \right)\textbf{v}
\end{align*}

We have the following basic computations for expectations. Remember that $\ind_{\B_n} = \ind_{c_{n,1}c_{n,2}\neq 0} = \ind_{c_{n,1} \neq 0} \ind_{c_{n,2} \neq 0}$.

$$\E_\chi h_{\textbf{u,v}} = \textbf{u}^T \left( \E_\chi \left[  \frac{\textbf{w}_{n,1}\textbf{w}_{n,2}^T}{c_{n,1}c_{n,2}} \ind_{\B_n} \right] \right)\textbf{v}$$
and
$$\E_{\chi_1\times\chi_2} h_{\textbf{u,v}} = \textbf{u}^T \left( \E_{\chi_1}\left[\frac{\textbf{w}_{n,1}}{c_{n,1}} \ind_{c_{n,1} \neq 0}\right]  \E_{\chi_2}\left[\frac{\textbf{w}_{n,2}^T}{c_{n,2}} \ind_{c_{n,2} \neq 0}\right] \right)\textbf{v}$$
This allows us to establish the following relation.
$$\sup_{\textbf{u,v} \in \Sph^{S-1}} \left\vert \E_\chi h_{\textbf{u,v}} - \E_{\chi_1\times\chi_2} h_{\textbf{u,v}} \right\vert = \left\Vert \E_\chi \left[  \frac{\textbf{w}_{n,1}\textbf{w}_{n,2}^T}{c_{n,1}c_{n,2}} \ind_{\B_n}\right] - \E_{\chi_1}\left[ \frac{\textbf{w}_{n,1}}{c_{n,1}} \ind_{c_{n,1} \neq 0}\right]\E_{\chi_2}\left[ \frac{\textbf{w}_{n,2}^T}{c_{n,2}} \ind_{c_{n,2} \neq 0}\right]  \right\Vert$$

Now, we want to use Lemma~\ref{lem:blocking-technique}. Note the following upper bound.
\begin{align*}
    |h_{\textbf{u,v}}| &\leq  \|\textbf{u}\|_2 \left\Vert \frac{\textbf{w}_{n,1}}{c_{n,1}} \right\Vert_2 \left\Vert \frac{\textbf{w}_{n,2}}{c_{n,2}} \right\Vert_2 \|\textbf{v}\|_2\\
    &\leq \left\Vert \frac{\textbf{w}_{n,1}}{c_{n,1}} \right\Vert_1 \left\Vert \frac{\textbf{w}_{n,2}}{c_{n,2}} \right\Vert_1\\
    &= 1
\end{align*}

So, we can use Lemma~\ref{lem:blocking-technique} with $C = C_{\textbf{u,v}} := \sup h_{\textbf{u,v}}$ and $n = 2$ for any $\textbf{u,v} \in \Sph^{S-1}$, giving us the following inequality.

\begin{equation}\label{eqn:mixing-bound-one}
    \left\vert \E_\chi h_{\textbf{u,v}} - \E_{\chi_1\times\chi_2} h_{\textbf{u,v}} \right\vert \leq 4\lambda^{\frac{T_n}{4}}
\end{equation}

Since inequality~\ref{eqn:mixing-bound-one} holds for any $\textbf{u}, \textbf{v} \in \Sph^{S-1}$, we can take the supremum over such $\textbf{u}, \textbf{v}$ to get the desired inequality below. We also recall that $\prob(\B_n) \geq \frac{d_{min}{s,a}}{2}$ from equation~\ref{eqn:cm1-cm2-bound}.

\begin{align*}
    \frac{1}{\prob(\B_n)}\left\Vert \E_\chi \left[  \frac{\textbf{w}_{n,1}\textbf{w}_{n,2}^T}{c_{n,1}c_{n,2}} \ind_{\B_n}\right] - \E_{\chi_1}\left[ \frac{\textbf{w}_{n,1}}{c_{n,1}} \ind_{c_{n,1} \neq 0}\right]\E_{\chi_2}\left[ \frac{\textbf{w}_{n,2}^T}{c_{n,2}} \ind_{c_{n,2} \neq 0}\right]  \right\Vert &\leq \frac{1}{\prob(\B_n)} 4\lambda^{\frac{T_n}{4}}\\
    &\leq \frac{8\lambda^{\frac{T_n}{4}}}{d_{min}(s,a)} 
\end{align*}

\textbf{Bounding the second term (sub-block-level mixing)}

Remember that the product distribution $\prod_g \chi_{i,g}$ is $Q_i$. First note that, since under $Q_i$, each observation is independent, we have the following expectation.

\begin{align*}
    \E_{Q_i}\left[ \frac{\textbf{w}_{n,i}}{c_{n,i}} \ind_{c_{n,i} \neq 0} \right] &= \E_{Q_i}\left[ \frac{\E_{Q_i}[\textbf{w}_{n,i} \mid c_{n,i}]}{c_{n,i}} \ind_{c_{n,i} \neq 0} \right]\\
    &= \E_{Q_i}\left[ \frac{\prob_j(\cdot \mid s,a)c_{n,i}}{c_{n,i}} \ind_{c_{n,i} \neq 0} \right]\\
    &= \prob_j(\cdot \mid s,a) \prob_{Q_i}(c_{n,i} \neq 0) \numberthis \label{eqn:q-i-estimator-expectation}
\end{align*}

\begin{remark}
Note that this holds crucially because we are working with the product distribution $Q_i$ over the single-step sub-blocks.
\end{remark}

Also, let $h_\textbf{u} = \textbf{u}^T\frac{\textbf{w}_{n,1}}{c_{n,1}}\ind_{c_{n,1}}$ and let $g_\textbf{v} = \frac{\textbf{w}_{n,2}^T}{c_{n,2}}\ind_{c_{n,2}}\textbf{v}$. Then the second term is exactly given by the following expression.

    $$\frac{1}{\prob(\B_n)} \sup_{\textbf{u}, \textbf{v} \in \Sph^{S-1}} \left\vert\E_{\chi_1}[ h_\textbf{u}]\E_{\chi_2}[g_\textbf{v}] - \E_{Q_1}[ h_\textbf{u}]\E_{Q_2}[g_\textbf{v}] \right\vert$$
    
Also note that both $|h_\textbf{u}|$ and $|g_\textbf{v}|$ are bounded by 1. We then have the following chain of inequalities.

\begin{align*}
    \left\vert\E_{\chi_1}[ h_\textbf{u}]\E_{\chi_2}[g_\textbf{v}] - \E_{Q_1}[ h_\textbf{u}]\E_{Q_2}[g_\textbf{v}] \right\vert &\leq \left\vert\E_{\chi_1}[ h_\textbf{u}] - \E_{Q_1}[ h_\textbf{u}]\right\vert |\E_{\chi_2}[g_\textbf{v}]| + |\E_{\chi_2}[g_\textbf{v}] - \E_{Q_2}[g_\textbf{v}]||\E_{Q_1}[ h_\textbf{u}]|\\
    &\leq |\E_{\chi_1}[ h_\textbf{u}] - \E_{Q_1}[ h_\textbf{u}]| + |\E_{\chi_2}[g_\textbf{v}] - \E_{Q_2}[g_\textbf{v}]|
\end{align*}

Since the single step sub-blocks are separated by at least $\frac{T_n}{8G}$ timesteps, we can apply Lemma~\ref{lem:blocking-technique} with $C=1$ and $n=G$ to get bounds on both terms here, since $Q_i = \prod_g \chi_{i,g}$. Also remember that $\prob(\B_n) \geq \frac{d_{min}(s,a)}{2}$ from equation~\ref{eqn:cm1-cm2-bound}.

\begin{align*}
    \frac{1}{\prob(\B_n)} \sup_{\textbf{u}, \textbf{v} \in \Sph^{S-1}} \left\vert\E_{\chi_1}[ h_\textbf{u}]\E_{\chi_2}[g_\textbf{v}] - \E_{Q_1}[ h_\textbf{u}]\E_{Q_2}[g_\textbf{v}] \right\vert &\leq \frac{1}{\prob(\B_n)}\left(4G\lambda^{\frac{T_n}{8G}} + 4G\lambda^{\frac{T_n}{8G}}\right)\\
    &\leq \frac{16G\lambda^{\frac{T_n}{8G}}}{d_{min}(s,a)}
\end{align*}

\textbf{Bounding the third term (sub-block-level mixing)}

Again, note that the third term is given by the following expression.
\begin{align*}
     \frac{1}{\prob(\B_n)} \left\vert\E_{Q_1}[\ind_{c_{n,1}\neq 0}]\E_{Q_2}[\ind_{c_{n,2}\neq 0}] - \E_{\chi_1}[\ind_{c_{n,1}\neq 0}]\E_{\chi_2}[\ind_{c_{n,2}\neq 0}]\right\vert
\end{align*}

We can bound this above using the fact that $|ab-cd| \leq |b||a-c| + |c||b-d|$, to get the following upper bound.
\begin{align*}
    \E_{Q_2}[\ind_{c_{n,2}\neq 0}]\left\vert\E_{Q_1}[\ind_{c_{n,1}\neq 0}] - \E_{\chi_1}[\ind_{c_{n,1}\neq 0}]\right\vert + \E_{\chi_1}[\ind_{c_{n,1}\neq 0}]\left\vert\E_{Q_2}[\ind_{c_{n,2}\neq 0}] - \E_{\chi_2}[\ind_{c_{n,2}\neq 0}]\right\vert
\end{align*}
This in turn is bounded above by the expression below.
$$|\E_{Q_1}[\ind_{c_{n,1}\neq 0}] - \E_{\chi_1}[\ind_{c_{n,1}\neq 0}]| + |\E_{Q_2}[\ind_{c_{n,2}\neq 0}] - \E_{\chi_2}[\ind_{c_{n,2}\neq 0}]|$$

Since indicator functions are bounded above by 1, we can apply Lemma~\ref{lem:blocking-technique} as in the second term ($C=1$, $n=G$) to bound both the differences above. Skipping the routine details, we finally get the following inequality, analogous to the second term.
\begin{align*}
     \frac{1}{\prob(\B_n)} \left\vert\E_{Q_1}[\ind_{c_{n,1}\neq 0}]\E_{Q_2}[\ind_{c_{n,2}\neq 0}] - \E_{\chi_1}[\ind_{c_{n,1}\neq 0}]\E_{\chi_2}[\ind_{c_{n,2}\neq 0}]\right\vert \leq \frac{16G\lambda^{\frac{T_n}{8G}}}{d_{min}(s,a)}
\end{align*}

\textbf{Bounding the fourth term (segment-level mixing)}

Now note that the fourth term is the same as the expression below.
$$\frac{1}{\prob(\B_n)} |\E_{\chi_1}[\ind_{c_{n,1}\neq 0}]\E_{\chi_2}[\ind_{c_{n,2}\neq 0}] - \E_\chi[\ind_{c_{n,1}\neq 0}\ind_{c_{n,2}\neq 0}]| = \frac{1}{\prob(\B_n)} |\E_{\chi_1\times \chi_2}[\ind_{c_{n,1}\neq 0}\ind_{c_{n,2}\neq 0}] - \E_\chi[\ind_{c_{n,1}\neq 0}\ind_{c_{n,2}\neq 0}]|$$

We can now apply Lemma~\ref{lem:blocking-technique} with $C=1$ and $n=2$. The segments are separated by $T$ and $\prob(\B_n) \geq \frac{d_{min}(s,a)}{2}$, giving us the following bound.
$$\frac{1}{\prob(\B_n)} \left\vert\prob_{\chi_1}(c_{n,1}\neq 0)\prob_{\chi_2}(c_{n,2} \neq 0) - \prob(\B_n)\right\vert \leq \frac{8\lambda^{\frac{T_n}{4}}}{d_{min}(s,a)}$$

Combining all these, we get that

\begin{align*}
    \|\hat{\textbf{M}}_{s,a} - \textbf{M}_{s,a}\| &< \sqrt{\frac{32G}{N_{traj}(s,a)} (2S\log(12)+\log(\frac{2}{\delta}))} + \frac{16}{d_{min}(s,a)}\left(\frac{1}{4}\right)^{\frac{T_n}{4t_{mix}}} + \frac{32G}{d_{min}(s,a)}\left(\frac{1}{4}\right)^{\frac{T_n}{8Gt_{mix}}}\\
    &\leq \sqrt{\frac{32}{N_{traj}(s,a)} (2S\log(12)+\log(\frac{2}{\delta}))} + \frac{48G}{d_{min}(s,a)}\left(\frac{1}{4}\right)^{\frac{T_n}{8Gt_{mix}}}\numberthis \label{eqn:h-hat-h-tilde}
\end{align*}

as desired.

\end{proof}
\newpage
\section{Proof of Theorem~\ref{thm:clustering}}

\Clustering*

\begin{proof}

Consider the testing of trajectories $m$ and $n$. Recall that we defined 
$$\dist_{1}(m,n) := \max_{(s,a) \in SA_{\alpha}} \left[ \left( \left(\hat{\prob}_{n,1}(\cdot \mid s,a) - \hat{\prob}_{m,1}(\cdot \mid s,a) \right)^T\textbf{V}_{s,a} \right) \left(\left( \hat{\prob}_{n,2}(\cdot \mid s,a) - \hat{\prob}_{m,2}(\cdot \mid s,a) \right)^T \textbf{V}_{s,a} \right)^T \right]$$ Let $k_m$ be the label of trajectory $m$ and $k_n$ the label of trajectory $n$. According to our assumptions, if $k_m \neq k_n$, then we have an $s,a$ so that $d_{k_m}(s,a), d_{k_n}(s,a) \geq \alpha$ and $\|\prob_{k_m}(\cdot \mid s,a) - \prob_{k_n}(\cdot \mid s,a)\|_2 \geq \Delta$. We will make $s,a$ implicit in our notation except in $\prob_j(\cdot \mid s,a)$. Let $c_{n,i} := N(n,i,s,a)$, $\textbf{w}_{n,i} := \textbf{N}(n,i,s,a, \cdot)$. Recall that we have two nested partitions: (1) of the entire trajectory into the two $\Omega_i$ and (2) of each segment $\Omega_i$ into $G$ blocks. Finally, define $\dist_{1,(s,a)}$ as below, suppressing $m$ and $n$. Note that $\dist_1(m,n)$ is the maximum of $\dist_{1,(s,a)}$ over all $(s,a) \in \Freq_\beta$, for the given two trajectories $m$ and $n$.
\begin{align*}
\dist_{1, (s,a)}&:= \left[ \left( (\hat{\prob}_{n,1}(\cdot \mid s,a) - \hat{\prob}_{m,1}(\cdot \mid s,a))^T\textbf{V}_{s,a} \right) \left( (\hat{\prob}_{n,2}(\cdot \mid s,a) - \hat{\prob}_{m,2}(\cdot \mid s,a))^T V_{s,a} \right)^T \right]
\end{align*}

We want to show that this is close to $\|\Delta_{m,n}(s,a)\|_2^2$ for the $(s,a)$ pairs that we search over, where
$$\Delta_{m,n}(s,a) = \prob_{k_m}(\cdot \mid s,a) - \prob_{k_n}(\cdot \mid s,a)$$

Assume the lemma below for now, we prove it in the next subsection.
\begin{restatable}{lemma}{ClusteringStatDeltaBound}\label{lem:clustering-stat-delta-mn-bound}
We claim that there is a universal constant $C_1$ so that for any $(s,a)$ with $d_{min}(s,a) \geq \alpha/3$, with probability at least $1-\delta$,
\begin{align*}
    \left\vert \dist_{1, (s,a)} - \left\Vert \Delta_{m,n}(s,a) \right\Vert_2^2 \right\vert &\leq C_1\left(\sqrt{\frac{K+\log(1/\delta)}{G\alpha}} \right) + 4\epsilon_{sub}(\delta/2)
\end{align*}
whenever $T_n \geq \Omega\left( Gt_{mix}\log(G/\delta)\log(1/\alpha) \right)$ and $G \geq \Omega\left( \frac{\log(1/\delta)}{\alpha^2}\right)$. Here, $\epsilon_{sub}(\delta)$ is the high probability bound on $\|\prob_{j}(\cdot \mid s,a) - \textbf{V}_{s,a}\textbf{V}_{s,a}^T\prob_{j}(\cdot \mid s,a)\|_2$ with $j=k_n,k_m$, from Theorem~\ref{thm:subspace-est} (satisfied with probability $> 1-\delta$).
\end{restatable}

We now set $G = \left( \frac{T_n}{t_{mix}}\right)^{\frac{2}{3}}$. Then a sufficient condition on $T_n$ to meet the conditions of the lemma is $T_n = \Omega(t_{mix}\log^4(1/\delta)/\alpha^3)$, under which, with probability at lest $1-\delta$, we have the following bound for $(s,a)$ with $d_{min}(s,a) \geq \alpha/3$.

\begin{align*}
    \left\vert \dist_{1, (s,a)} - \left\Vert \Delta_{m,n}(s,a) \right\Vert_2^2 \right\vert \leq O\left(\sqrt{\frac{K\log(1/\delta)}{\alpha}}\left( \frac{t_{mix}}{T_n} \right)^{\frac{1}{3}} \right) + 4\epsilon_{sub}(\delta/2) \numberthis \label{eqn:single-s-a-dist-bound}
\end{align*}

It is now easy to see that the first term on the right-hand side is less than $\Delta^2/8$ when $T_n = \Omega\left(K^{3/2}t_{mix} \frac{\log^{3/2}(1/\delta)}{\Delta^6\alpha^{3/2}}\right)$ and $T_n = \Omega(t_{mix}\log^4(1/\delta)/\alpha^3)$. We can combine these to have the guarantee that the first term on the right-hand side is less $\Delta^2/8$ with probability at least $1-\delta$ when $T_n =  \Omega\left( K^{3/2}t_{mix}\frac{\log^4(1/\delta)}{\Delta^6\alpha^{3}}\right)$. 

Now note that if $\beta \geq \alpha/3$, then a separating state action pair always lies in $\Freq_\beta$ and thus, the maximum over the $\left\Vert \Delta_{m,n}(s,a) \right\Vert_2^2$ values corresponding to $\Freq_\beta$ is in fact either $0$ if $k_m = k_n$ or larger than $\Delta^2$ if $k_m \neq k_n$. So, if $\epsilon_{sub}(\delta/2) \leq \Delta^2/32$ and for each of the $(s,a)$ pairs, the first term on the right-hand side of inequality~\ref{eqn:single-s-a-dist-bound} is less than $\Delta^2/8$, then our distance estimate $\dist_1(m,n)$ is on the right side of any threshold as long as $\Delta^2/4 \leq \tau \leq \Delta^2/2$. That is, the distance estimate is then less than the threshold if $k_m=k_n$, and larger than it if $k_m \neq k_n$.

Note that upon choosing an occurrence threshold of order $\alpha$, we will have at most $O(1/\alpha)$ many $(s,a)$ pairs in $\Freq_\beta$ to maximize $\dist_{1, (s,a)}$ over to get $\dist_1(m,n)$. By applying a union bound over all $(s,a)$ pairs in $\Freq_\beta$ and using the conclusion of the previous paragraph, we correctly determine if $k_m = k_n$ with probability $1-\delta$ for $T_n =  \Omega\left( K^{3/2}t_{mix}\frac{\log^4(1/(\alpha\delta))}{\Delta^6\alpha^{3}}\right)$, as long as $\epsilon_{sub}(\delta/2) \leq \Delta^2/32$ and $\Delta^2/4 \leq \tau \leq \Delta^2/2$.

By applying a union bound over incorrectly deciding whether or not $k_m = k_n$ for any of the $N_{clust}(N_{clust}-1)/2$ pairs, we get that we can recover the true clusters with probability at least $1-\delta$ for $T_n =  \Omega\left( K^{3/2}t_{mix}\frac{\log^4(N_{clust}/(\alpha\delta))}{\Delta^6\alpha^{3}}\right)$, whenever $\epsilon_{sub} \leq \Delta^2/32$ and as long as $\epsilon_{sub}(\delta/2) \leq \Delta^2/32$ and $\Delta^2/4 \leq \tau \leq \Delta^2/2$.
\end{proof}

\subsection{Proof of Lemma~\ref{lem:clustering-stat-delta-mn-bound}}

We recall the statement of the lemma.
\ClusteringStatDeltaBound*

\textbf{Notation:} We say $c_{n,i} = N(n,i,s,a)$ as in the statement of the lemma and $\textbf{w}_{n,i} = \textbf{N}(n,i,s,a, \cdot)$. Let the joint distribution of the observations over the pair of trajectories $(m,n)$ be $\chi$. This means that $\chi$ is the product of the joint distribution of the observations over the trajectory $m$ and that of the observations over the trajectory $n$, since trajectories are generated independently. Let its marginals on the segments $\Omega_i$ be $\chi_{i}$. Let the marginals on each of the $G$ single-step sub-blocks along with their next states be $\chi_{i,g}$. Denote the product distribution $\prod_g \chi_{i,g}$ by $Q_i$. Let $\G(s,a)$ denote the two sets of indices where the state-action pair $(s,a)$ is observed in trajectories $n$ and $m$. For brevity, we will abbreviate $\G(s,a)$ to $\G$. Note that the sizes of these two sets are exactly $c_{n,i}$ and $c_{m,i}$ respectively.

We first prove some preliminary lemmas.

\subsubsection{Decomposition of \texorpdfstring{$|\dist_{1, (s,a)} - \left\Vert \Delta_{m,n}(s,a) \right\Vert_2^2|$}{blah}}

\begin{lemma}\label{lem:clustering-stat-decomp}
We claim that for each fixed value of $\G(s,a)$ (abbreviated to $\G$), with probability at least $1-\delta$, the following bound holds.
\begin{equation}
    \left\vert \dist_{1, (s,a)} - \left\Vert \Delta_{m,n}(s,a) \right\Vert_2^2 \right\vert\leq \sum_{i=1}^2 2\left\Vert \bm{\Delta}_i - \E_{Q_i}[\bm{\Delta}_i \mid \G] \right\Vert_2 + 4\epsilon_{sub}(\delta) + 4\left(\max_{i} \ind_{c_{n,i} = 0} + \max_{i} \ind_{c_{m,i} = 0} \right) 
\end{equation}
Here $c_{n,i} = N(n,i,s,a)$, $\epsilon_{sub}(\delta)$ is the high probability bound on $\|\prob_{j_l}(\cdot \mid s,a) - \textbf{V}_{s,a}\textbf{V}_{s,a}^T\prob_{j_l}(\cdot \mid s,a)\|_2$ from Theorem~\ref{thm:subspace-est} (satisfied with probability $> 1-\delta$), and
$$\bm{\Delta}_i^T = (\hat{\prob}_{n,i}(\cdot \mid s,a) - \hat{\prob}_{m,i}(\cdot \mid s,a))^T\textbf{V}_{s,a}$$
\end{lemma}

\begin{remark}
In the inequality,
\begin{itemize}
    \item The first term is a concentration-type term, which will be broken into an ``independent concentration" error and a mixing error to account for the low but non-zero dependence across blocks.
    \item The second term accounts for subspace estimation error.
    \item The third term accounts for actually observing $s,a$ in our blocks.
\end{itemize}
\end{remark}
\begin{proof}

We first establish a simple inequality.
\begin{align*}
    |\dist_{1, (s,a)} - \E_{Q_1}[\bm{\Delta}_1^T \mid \G] \E_{Q_2}[\bm{\Delta}_2 \mid \G]| &= |\bm{\Delta}_1^T\bm{\Delta}_2 - \E_{Q_1}[\bm{\Delta}_1^T \mid \G] \E_{Q_2}[\bm{\Delta}_2| \mid \G]\\
    &\leq |(\bm{\Delta}_1^T - \E_{Q_1}[\bm{\Delta}_1^T \mid \G])\E_{Q_2}[\bm{\Delta}_2 \mid \G]| + |\bm{\Delta}_1^T(\bm{\Delta}_2 - \E_{Q_2}[\bm{\Delta}_2 \mid \G])|\\
    &\leq \left\Vert \bm{\Delta}_1 - \E_{Q_1}[\bm{\Delta}_1 \mid \G] \right\Vert_2 \left\Vert \E_{Q_2}[\bm{\Delta}_2 \mid \G] \right\Vert_2 + \left\Vert \bm{\Delta}_1 \right\Vert_2 \left\Vert \bm{\Delta}_2 - \E_{Q_2}[\bm{\Delta}_2 \mid \G] \right\Vert_2\\
    &\leq 2\left\Vert \bm{\Delta}_1 - \E_{Q_1}[\bm{\Delta}_1 \mid \G] \right\Vert_2  + 2\left\Vert \bm{\Delta}_2 - \E_{Q_2}[\bm{\Delta}_2 \mid \G] \right\Vert_2 \numberthis \label{eqn:dist-expectation-inequality}
\end{align*}
\begin{remark}
Notice that because of this inequality, the double estimator does not impact any theoretical guarantees for exact clustering w.h.p, which is the form of the guarantees in both \citet{kakade2020meta} and \citet{poor2022mixdyn}. However, we find that using a double estimator allows for better performance in real life. This makes sense because while exact clustering doesn't need a double estimator, approximate clustering w.h.p. does depend on the expectation of the distances across pairs of trajectories. This expectation is controlled by the covariance of $\bm{\Delta}_1$ and $\bm{\Delta}_2$.
\end{remark}
We define the following quantity.
$$\textbf{diff}_{i} = \left( \ind_{c_{n,i} \neq 0}\prob_{k_m}(\cdot \mid s,a) - \ind_{c_{m,i} \neq 0}\prob_{k_n}(\cdot \mid s,a) \right)$$

Note that $\|\textbf{diff}_i\|_2 \leq 2$. Note the following expectation, which uses the dieas from equation~\ref{eqn:q-i-estimator-expectation}.
\begin{align*}
    \E_{Q_i}[\Delta_i \mid \G] &= \E_{Q_i}\left[ \textbf{V}_{s,a}^T(\hat{\prob}_{n,i}(\cdot \mid s,a) - \hat{\prob}_{m,i}(\cdot \mid s,a)) \right]\\
    &= \textbf{V}_{s,a}^T \left(\E_{Q_i}[\hat{\prob}_{n,i}(\cdot \mid s,a) \mid \G] - \E_{Q_i}[\hat{\prob}_{m,i}(\cdot \mid s,a) \mid \G]\right) \\
    &= \textbf{V}_{s,a}^T \left(\E_{Q_i}\left[\frac{\textbf{w}_{n,i}}{c_{n,i}} \ind_{c_{n,i} \neq 0} \mid \G\right] - \E_{Q_i}\left[\frac{\textbf{w}_{m,i}}{c_{m,i}} \ind_{c_{m,i} \neq 0} \mid \G\right]\right) \\
    &= \textbf{V}_{s,a}^T \left(\frac{\E_{Q_i}[\textbf{w}_{n,i} \mid \G]}{c_{n,i}} \ind_{c_{n,i} \neq 0}  - \frac{\E_{Q_i}[\textbf{w}_{m,i} \mid \G]}{c_{m,i}} \ind_{c_{m,i} \neq 0}\right)\\
    &= \textbf{V}_{s,a}^T \left(\frac{\prob_{k_n}(\cdot \mid s,a)c_{n,i}}{c_{n,i}} \ind_{c_{n,i} \neq 0}  - \frac{\prob_{k_m}(\cdot \mid s,a)c_{m,i}}{c_{m,i}} \ind_{c_{m,i} \neq 0}\right)\\
    &= \textbf{V}_{s,a}^T \left(\prob_{k_n}(\cdot \mid s,a)\ind_{c_{n,i} \neq 0}  - \prob_{k_m}(\cdot \mid s,a)\ind_{c_{m,i} \neq 0}\right)\\
    &= \textbf{V}_{s,a}^T\textbf{diff}_i
\end{align*}

We recall the following definition before proceeding to show the main inequality.
$$\Delta_{m,n}(s,a) = \prob_{k_m}(\cdot \mid s,a) - \prob_{k_n}(\cdot \mid s,a)$$
\begin{align*}
    \left\vert \E_{Q_1}[\bm{\Delta}_1^T \mid \G]\E_{Q_2}[\bm{\Delta}_2\mid \G] - \left\Vert \Delta_{m,n}(s,a) \right\Vert_2^2 \right\vert &=  \left\vert \textbf{diff}_1^T\textbf{V}_{s,a}\textbf{V}_{s,a}^T\textbf{diff}_2 - \textbf{diff}_1^T\textbf{diff}_2 \right\vert + \left\vert \textbf{diff}_1^T\textbf{diff}_2 - \left\Vert \Delta_{m,n}(s,a)\right\Vert_2^2 \right\vert\\
    \
    &\leq \left\Vert \textbf{diff}_1 \right\Vert_2\left\Vert\textbf{diff}_2 - \textbf{V}_{s,a}\textbf{V}_{s,a}^T\textbf{diff}_2 \right\Vert_2 + \left\Vert \textbf{diff}_1 - \Delta_{m,n}(s,a)\right\Vert_2\left\Vert \textbf{diff}_2 \right\Vert_2\\
    &\hspace{5ex} + \left\Vert \textbf{diff}_1 \right\Vert_2\left\Vert \textbf{diff}_2 - \Delta_{m,n}(s,a)\right\Vert_2\\
    \
    &\leq \left\Vert \textbf{diff}_1 \right\Vert_1\left\Vert\textbf{diff}_2 - \textbf{V}_{s,a}\textbf{V}_{s,a}^T\textbf{diff}_2 \right\Vert_2 + \left\Vert \textbf{diff}_1 - \Delta_{m,n}(s,a)\right\Vert_2\left\Vert \textbf{diff}_2 \right\Vert_1\\
    &\hspace{5ex} + \left\Vert \textbf{diff}_1 \right\Vert_1\left\Vert \textbf{diff}_2 - \Delta_{m,n}(s,a)\right\Vert_2\\
    \
    &\leq 2\left\Vert\textbf{diff}_2 - \textbf{V}_{s,a}\textbf{V}_{s,a}^T\textbf{diff}_2 \right\Vert_2+2\left\Vert \textbf{diff}_1 - \Delta_{m,n}(s,a)\right\Vert_2+2\left\Vert \textbf{diff}_2 - \Delta_{m,n}(s,a)\right\Vert_2\\
    \
    &\leq 2 \left\Vert \prob_{k_m}(\cdot \mid s,a) - \textbf{V}_{s,a}\textbf{V}_{s,a}^T\prob_{k_m}(\cdot \mid s,a) \right\Vert_2 \\
    &\hspace{5ex} + 2\left\Vert \prob_{k_n}(\cdot \mid s,a) - \textbf{V}_{s,a}\textbf{V}_{s,a}^T\prob_{k_n}(\cdot \mid s,a) \right\Vert_2\\
    &\hspace{5ex} + 2\left\Vert \ind_{c_{m,1} = 0}\prob_{k_m}(\cdot \mid s,a) - \ind_{c_{n,1} = 0}\prob_{k_n}(\cdot \mid s,a) \right\Vert_2\\
    &\hspace{5ex} + 2\left\Vert \ind_{c_{m,2} = 0}\prob_{k_m}(\cdot \mid s,a) - \ind_{c_{n,2} = 0}\prob_{k_n}(\cdot \mid s,a) \right\Vert_2\\
    \
    &\leq 4\epsilon_{sub}(\delta) + 2\left(\ind_{c_{m,1} = 0} \left\Vert\prob_{k_m}(\cdot \mid s,a)\right\Vert_2 + \ind_{c_{n,1} = 0} \left\Vert\prob_{k_n}(\cdot \mid s,a)\right\Vert_2\right)\\
    &\hspace{5ex} + 2\left(\ind_{c_{m,2} = 0} \left\Vert\prob_{k_m}(\cdot \mid s,a)\right\Vert_2+ \ind_{c_{n,2} = 0} \left\Vert\prob_{k_n}(\cdot \mid s,a)\right\Vert_2\right)\\
    &\leq 4\epsilon_{sub}(\delta) + 4\left(\max_{i} \ind_{c_{n,i} = 0} + \max_{i} \ind_{c_{m,i} = 0} \right)
\end{align*}

Combining this with inequality~\ref{eqn:dist-expectation-inequality}, we have the following final bound.
\begin{equation}
    \left\vert \dist_{1, (s,a)} - \left\Vert \Delta_{m,n}(s,a) \right\Vert_2^2 \right\vert\leq \sum_{i=1}^2 2\left\Vert \bm{\Delta}_i - \E_{Q_i}[\bm{\Delta}_i \mid \G] \right\Vert_2 + 4\epsilon_{sub}(\delta) + 4\left(\max_{i} \ind_{c_{n,i} = 0} + \max_{i} \ind_{c_{m,i} = 0} \right) 
\end{equation}
where we remind the reader that $c_{n,i} = N(n,i,s,a)$ and recall the definition of $\Delta_i$.
$$\bm{\Delta}_i^T = (\hat{\prob}_{n,i}(\cdot \mid s,a) - \hat{\prob}_{m,i}(\cdot \mid s,a))^T\textbf{V}_{s,a}$$
\end{proof}

\subsubsection{\textbf{Bounding the concentration-type term}} \label{sssec:clustering-concentration-type-term}

We bound the first term in the decomposition lemma (Lemma~\ref{lem:clustering-stat-decomp}) with high probability.

\begin{lemma}\label{lem:clustering-concentration-type-term-bound}
With probability at least $1-\delta$, when $T_n \geq \Omega\left(Gt_{mix}\log\left( \frac{G}{\delta}\log(1/\alpha)\right)\right)$ and $G \geq \Omega\left( \frac{\log(1/\delta)}{\alpha^2}\right)$, we have the following bound.
\begin{align*}
   \|\Delta_i - \E_{Q_i}[\Delta_i \mid \G])\|_2 \leq O\left(\sqrt{\frac{K+\log(1/\delta)}{G\alpha}} \right)
\end{align*}
\end{lemma}

\begin{proof}
Recall that the joint distribution of the observations over the pair of trajectories $(m,n)$ is $\chi$. Its marginals on the segments $\Omega_i$ are $\chi_{i}$. The marginals on each of the $G$ single-step sub-blocks is $\chi_{i,g}$. The product distribution $\prod_g \chi_{i,g}$ is $Q_i$. Recall that $\G(n,s,a)$ denotes the two sets of indices where $(s,a)$ is observed in trajectory $n$ and $m$ respectively, and the sets have sizes $c_{n,i}$ and $c_{m,i}$ respectively.

Let $\textbf{w}_{n,i,g}$ be the one hot vector of the next state if the $(i,g)$ sub-block witnesses $(s,a)$, and the zero vector otherwise. Let $c_{n,i,g}$ be the indicator of $(s,a)$ in the $(i,g)$ sub-block. Then $\textbf{w}_{n,i} = \sum_{g} \textbf{w}_{n,i,g}$ and $c_{n,i} = \sum_g c_{n,i,g}$.  

\textbf{1. Covering argument for the product distribution}

Pick a unit vector $\textbf{u} \in \R^K$ and consider the following inequality. Remember that we abbreviate $\G(n,s,a)$ to $\G$.
\begin{align*}
    |\textbf{u}^T(\Delta_i - \E_{Q_i}[\Delta_i \mid \G])| &\leq |\textbf{u}^T\textbf{V}_{s,a}(\hat{\prob}_{n,i}(\cdot \mid s,a) - \E_{Q_i}[\hat{\prob}_{n,i}(\cdot \mid s,a) \mid \G])|\\
    &\hspace{5ex} + |\textbf{u}^T\textbf{V}_{s,a}(\hat{\prob}_{m,i}(\cdot \mid s,a) - \E_{Q_i}[\hat{\prob}_{m,i}(\cdot \mid s,a) \mid \G])|\\
\end{align*}
We work with the term for trajectory $n$, WLOG. Any bounds thus obtained will also apply to trajectory $m$. Notice the following equation.
$$|\textbf{u}^T\textbf{V}_{s,a}^T(\hat{\prob}_{n,i}(\cdot \mid s,a) - \E_{Q_i}[\hat{\prob}_{n,i}(\cdot \mid s,a) \mid \G])| = \left\vert\frac{1}{c_{n,i}}\sum_{g\in \G(n,s,a)} \left(\textbf{u}^T\textbf{V}_{s,a}^T\textbf{w}_{n,i,g} - \E_{Q_i}[\textbf{u}^T\textbf{V}_{s,a}^T\textbf{w}_{n,i,g} \mid \G])\right)\right\vert$$

Note that $|\textbf{u}^T\textbf{V}_{s,a}^T\textbf{w}_{n,i,g}| \leq \|\textbf{u}\|_2\|\textbf{V}_{s,a}^T\textbf{w}_{n,i,g}\|_2 \leq 1$. Note that conditioned on the set of $(s,a)$ observations in trajectory $n$, the next states are independent under the product distribution $Q_i$ (but not under $\chi_i$, of course). Now, using the conditional version of Hoeffding's inequality from Lemma~\ref{lem:conditional-hoeffdings}, we get the following bound.

$$\prob_{Q_i} \left( \left\vert \frac{1}{c_{n,i}}\sum_{g\in \G(n,s,a)} \left(\textbf{u}^T\textbf{V}_{s,a}^T\textbf{w}_{n,i,g} - \E_{Q_i}[\textbf{u}^T\textbf{V}_{s,a}^T\textbf{w}_{n,i,g} \mid \G])\right) \right\vert > \frac{\epsilon}{8} \middle\vert \G \right) \leq 2 e^{-\frac{\epsilon^2c_{n,i}}{32}}$$

Note that if $X \leq Y+Z$, then $\prob(X > \frac{\epsilon}{4}) \leq \prob(Y> \frac{\epsilon}{8}) + \prob(Z>\frac{\epsilon}{8})$ by a union bound. We apply this to the inequalities above with $X = |\textbf{u}^T(\Delta_i - \E_{Q_i}[\Delta_i])|$ to get the following concentration inequality.

$$\prob_{Q_i} \left(|\textbf{u}^T(\Delta_i - \E_{Q_i}[\Delta_i \mid \G])| > \frac{\epsilon}{4} \mid \G \right) \leq 2 e^{-\frac{\epsilon^2c_{n,i}}{32}} + 2 e^{-\frac{\epsilon^2c_{n,i}}{32}} = 4 e^{-\frac{\epsilon^2c_{n,i}}{32}}$$

Consider a covering of $\Sph^{K-1}$ by balls of radius $1/4$. We will need at most $12^K$ such balls. Call the set of their centers $C$. We know that for any vector $\textbf{v}$, the following holds.
$$\sup_{\|\textbf{u}\|_2 \leq 1} \textbf{u}^T\textbf{v} = \|\textbf{v}\|_2 \leq 2 \sup_{\textbf{u} \in C} \textbf{u}^T\textbf{v}$$

We use this to arrive at the concentration inequality below.
\begin{align*}
    \prob_{Q_i} \left(\|\Delta_i - \E_{Q_i}[\Delta_i \mid \G])\|_2 > \frac{\epsilon}{2} \mid \G \right) &\leq \prob_{Q_i}\left(\exists \textbf{u} \in C; |\textbf{u}^T(\Delta_i - \E_{Q_i}[\Delta_i \mid \G])| > \frac{\epsilon}{4} \mid \G \right)\\
    &\leq \sum_{\textbf{u} \in C} \prob_{Q_i}\left(|\textbf{u}^T(\Delta_i - \E_{Q_i}[\Delta_i \mid \G])| > \frac{\epsilon}{4} \mid \G \right)\\
    &< 4*12^{K}*e^{-\frac{\epsilon^2c_{n,i}}{32}}
\end{align*}

\textbf{3. Accounting for non-independence (mixing error)}

We know that we can bound the difference in the probability of any event $E$ between $\chi_i$ and $Q_i$ by applying Lemma~\ref{lem:blocking-technique} to the function $h = \ind_E$ with $n=G$ and $C=1$ as we have before, giving us the following inequality.
\begin{align*}
    \prob_{\chi_i}\left(\|\Delta_i - \E_{Q_i}[\Delta_i \mid \G])\|_2 > \frac{\epsilon}{2} \right) &\leq \prob_{Q_i}\left(\|\Delta_i - \E_{Q_i}[\Delta_i \mid \G])\|_2 > \frac{\epsilon}{2} \right) + \frac{\delta}{2} + 4G\left(\frac{1}{4} \right)^{\frac{T_n}{8Gt_{mix}}}\\
    &\leq 4*12^{K}*e^{-\frac{\epsilon^2Gd_{min}(s,a)}{128}} + \frac{\delta}{2} + 4G\left(\frac{1}{4} \right)^{\frac{T_n}{8Gt_{mix}}}
\end{align*}

We know that both terms are less than $\frac{\delta}{4}$ when $T_n \geq \Omega\left(Gt_{mix}\log\left( \frac{G}{\delta}\right)\right)$ and $G \geq \Omega\left(\frac{K+\log(1/\delta)}{\epsilon^2\alpha}\right)$, since $d_{min}(s,a) \geq \alpha/3$. We thus have the following bound with probability at least $1-\delta$, when $T_n \geq \Omega\left(Gt_{mix}\log\left( \frac{G}{\delta}\right)\log(1/\alpha)\right)$ and $G \geq \Omega\left( \frac{\log(1/\delta)}{\alpha^2}\right)$.

\begin{align*}
   \|\Delta_i - \E_{Q_i}[\Delta_i \mid \G])\|_2 \leq O\left(\sqrt{\frac{K+\log(1/\delta)}{G\alpha}} \right)
\end{align*}

\end{proof}

\subsubsection{\textbf{Bounding the probability of not observing} $s,a$}

We bound the third term in the decomposition lemma (Lemma~\ref{lem:clustering-stat-decomp}) with high probability. We first need an auxiliary lemma for this.

\begin{lemma}\label{lem:clustering-not-observing-s-a-bound}
For $T_n \geq \Omega\left(Gt_{mix}\log(1/\alpha)\right)$, we have the following bound. 

$$ \prob(c_{n,i} = 0) \leq \left(1-\frac{d_{min}(s,a)}{2}\right)^{G} + 4G\left(\frac{1}{4} \right)^{\frac{T_n}{8Gt_{mix}}}$$

\end{lemma}

\begin{remark}
Again, we can think of this sum as a bound on the probability of not observing $s,a$ in the blocks if they were independent (the first term) versus a mixing error between blocks to account for their non-independence (the second term).
\end{remark}

\begin{proof}

Recall that the joint distribution of the observations over the pair of trajectories $(m,n)$ is $\chi$. Its marginals on the segments $\Omega_i$ are $\chi_{i}$. The marginals on each of the $G$ single-step sub-blocks is $\chi_{i,g}$. The product distribution $\prod_g \chi_{i,g}$ is $Q_i$. Recall that $\G(n,s,a)$ denotes the two sets of indices where $(s,a)$ is observed in trajectory $n$ and $m$ respectively, and the sets have sizes $c_{n,i}$ and $c_{m,i}$ respectively.

Remember that $\textbf{w}_{n,i,g}$ is the one hot vector of the next state if the $(i,g)$ sub-block witnesses $(s,a)$, and the zero vector otherwise, and that $c_{n,i,g}$ is the indicator of $(s,a)$ in the $(i,g)$ sub-block. Also recall that then $\textbf{w}_{n,i} = \sum_{g} \textbf{w}_{n,i,g}$ and $c_{n,i} = \sum_g c_{n,i,g}$.  

Define $h := \prod_{g = 1}^{G} (1-c_{n,i,g})$. Under any distribution $Q$ over these sub-blocks, $\E_Q h$ is the probability of not observing $s,a$ in any of them. Let $d_{i,g,n}$ be the distribution of state-action pairs at the first observation of sub-block $(i,g)$. Let $d_{k_n}(\cdot, \cdot)$ be the stationary distribution under label $k_n$ for state-action pairs. We use Lemma~\ref{lem:blocking-technique} with $h$ as above, $C=1$, $n=G$ and $a_n = \frac{T_n}{8G}$ to note the following chain of inequalities.

\begin{align*}
    \prob(c_{n,i} = 0) &= \E_{\chi_i} h\\
    &\leq \E_{Q_i} h + \left\vert \E_{Q_i} h - \E_{\chi_i} h \right\vert\\
    &\leq \left( \prod_{g=1}^{G} \E_{Q_i}(1-c_{n,i,g})\right) + 4G\lambda^{\frac{T_n}{8G}}\\
    &\leq \left( \prod_{g=1}^{G} (1-d_{k_n}(s,a) + TV(d_{i,g,n}, d_{k_n})\right) + 4G\lambda^{\frac{T_n}{8G}}\\
    &\leq \left( \prod_{g=1}^{G} (1-d_{k_n}(s,a) + 4\lambda^{\frac{T_n}{8G}})\right) + 4G\lambda^{\frac{T_n}{8G}}\\
    &= \left( 1-d_{k_n}(s,a) + 4\lambda^{\frac{T_n}{8G}}\right)^{G} + 4G\lambda^{\frac{T_n}{8G}}\\
    &\leq \left( 1-\frac{d_{k_n}(s,a)}{2}\right)^{G} + 4G\lambda^{\frac{T_n}{8G}}\\
    &\leq \left( 1-\frac{d_{min}(s,a)}{2}\right)^{G} + 4G\lambda^{\frac{T_n}{8G}}\\
\end{align*}

where the inequality in the second to last line holds for $T_n \geq \Omega\left(Gt_{mix}\log(1/\alpha)\right) \geq \Omega\left( Gt_{mix}\log(1/d_{min}(s,a))\right)$.

\end{proof}

From the above lemma, the following corollary immediately follows by getting conditions to bound each term on the right hand side by $\delta/2$, upon also noting that $-\log(1-x) \geq x$, so $\log\left(\frac{1}{1-\alpha/2}\right) \geq \alpha/2$.

\begin{corollary}\label{cor:clustering-not-observing-s-a-bound}
For $T_n \geq \Omega\left( Gt_{mix}\log(G/\delta)\log(1/\alpha) \right)$ and $G \geq \Omega\left( \frac{\log(1/\delta)}{\alpha}\right)$, we have with probability at least $1-\delta$ that

$$4\left(\max_{i} \ind_{c_{n,i} = 0} + \max_{i} \ind_{c_{m,i} = 0} \right) = 0$$
\end{corollary}

\subsubsection{\textbf{Combining the bounds}}

We finally combine these lemmas to prove Lemma~\ref{lem:clustering-stat-delta-mn-bound} -- the lemma that this section was dedicated to. The conditions of the lemmas combine to ask that $T_n \geq \Omega\left( Gt_{mix}\log(G/\delta)\log(1/\alpha) \right)$ and $G \geq \Omega\left( \frac{\log(1/\delta)}{\alpha^2}\right)$.
\begin{proof}[Proof of Lemma~\ref{lem:clustering-stat-delta-mn-bound}]
Combining the decomposition from Lemma~\ref{lem:clustering-stat-decomp} with the bounds in Lemma~\ref{lem:clustering-concentration-type-term-bound} and Corollary~\ref{cor:clustering-not-observing-s-a-bound}, we conclude using union bounds on the low probability events that we are excluding that there is a universal constant $C_1$ so that with probability at least $1-\delta$,
\begin{align*}
    \left\vert \dist_{1, (s,a)} - \left\Vert \Delta_{m,n}(s,a) \right\Vert_2^2 \right\vert &\leq C_1\left(\sqrt{\frac{K+\log(1/\delta)}{G\alpha}} \right) + 4\epsilon_{sub}(\delta/2)
\end{align*}
whenever $T_n \geq \Omega\left( Gt_{mix}\log(G/\delta)\log(1/\alpha) \right)$ and $G \geq \Omega\left( \frac{\log(1/\delta)}{\alpha^2}\right)$. 

\end{proof}

\newpage
\section{Guarantees for one step of the EM Algorithm for mixtures of MDPs}\label{sec:em-guarantee}

Remember that the M-step is just the model estimation step, so Theorem~\ref{thm:model-est} provides guarantees for that. We also have the following guarantees for the E-step of hard EM.

\begin{restatable}{thm}{EMGuarantee}\label{thm:em-guarantee}
Consider any $(s,a)$ with $d_{min}(s,a) \geq \alpha/3$ where model estimation accuracy is $\epsilon$ with $\epsilon \leq \min(\Delta/4, \Delta^2g_{min}/64)$ where $g_{min}$ is the least non-zero value of $\prob_k(s' \mid s,a)$ across $k, s'$. Using log-likelihood ratios of transitions of all such $(s,a)$ pairs, we can classify any set of $N$ new trajectories with probability $1-\delta$ if it has length $T_n = \Omega(t_{mix}\log^4(N/\delta)\log^3(1/f_{min})/\alpha^3\Delta^3)$.
\end{restatable}

\begin{remark}
The dependence on $g_{min}$ is unavoidable. For example, if the estimate for the models was only off at the value of $k,s'$ attaining $g_{min}$ and our estimate for $g_{min}$ was $\hat{\prob}_k(s' \mid s,a) = 0$, then no trajectory from label $k$ witnessing $s'$ will get correctly classified. This event will happen roughly with probability $g_{min}$, up to a mixing error, and $g_{min}$ cannot be made less than some arbitrary $\delta$ chosen to bound the probability of all undesirable events.
\end{remark}

\begin{proof}
We are inspired by the lower bound obtained in Lemma 1 of \citet{wongshen1995inequalities} for obtaining our sample complexity bounds.  Consider a separating state-action pair $s,a$.  We first establish Hellinger distance lower bounds between the distributions $\hat{\prob}_k(\cdot \mid s,a)$ and $\hat{\prob}_l(\cdot \mid s,a)$. Notice that 
$$TV(\hat{\prob}_k(\cdot \mid s,a), \prob_k(\cdot \mid s,a)) = \frac{1}{2}\|\hat{\prob}_k(\cdot \mid s,a) - {\prob}_k(\cdot \mid s,a)\|_1 \leq \epsilon/2 \leq \Delta/4$$
The same holds for $l$ as well. Combining the latter with $\|{\prob}_k(\cdot \mid s,a) - {\prob}_l(\cdot \mid s,a)\|_1 \geq \|{\prob}_k(\cdot \mid s,a) - {\prob}_l(\cdot \mid s,a)\|_2 \geq \Delta$ and using the inequality $H(P,Q) \geq TV(P,Q)/\sqrt{2}$, we get the following bound.
$$H({\prob}_k(\cdot \mid s,a), \hat{\prob}_l(\cdot \mid s,a)) \geq \frac{1}{\sqrt{2}}TV(\hat{\prob}_k(\cdot \mid s,a), \hat{\prob}_l(\cdot \mid s,a)) \geq \frac{\Delta}{4\sqrt{2}}$$

We now recall notation from the previous section. Again, we modify notation slightly, in a natural way. Let $\chi_{n}$ be the joint distribution of observations recorded in trajectory $n$, with their marginals on each single-element sub-block being $\chi_{n,g}$. Let $Q_{n}$ be the product distribution $Q_n = \prod_{n,g} \chi_{n,g}$. Let $\G(n,s,a)$ be the set of sub-blocks $(n,g)$ in which $(s,a)$ is observed in trajectory $n$. Let $c_n$ be the size of this set. We have the following lemma.

\begin{lemma}\label{lem:likelihood-ratio-bound}
Let the random variables for the next states following each $(s,a)$ observation given by $S_1, S_2, \dots S_{c_n}$ and let the true label be $k_n = k$. Then for any $l \neq k$, consider the likelihood ratio over next state transitions from $(s,a)$.
$$LR_n(s,a) = \prod_{i=1}^{c_n} \frac{\hat{\prob}_k(S_i \mid s,a)}{\hat{\prob}_l(S_i \mid s,a)}$$
We claim that $LR_n(s,a) >0$ with probability at least $1-\delta$ for  $T_n \geq \Omega\left(Gt_{mix}\log\left( \frac{G}{\delta}\right)\log(1/\alpha)\right)$ and $G \geq \Omega\left( \frac{\log(1/f_{min})\log(1/\delta)}{\alpha^2\Delta^2}\right)$.

\end{lemma}

Just like in the proof of Theorem~\ref{thm:clustering}, now set $G = \left( \frac{T_n}{t_{mix}}\right)^{\frac{2}{3}}$. Then a sufficient condition on $T_n$ to meet the conditions of the lemma is $T_n = \Omega(t_{mix}\log^4(1/\delta)\log^3(1/f_{min})/\alpha^3\Delta^3)$.

Now remember that upon choosing an occurrence threshold $\beta$ of order $\alpha$, we will have at most $O(1/\alpha)$ many $(s,a)$ pairs in $\Freq_\beta$. By applying a union bound over all $(s,a)$ pairs in $\Freq_\beta$, we get that with probability $1-\delta$, we get that the sum of the log-likelihood ratios of next-state transitions starting in $\Freq_\beta$ between the true label's model estimate and any other label's model estimate is positive whenever $T_n = \Omega(t_{mix}\log^4(1/\delta)\log^3(1/f_{min})/\alpha^3\Delta^3)$.

We now take another union bound over the $N$ new trajectories to get that we can exactly classify all of them with probability at least $1-\delta$ whenever $T_n \geq \Omega(t_{mix}\log^4(N/\delta)\log^3(1/f_{min})/\alpha^3\Delta^3)$.

\subsection{Proof of Lemma~\ref{lem:likelihood-ratio-bound}}

We first perform a computation analogous to Lemma 1 in \citet{wongshen1995inequalities}. Let $D_1 = {\prob}_k(\cdot \mid s,a)$, $D_2 = {\prob}_l(\cdot \mid s,a)$, $\hat{D}_1 = \hat{\prob}_k(\cdot \mid s,a)$, $\hat{D}_2 = \hat{\prob}_l(\cdot \mid s,a)$. Fix $b>0$. We use the conditional Markov inequality and the fact that conditioned on $\G(n,s,a)$ and under the product distribution $\hat{Q}_n$, the Hellinger distance between the next-state distributions at any $(s,a)$ observation is $H(\hat{D}_1, \hat{D}_2)$, which satisfies $H(\hat{D}_1, \hat{D}_2) \geq \Delta/4\sqrt{2}$. This is crucially due to the independence and the fact that we are fixing $\G(n,s,a)$ by conditioning on it. As usual, abbreviate $\G(n,s,a)$ to $\G$ for brevity.

\begin{align*}
    \prob_{{Q}_n}(LR_n(s,a) \leq e^{c_nb/2} \mid \G) &= \prob_{{Q}_n}\left( \prod_{i=1}^{c_n} \left(\frac{\hat{D}_2(S_i)}{\hat{D}_1(S_i)}\right)^{1/2} \geq e^{-c_nb/2} \middle\vert \G\right)\\
    &\leq e^{c_nb/2}\left(\E_{{Q}_n}\left[ \left(\frac{\hat{D}_2(S_i)}{\hat{D}_1(S_i)}\right)^{1/2} \middle\vert \G \right]\right)^{c_n}\\
    &= e^{c_nb/2}\left(\E_{D_1}\left[ \left(\frac{\hat{D}_2(S_i)}{\hat{D}_1(S_i)}\right)^{1/2} \right]\right)^{c_n}\\
    &= e^{c_nb/2}\left(\E_{D_1}\left[ \left(\frac{{D}_1(S_i)}{\hat{D}_1(S_i)}\right)^{1/2}\left(\frac{\hat{D}_2(S_i)}{{D}_1(S_i)}\right)^{1/2} \right]\right)^{c_n}\\
    &\leq e^{c_nb/2}\left(\E_{D_1}\left[ \left(1+\Delta^2/64\right)^{1/2}\left(\frac{\hat{D}_2(S_i)}{{D}_1(S_i)}\right)^{1/2} \right]\right)^{c_n}\\
    &= e^{c_nb/2}\left(1+\Delta^2/64\right)^{c_n/2}\left(1 - \frac{H(D_1,D_2)^2}{2}\right)^{c_n}\\
    &\leq e^{c_nb/2}\left(1-\Delta^2/128\right)^{c_n/2}\\
    &\leq e^{c_nb/2}e^{-c_n\Delta^2/128}
\end{align*}

Setting $b = \Delta^2/256$, we get that $\prob_{{Q}_n}(LR_n(s,a) \leq e^{c_n\Delta^2/256} \mid \G) \leq e^{-c_n\Delta^2/256}$. Now by following a very similar computation to that in point 2 in section~\ref{sssec:clustering-concentration-type-term}, we get that for $T_n \geq \Omega(Gt_{mix}\log(1/\alpha))$ and $G \geq \Omega\left(\frac{\log(1/\delta)}{\alpha^2}\right)$, $c_n \geq Gd_{min}(s,a)/4$ with probability at least $1-\delta/2$. That is, for such $T_n$ and $G$,
$$\prob_{Q_n}(LR_n(s,a) \leq e^{Gd_{min}(s,a)\Delta^2/512} \mid \G) \leq \prob_{Q_n}(LR_n(s,a) \leq e^{c_n\Delta^2/128} \mid \G) \leq e^{-Gd_{min}(s,a)\Delta^2/512} + \frac{\delta}{2}$$

Since this holds for any value of $\G = \G(n,s,a)$, we can say that with probability at least $1-\delta$, for $T_n \geq \Omega(Gt_{mix}\log(1/\alpha))$ and $G \geq \Omega\left(\frac{\log(1/\delta)}{\alpha^2}\right)$, $c_n \geq Gd_{min}(s,a)/4$, we have the following bound.

$$\prob_{Q_n}(LR_n(s,a) \leq e^{Gd_{min}(s,a)\Delta^2/512}) \leq e^{-Gd_{min}(s,a)\Delta^2/512}  + \frac{\delta}{2}$$

After following a computation very similar to that in point 3 of section~\ref{sssec:clustering-concentration-type-term}, we get that for $T_n \geq \Omega\left(Gt_{mix}\log\left( \frac{G}{\delta}\right)\log(1/\alpha)\right)$ and $G \geq \Omega\left( \frac{\log(1/\delta)}{\alpha^2\Delta^2}\right)$, 

$$\prob_\chi(LR_n(s,a) \leq e^{Gd_{min}(s,a)\Delta^2/512}) \leq \delta$$

Note that we want $e^{Gd_{min}(s,a)\Delta^2/512} \geq f_l/f_k$, in which case it suffices to ask $e^{Gd_{min}(s,a)\Delta^2/512} \geq 1/f_{min}$. Combining this with earlier conditions, for $G \geq \Omega\left(\frac{\log(1/\delta)\log(1/f_{min})}{\alpha^2\Delta^2}\right)$ and $T_n \geq \Omega\left(Gt_{mix}\log\left( \frac{G}{\delta}\right)\log(1/\alpha)\right)$, 

$$\prob_\chi\left(\frac{f_k}{f_l}LR_n(s,a) \leq 1\right) \leq \delta$$

\end{proof}
\newpage
\section{Proof of Theorem~\ref{thm:model-est}}\label{sec:model-est}

\ModelEstimation*

\begin{proof}

The proof is quite straightforward and employs the techniques used so far, especially those used in section~\ref{sssec:clustering-concentration-type-term}.
Let $k$ be the (now known) label that we're working with.

We modify previous notation a bit for this proof. For brevity of notation, we denote by $c_{n, g}$ the indicator variable for observing $(s,a)$ in the $g^{th}$ single-step sub-block of the trajectory $n$. Denote by $\textbf{w}_{n,g}$ one-hot vector of the next state observed if the currect state-action pair is $(s,a)$, and set it to the zero-vector otherwise. Note that $\sum_{g} c_{n,g} = N(n,s,a)$ and $\sum_g \textbf{w}_{n,g} = \textbf{N}(n,s,a, \cdot)$. We denote the set of indices $(n,g)$ of all $s,a$ observations that come from label $k$ (across the $GN_{clust}$ observations recorded) by $\N(s,a,k)$. Let the size of this set be $N(s,a,k)$. Note that $N(s,a,k) = \sum_{n \in \C_k} N(n,s,a) = \sum_{n,g} c_{n,g}$. Also note the following alternate expression for $\hat{\prob}_k(\cdot \mid s,a)$.
\begin{align*}
    \hat{\prob}_k(\cdot \mid s,a) := \frac{\sum_{(n,g) \in \N(s,a,k)} \textbf{w}_{n,g}}{\sum_{(n,g) \in \N(s,a,k)} c_{n,g}} \ind_{N(s,a,k) \neq 0} =  \frac{\sum_{(n,g) \in \N(s,a,k)} \textbf{w}_{n,g}}{N(s,a,k)} \ind_{N(s,a,k) \neq 0} \numberthis \label{eqn:model-estimate-alternate-form}
\end{align*}

Let $\chi_{n}$ be the joint distribution of observations recorded in trajectory $n$, with their marginals on each single-element sub-block being $\chi_{n,g}$. Let $\chi$ be the joint distribution of all observations recorded across all trajectories. Since the trajectories are independent, we know that $\chi = \prod_n \chi_n$. Let $Q_g$ be the joint distribution of the observations at the $g^{th}$ sub-block. Note that this is also the marginal of the joint distribution $\chi$ on the $g^{th}$ sub-block, and since the trajectories are independent, $Q_g = \prod_n \chi_{g,n}$. Finally, denote by $Q$ the product distribution $\prod_g Q_g = \prod_g \prod_n \chi_{g,n}$. This would be the distribution if all observations recorded were independent (across sub-blocks).

\textbf{1. Concentration under the product distribution}

We have the following computation. 
\begin{align*}
    \E_{Q}[\hat{\prob}_k(\cdot \mid s,a) \mid \N(s,a,k)] &= \E_Q\left[\frac{\sum_{n \in \N_{clust}} \textbf{w}_{n}}{N(s,a,k)} \ind_{N(s,a,k) \neq 0} \middle\vert N(s,a,k) \right]\\
    &= \E\left[ \frac{\sum_{n \in \N(s,a,k)} \textbf{w}_{n}}{N(s,a,k)} \middle\vert N(s,a,k) \right]\ind_{N(s,a,k) \neq 0}\\
    &= \frac{\sum_n \E_Q[\textbf{w}_{n} \mid \N(s,a,k)]}{N(s,a,k)}\ind_{N(s,a,k) \neq 0}\\
    &= \frac{\sum_n \prob_k(\cdot \mid s,a)c_n}{N(s,a,k)}\ind_{N(s,a,k) \neq 0}\\
    &= \frac{\prob_k(\cdot \mid s,a) (\sum_n c_n)}{N(s,a,k)}\ind_{N(s,a,k) \neq 0}\\
    &= \frac{\prob_k(\cdot \mid s,a) N(s,a,k)}{N(s,a,k)}\ind_{N(s,a,k) \neq 0}\\
    &= \prob_k(\cdot \mid s,a)\ind_{N(s,a,k) \neq 0}
\end{align*}

Now we set up our covering argument. Remember that $[-1,1]^S$ is the set of all vectors $\textbf{u} \in \R^S$ with $\|u\|_\infty \leq 1$. Consider a covering of $[-1,1]^{S}$ by boxes of side length $\frac{1}{4}$ and centers lying in $[-1,1]^{S}$. We will need at most $12^S$ such boxes and if $C$ is the set of their centers, then for any vector $\textbf{v}$
$$\|\textbf{v}\|_1 = \sup_{\textbf{u}\in [-1,1]^{S-1}} |\textbf{u}^T\textbf{v}| \leq 2 \max_{\textbf{u} \in C} |\textbf{u}^T\textbf{v}| \leq 2 \|\textbf{v}\|_1$$

Also, for any $\textbf{u} \in C$, note that
\begin{align*}
    |\textbf{u}^T\hat{\prob}_{n,i}(\cdot \mid s,a)| &\leq \|\textbf{u}\|_\infty
\left\Vert \frac{\textbf{w}_{n,1}}{c_{n,1}}\right\Vert_1\\
&\leq \left\Vert \frac{\textbf{w}_{n,1}}{c_{n,1}}\right\Vert_1\\
&= 1
\end{align*}

and so $|\textbf{u}^T\E_Q[\hat{\prob}_k(\cdot \mid s,a) \mid \N(s,a,k)]| \leq \E[|\textbf{u}^T\hat{\prob}_k(\cdot \mid s,a)| \mid \N(s,a,k)] \leq 1$. Again, note that conditioned on the set of all $(s,a)$ observations recorded, the next states $\textbf{w}_{n,g}$ are all independent under the product distribution $Q$ (but not under $\chi$, of course). Recalling the expression for $\hat{\prob}_k(\cdot \mid s,a)$ from equation~\ref{eqn:model-estimate-alternate-form}, this means that we can use the conditional version of Hoeffding's inequality, giving us the following bound.
$$\prob_Q\left( \left\vert \textbf{u}^T(\hat{\prob}_k(\cdot \mid s,a) - \E_Q[\hat{\prob}_k(\cdot \mid s,a) \mid \N(s,a,k)]) \right\vert > \frac{\epsilon}{4} \middle\vert \N(s,a,k)\right) < 2e^{-\frac{\epsilon^2N(s,a,k)}{8}}$$

Doing this for all $12^{S}$ vectors $\textbf{u} \in C$, we get the following inequality.

$$\prob_Q\left( \left\Vert (\hat{\prob}_{n,i}(\cdot \mid s,a) - \E_Q[\hat{\prob}_{n,i}(\cdot \mid s,a) \mid \N(s,a,k)]) \right\Vert_1 > \frac{\epsilon}{2} \middle\vert \N(s,a,k)\right) $$

is bounded above by
\begin{align*}
    &\prob_Q\left(\exists \textbf{u} \in C; \left\vert \textbf{u}^T(\hat{\prob}_k(\cdot \mid s,a) - \E_Q[\hat{\prob}_k(\cdot \mid s,a) \mid \N(s,a,k)]) \right\vert > \frac{\epsilon}{4} \middle\vert \N(s,a,k)\right)\\
    &\leq \sum_{\textbf{u} \in C} \prob_Q\left(\left\vert \textbf{u}^T(\hat{\prob}_k(\cdot \mid s,a) - \E_Q[\hat{\prob}_k(\cdot \mid s,a) \mid \N(s,a,k)]) \right\vert > \frac{\epsilon}{4} \middle\vert \N(s,a,k)\right)\\
    &< 12^{S}*e^{-\frac{\epsilon^2N(s,a,k)}{8}}
\end{align*}

\textbf{2. Bounding $N(s,a,k)$ under the product distribution}

Now note that $N(s,a,k) = \sum_{(n,g) \in \N_{clust}\times [G]} c_{n,g}$. So, 
$$\E_Q[N(s,a,k)] = \sum_{(n,g) \in \N_{clust}\times [G]} \E_Q[c_{n,g}] = \sum_{(n,g) \in \N_{clust}\times [G]} \prob_\chi(c_{n,g} \neq 0)$$
We can show the following inequality. 
$$\prob_\chi(c_{n,g} \neq 0) = \prob_\chi(c_{n,g} \neq 0 \mid k_n=k) \prob(k_n=k) \geq \frac{d_{min}(s,a)}{2}f_{min}$$
for $T_n \geq \Omega(Gt_{mix}\log(1/\alpha))$, getting the last inequality by using a computation very similar to the one in equation~\ref{eqn:cn-i-bound}, along with the fact that $\prob(k_n = k) = f_k$. So, $\E_Q[N(s,a,k)] \geq \frac{GN_{clust}f_{min}d_{min}(s,a)}{2}$.

\begin{align*}
    \prob_Q\left(N(s,a,k) < GN_{clust}\frac{f_{min}d_{min}(s,a)}{4}\right) &= \prob_Q\left(N(s,a,k) < GN_{clust}\frac{f_{min}d_{min}(s,a)}{2} - GN_{clust}\frac{f_{min}d_{min}(s,a)}{4}\right)\\
    &\leq \prob_Q\left(N(s,a,k) < \E[N(s,a,k)] - GN_{clust}\frac{f_{min}d_{min}(s,a)}{4}\right)\\
    &= \prob_Q\left(\sum_{(n,g) \in \N_{clust}\times [G]} c_{n,g} < \E[N(s,a,k)] - GN_{clust}\frac{f_{min}d_{min}(s,a)}{4}\right)\\
    &\leq \exp\left(-\frac{f_{min}^2d_{min}(s,a)^2GN_{clust}}{8}\right)
\end{align*}

This is less than $\delta/2$ for $GN_{clust} \geq \Omega\left( \frac{\log(1/\delta)}{f_{min}^2\alpha^2}\right)$. So, with probability at least $1-\delta/2$, for $GN_{clust} \geq \Omega\left( \frac{\log(1/\delta)}{f_{min}^2\alpha^2}\right)$ and $T_n \geq \Omega(Gt_{mix}\log(1/\alpha))$, we have the following bound.

$$\prob_Q\left( \left\Vert (\hat{\prob}_{n,i}(\cdot \mid s,a) - \E_Q[\hat{\prob}_{n,i}(\cdot \mid s,a) \mid \N(s,a,k)]) \right\Vert_1 > \frac{\epsilon}{2}\right) \leq 12^Se^{-\frac{\epsilon^2GN_{clust}f_{min}d_{min}(s,a)}{128}}$$

\textbf{3. Mixing error to account for non-independence in the true joint distribution}

Note that we can think of the combined dataset as a Markov chain over the tuple of $n$ observations, with a joint distribution $\chi$ over observations. Its marginal over the $g^{th}$ single-step sub-blocks is $Q_g$ and $Q = \prod_g Q_g$. We now want to apply Lemma~\ref{lem:blocking-technique}, noting that the relevant function of this Markov chain is $\ind_E$ where $E$ is the event $\|\hat{\prob}_k(\cdot \mid s,a) - \prob_k(\cdot \mid s,a)\|_1 < \frac{\epsilon}{2}$. Clearly, in this case, $n$ from the lemma is $G$ and $C$ from the lemma is $1$. We use this to get the following bound.
$$\prob_\chi\left( \left\Vert (\hat{\prob}_{n,i}(\cdot \mid s,a) - \E_Q[\hat{\prob}_{n,i}(\cdot \mid s,a) \mid \N(s,a,k)]) \right\Vert_1 > \frac{\epsilon}{2} \right) $$
is bounded above by
\begin{align*}
    &\prob_Q\left( \left\Vert (\hat{\prob}_{n,i}(\cdot \mid s,a) - \E_Q[\hat{\prob}_{n,i}(\cdot \mid s,a) \mid \N(s,a,k)]) \right\Vert_1 > \frac{\epsilon}{2} \right) + 4G\left(\frac{1}{4} \right)^{\frac{T_n}{8Gt_{mix}}}\\
    &\leq 12^Se^{-\frac{\epsilon^2GN_{clust}f_{min}d_{min}(s,a)}{128}} + 4G\left(\frac{1}{4} \right)^{\frac{T_n}{8Gt_{mix}}}
\end{align*}

Each term is less than $\delta/4$ for $GN_{clust} \geq \Omega\left(\frac{1}{\epsilon^2f_{min}\alpha} (S+\log(\frac{1}{\delta})\right)$ and $T_n \geq \Omega(Gt_{mix}\log(G/\delta))$. So for such $G, N_{clust}, T_n$, with probability greater than $1-\delta$, 

$$\|\hat{\prob}_k(\cdot \mid s,a) - \prob_k(\cdot \mid s,a)\|_1 < \epsilon$$

Alternatively, for $GN_{clust} \geq \Omega\left( \frac{\log(1/\delta)}{f_{min}^2\alpha^2}\right)$ and $T_n \geq \Omega(Gt_{mix}\log(G/\delta)\log(1/\alpha))$, with probability greater than $1-\delta$, 

$$\|\hat{\prob}_k(\cdot \mid s,a) - \prob_k(\cdot \mid s,a)\|_1 \leq O\left(\sqrt{\frac{1}{GN_{clust}f_{min}\alpha} (S+\log(\frac{1}{\delta}))}\right)$$

Letting $G = \left( \frac{T_n}{t_{mix}}\right)^{2/3}$, for $\left( \frac{T_n}{t_{mix}}\right)^{2/3}N_{clust} \geq \Omega\left( \frac{\log(1/\delta)}{f_{min}^2\alpha^2}\right)$ and $T_n \geq \Omega(t_{mix}\log^4(1/\delta)\log^4(1/\alpha))$, with probability greater than $1-\delta$, 

$$\|\hat{\prob}_k(\cdot \mid s,a) - \prob_k(\cdot \mid s,a)\|_1 \leq O\left(\left( \frac{t_{mix}}{T_n}\right)^{1/3}\sqrt{\frac{1}{N_{clust}f_{min}\alpha} (S+\log(\frac{1}{\delta}))}\right)$$
\end{proof}

\newpage
\section{Proof of Theorem~\ref{thm:classification}}
We recall the theorem here.

\Classification*

\begin{proof}
The proof is very similar to the proof of theorem~\ref{thm:clustering}. Consider the testing of trajectory $n$. Recall that in algorithm~\ref{alg:classification}, we defined 
$$\dist_{1}(n,k) := \max_{(s,a) \in SA_{\alpha}} \left[ \left( \left(\hat{\prob}_{n,1}(\cdot \mid s,a) - \hat{\prob}_{k}(\cdot \mid s,a) \right)^T\Tilde{\textbf{V}}_{s,a} \right) \left(\left( \hat{\prob}_{n,2}(\cdot \mid s,a) - \hat{\prob}_{k}(\cdot \mid s,a) \right)^T \Tilde{\textbf{V}}_{s,a} \right)^T \right]$$ 
Let $k_n$ the label of trajectory $n$. According to our assumptions, if $k_n \neq k$, then we have an $s,a$ so that $d_{k_n}(s,a) \geq \alpha$ and $\|\prob_{k_n}(\cdot \mid s,a) - \prob_{k}(\cdot \mid s,a)\|_2 \geq \Delta$. Again, we will make $s,a$ implicit in our notation except in $\prob_j(\cdot \mid s,a)$. Let $c_{n,i} := N(n,i,s,a)$, $\textbf{w}_{n,i} := \textbf{N}(n,i,s,a, \cdot)$. Recall that we have two nested partitions: (1) of the entire trajectory into the two $\Omega_i$ and (2) of each segment $\Omega_i$ into $G$ blocks. Finally, define $\dist_{1,(s,a)}$ as below, suppressing $n$ and $k$. Note that $\dist_1(n,k)$ is the maximum of $\dist_{1,(s,a)}$ over all $(s,a) \in \Freq_\beta$, for the given trajectory $n$ and label $k$.
\begin{align*}
\dist_{1, (s,a)}&:= \left[ \left( (\hat{\prob}_{n,1}(\cdot \mid s,a) - \hat{\prob}_{k}(\cdot \mid s,a))^T\Tilde{\textbf{V}}_{s,a} \right) \left( (\hat{\prob}_{n,2}(\cdot \mid s,a) - \hat{\prob}_{k}(\cdot \mid s,a))^T\Tilde{\textbf{V}}_{s,a} \right)^T \right]
\end{align*}

We want to show that this is close to $\|\Delta_{n,k}(s,a)\|_2^2$ for the $(s,a)$ pairs that we search over, where
$$\Delta_{n,k}(s,a) = \prob_{k_n}(\cdot \mid s,a) - \prob_{k}(\cdot \mid s,a)$$

Recall that $\|\hat{\prob}_k(\cdot \mid s,a) - \prob_k(\cdot \mid s,a)\|_2 \leq \epsilon_{mod}(\delta)$ for any $1 \leq k \leq K$. Let $\textbf{M}^{true}_{s,a} = \sum_{1 \leq k \leq K} \hat{f}_{k,s,a}\prob_k(\cdot \mid s,a)\prob_k(\cdot \mid s,a)^T$. We use the fact that $\|aa^T - bb^T\| \leq (\|a\|_2+ \|b\|_2)\|a-b\|_2$ in the bound below.

\begin{align*}
    \|\textbf{M}^{true}_{s,a} - \Tilde{\textbf{M}}_{s,a}\| &\leq \sum_{1 \leq k \leq K} \hat{f}_{k,s,a} \|\prob_k(\cdot \mid s,a)\prob_k(\cdot \mid s,a)^T - \hat{\prob}_k(\cdot \mid s,a)\hat{\prob}_k(\cdot \mid s,a)^T\|\\
    &\leq \sum_{1 \leq k \leq K} \hat{f}_{k,s,a}(\|\hat{\prob}_k(\cdot \mid s,a)\|_2 + \|\prob_k(\cdot \mid s,a)\|_2) \|\hat{\prob}_k(\cdot \mid s,a) - \prob_k(\cdot \mid s,a)\|_2\\
    &\leq \sum_{1 \leq k \leq K} 2\hat{f}_{k,s,a} \|\hat{\prob}_k(\cdot \mid s,a) - \prob_k(\cdot \mid s,a)\|_2\\
    &\leq 2\epsilon_{mod}(\delta)
\end{align*}

Also note that if we redefine $\B_n$ to be the event of observing $(s,a)$ in a trajectory (instead of in both segments as in the notation in previous proofs), then $\hat{f}_{k,s,a} = \frac{\sum_n \ind_{k_n=k}\ind_{\B_n}}{\sum_n \ind_{\B_n}} \geq  \frac{\sum_n \ind_{k_n=k}\ind_{\B_n}}{N_{clust}}$. So, $\E[\hat{f}_{k,s,a}] \geq \prob(k_n = k \cap \B_n) = \prob(\B_n \mid k_n=k)\prob(k_n=k) \geq f_{min}\prob(\B_n \mid k_n=k)$. Using a computation very similar to the one leading up to inequality~\ref{eqn:cm1-cm2-bound}, we note that $\prob(\B_n \mid k_n=k) \geq d_{min}(s,a)/2$ for $T_n \geq \Omega(t_{mix}\log(1/\alpha))$. In that case, $\E[\hat{f}_{k,s,a}] \geq f_{min}d_{min}(s,a)/2 \geq f_{min}\alpha/2$. Additionally, using a standard concentration argument, $\hat{f}_{k,s,a} \geq \E[\hat{f}_{k,s,a}]/2 \geq f_{min}\alpha/4$ for $N_{clust} \geq \Omega\left(\frac{\log(1/\delta)}{f_{min}^2\alpha^2}\right) \geq \Omega\left(\frac{\log(1/\delta)}{\E[\hat{f}_{k,s,a}]^2}\right)$

We now apply Lemma 3 of \citet{poor2022mixdyn}, with $p^{(k)} = \hat{f}_{k,s,a}$, $\textbf{y}^{(k)} = \prob_k(\cdot \mid s,a)$, $\textbf{M} = \textbf{M}^{true}_{s,a}$ and $\textbf{M}_* = \textbf{M}_{s,a}$. We use the right-hand side of the bound in the lemma to get the bound below for all $1 \leq k \leq K$, which holds for a universal constant $C_2$ with probability at least $1-\delta$ whenever $N_{clust} \geq \Omega\left(\frac{\log(1/\delta)}{f_{min}^2\alpha^2}\right)$ and $T_n \geq \Omega(t_{mix}\log(1/\alpha))$.
\begin{align*}
    \|\prob_{k}(\cdot \mid s,a) - \Tilde{\textbf{V}}_{s,a}\Tilde{\textbf{V}}_{s,a}^T\prob_{k}(\cdot \mid s,a)\|_2 \leq \sqrt{\frac{2K\epsilon_{mod}(\delta)}{\hat{f}_{k,s,a}}} \leq C_2\sqrt{\frac{K\epsilon_{mod}(\delta)}{f_{min}\alpha}} \numberthis \label{eqn:epsilon-mod-bound}
\end{align*}

Assume the lemma below for now, we prove it in the next subsection.
\begin{restatable}{lemma}{ClassificationStatDeltaBound}\label{lem:classification-stat-delta-mn-bound}
We claim that there is a universal constant $C_1$ so that for any $(s,a)$ with $d_{min}(s,a) \geq \alpha/3$, with probability at least $1-\delta$,
\begin{align*}
    \left\vert \dist_{1, (s,a)} - \left\Vert \Delta_{n,k}(s,a) \right\Vert_2^2 \right\vert \leq O\left(\sqrt{\frac{K+\log(1/\delta)}{G\alpha}} \right) + 8C_2\sqrt{\frac{K\epsilon_{mod}(\delta/2)}{f_{min}\alpha}}
\end{align*}
whenever $T_n \geq \Omega\left( Gt_{mix}\log(G/\delta)\log(1/\alpha) \right)$ and $G \geq \Omega\left( \frac{\log(1/\delta)}{\alpha^2}\right)$. Here, $\epsilon_{mod}(\delta)$ is a high probability bound on $\|\prob_{k}(\cdot \mid s,a) - \Tilde{\textbf{V}}_{s,a}\Tilde{\textbf{V}}_{s,a}^T\prob_{k}(\cdot \mid s,a)\|_2$ for all $1 \leq k \leq K$ (which holds with probability at least $1-\delta)$.
\end{restatable}

We now set $G = \left( \frac{T_n}{t_{mix}}\right)^{\frac{2}{3}}$. Then a sufficient condition on $T_n$ to meet the conditions of the lemma is $T_n = \Omega(t_{mix}\log^4(1/\delta)/\alpha^3)$, under which, with probability at lest $1-\delta$, we have the following bound for $(s,a)$ with $d_{min}(s,a) \geq \alpha/3$.

\begin{align*}
    \left\vert \dist_{1, (s,a)} - \left\Vert \Delta_{n,k}(s,a) \right\Vert_2^2 \right\vert \leq O\left(\sqrt{\frac{K\log(1/\delta)}{\alpha}}\left( \frac{t_{mix}}{T_n} \right)^{\frac{1}{3}} \right) + 8C_2\sqrt{\frac{K\epsilon_{mod}(\delta/2)}{f_{min}\alpha}} \numberthis \label{eqn:single-s-a-dist-bound}
\end{align*}

It is now easy to see that the first term on the right-hand side is less than $\Delta^2/8$ when $T_n = \Omega\left(K^{3/2}t_{mix} \frac{\log^{3/2}(1/\delta)}{\Delta^6\alpha^{3/2}}\right)$ and $T_n = \Omega(t_{mix}\log^4(1/\delta)/\alpha^3)$. We can combine these to have the guarantee that the first term on the right-hand side is less $\Delta^2/8$ with probability at least $1-\delta$ when $T_n =  \Omega\left( K^{3/2}t_{mix}\frac{\log^4(1/\delta)}{\Delta^6\alpha^{3}}\right)$. 

Now note that if $\beta \geq \alpha/3$, then a separating state action pair always lies in $\Freq_\beta$ and thus, the maximum over the $\left\Vert \Delta_{n,k}(s,a) \right\Vert_2^2$ values corresponding to $\Freq_\beta$ is in fact either $0$ if $k = k_n$ or larger than $\Delta^2$ if $k \neq k_n$. So, if $8C_2\sqrt{\frac{K\epsilon_{mod}(\delta/2)}{f_{min}\alpha}} \leq \Delta^2/32$ and for each of the $(s,a)$ pairs, the first term on the right-hand side of inequality~\ref{eqn:single-s-a-dist-bound} is less than $\Delta^2/8$, then our distance estimate $\dist_1(n,k)$ is on the right side of $\Delta^2/3$. That is, the distance estimate is then less than $\Delta^2/4$ if $k=k_n$, and larger than it if $k \neq k_n$. As a consequence, the output of the $\argmin$ in algorithm~\ref{alg:classification} is $k_n$ in this situation.

Note that upon choosing an occurrence threshold of order $\alpha$, we will have at most $O(1/\alpha)$ many $(s,a)$ pairs in $\Freq_\beta$ to maximize $\dist_{1, (s,a)}$ over to get $\dist_1(n,k)$. By applying a union bound over all $(s,a)$ pairs in $\Freq_\beta$ and using the conclusion of the previous paragraph, algorithm~\ref{alg:classification} correctly predicts the label $k_n$ for trajectory $n$ with probability $1-\delta$ whenever $T_n =  \Omega\left( K^{3/2}t_{mix}\frac{\log^4(1/(\alpha\delta))}{\Delta^6\alpha^{3}}\right)$ and $8C_2\sqrt{\frac{K\epsilon_{mod}(\delta/2)}{f_{min}\alpha}} \leq \Delta^2/32$.

By applying a union bound over incorrectly predicting $k_n$ for any of the $N_{class}(N_{class}-1)/2$ pairs, we get that algorithm~\ref{alg:classification} can recover the true labels with probability at least $1-\delta$ for $T_n =  \Omega\left( K^{3/2}t_{mix}\frac{\log^4(N_{class}/(\alpha\delta))}{\Delta^6\alpha^{3}}\right)$, whenever $8C_2\sqrt{\frac{K\epsilon_{mod}(\delta/2)}{f_{min}\alpha}} \leq \Delta^2/32$.

Finally note that due to inequality~\ref{eqn:epsilon-mod-bound}, we get that algorithm~\ref{alg:classification} can recover the true labels with probability at least $1-\delta$ for $T_n =  \Omega\left( K^{3/2}t_{mix}\frac{\log^4(N_{class}/(\alpha\delta))}{\Delta^6\alpha^{3}}\right)$, whenever $\epsilon_{mod}(\delta/2) \leq \frac{C_3\Delta^4f_{min}\alpha}{K}$.

\end{proof}

\subsection{Proof of Lemma~\ref{lem:classification-stat-delta-mn-bound}}

We recall the lemma here.
\ClassificationStatDeltaBound*
\begin{proof}
The proof of this lemma is very similar to the proof of Lemma~\ref{lem:clustering-stat-delta-mn-bound}.

\textbf{Notation:} We say $c_{n,i} = N(n,i,s,a)$ as in the statement of the lemma and $\textbf{w}_{n,i} = \textbf{N}(n,i,s,a, \cdot)$. Let the joint distribution of the observations over trajectory $n$ be $\chi$. Let its marginals on the segments $\Omega_i$ be $\chi_{i}$. Let the marginals on each of the $G$ single-step sub-blocks along with their next states be $\chi_{i,g}$. Denote the product distribution $\prod_g \chi_{i,g}$ by $Q_i$. Let $\G(n,s,a)$ denote the set of indices where the state-action pair $(s,a)$ is observed in trajectory $n$. For brevity, we will abbreviate $\G(n,s,a)$ to $\G$. Note that the size of this set is exactly $c_{n,i}$.

We first prove a preliminary lemma, similar to lemma~\ref{lem:clustering-stat-decomp}.

\subsubsection{Decomposition of \texorpdfstring{$|\dist_{1, (s,a)} - \left\Vert \Delta_{n,k}(s,a) \right\Vert_2^2|$}{blah}}

\begin{lemma}\label{lem:classification-stat-decomp}
We claim that for each fixed value of $\G(s,a)$ (abbreviated to $\G$), with probability at least $1-\delta$, the following bound holds.
\begin{equation}
    \left\vert \dist_{1, (s,a)} - \left\Vert \Delta_{n,k}(s,a) \right\Vert_2^2 \right\vert\leq \sum_{i=1}^2 2\left\Vert\hat{\prob}_{n,i}(\cdot \mid s,a) - \E_{Q_i}[\hat{\prob}_{n,i}(\cdot \mid s,a) \mid \G]\right\Vert_2 + 8C_2\sqrt{\frac{K\epsilon_{mod}(\delta)}{f_{min}\alpha}} + 4\left(\max_{i} \ind_{c_{n,i} = 0} \right) 
\end{equation}
Here $c_{n,i} = N(n,i,s,a)$ and $\epsilon_{mod}(\delta)$ is a high probability bound on $\|\prob_{k}(\cdot \mid s,a) - \Tilde{\textbf{V}}_{s,a}\Tilde{\textbf{V}}_{s,a}^T\prob_{k}(\cdot \mid s,a)\|_2$ (satisfied with probability $> 1-\delta$).

\end{lemma}

\begin{remark}
In the inequality,
\begin{itemize}
    \item The first term is a concentration-type term, which will be broken into an ``independent concentration" error and a mixing error to account for the low but non-zero dependence across blocks.
    \item The second term accounts for subspace estimation error.
    \item The third term accounts for actually observing $s,a$ in our blocks.
\end{itemize}
\end{remark}
\begin{proof}

Define the following quantities.
\begin{align*}
    \bm{\Delta}_i^T &= (\hat{\prob}_{n,i}(\cdot \mid s,a) - \hat{\prob}_{k}(\cdot \mid s,a))^T\Tilde{\textbf{V}}_{s,a}\\
    \bar{\bm{\Delta}}_i^T &= (\E_{Q_i}[\hat{\prob}_{n,i}(\cdot \mid s,a) \mid \G] - {\prob}_{k}(\cdot \mid s,a))^T\Tilde{\textbf{V}}_{s,a}
\end{align*}

We first establish a simple inequality, using the fact that $|a^Tb - c^Td| \leq \|b\|_2\|a-c\|_2 + \|c\|_2\|b-d\|_2$
\begin{align*}
    |\dist_{1, (s,a)} - \bar{\bm{\Delta}}_1^T\bar{\bm{\Delta}}_2| &= |\bm{\Delta}_1^T\bm{\Delta}_2 -\bar{\bm{\Delta}}_1^T\bar{\bm{\Delta}}_2|\\
    &\leq \|\bm{\Delta}_1 - \bar{\bm{\Delta}}_1\|_2\|{\bm{\Delta}}_2\|_2 + \|\bar{\bm{\Delta}}_1^T\|_2\|\bm{\Delta}_2 - \bar{\bm{\Delta}}_2\|_2\\
    &\leq 2\|\bm{\Delta}_1 - \bar{\bm{\Delta}}_1\|_2 + 2\|\bm{\Delta}_2 - \bar{\bm{\Delta}}_2\|_2\\
    &\leq \sum_{i=1}^2 2\|\hat{\prob}_{n,i}(\cdot \mid s,a) - \E_{Q_i}[\hat{\prob}_{n,i}(\cdot \mid s,a) \mid \G]\|_2 + \sum_{i=1}^2 2 \| \hat{\prob}_{k}(\cdot \mid s,a) - {\prob}_{k}(\cdot \mid s,a)\|_2\\
    &\leq \sum_{i=1}^2 2\|\hat{\prob}_{n,i}(\cdot \mid s,a) - \E_{Q_i}[\hat{\prob}_{n,i}(\cdot \mid s,a) \mid \G]\|_2 + 4\epsilon_{mod}(\delta) \numberthis \label{eqn:classification-dist-expectation-inequality}
\end{align*}
Also note the following computation.
\begin{align*}
    \E_{Q_i}[\hat{\prob}_{n,i}(\cdot \mid s,a) \mid \G] &= \E_{Q_i}\left[\frac{\textbf{w}_{n,i}}{c_{n,i}} \ind_{c_{n,i} \neq 0} \mid \G\right]\\
    &= \frac{\E_{Q_i}[\textbf{w}_{n,i} \mid \G]}{c_{n,i}} \ind_{c_{n,i} \neq 0}\\
    &= \frac{\prob_{k_n}(\cdot \mid s,a)c_{n,i}}{c_{n,i}} \ind_{c_{n,i} \neq 0}\\
    &= \ind_{c_{n,i} \neq 0}\prob_{k_n}(\cdot \mid s,a)
\end{align*}
We define the following quantity, overloading notation from Lemma~\ref{lem:clustering-stat-decomp}.
$$\textbf{diff}_i = \ind_{c_{n,i} \neq 0}\prob_{k_n}(\cdot \mid s,a) - \prob_{k}(\cdot \mid s,a)$$
Note that $\bar{\bm{\Delta}_i} = \textbf{diff}_i^T\Tilde{\textbf{V}}_{s,a}$. We recall the following definition before proceeding to show the main inequality.
$$\Delta_{n,k}(s,a) = \prob_{k_n}(\cdot \mid s,a) - \prob_{k}(\cdot \mid s,a)$$
\begin{align*}
    \left\vert \bar{\bm{\Delta}}_1^T\bar{\bm{\Delta}}_2 - \left\Vert \Delta_{n,k}(s,a) \right\Vert_2^2 \right\vert &=  \left\vert \textbf{diff}_1^T\Tilde{\textbf{V}}_{s,a}\Tilde{\textbf{V}}_{s,a}^T\textbf{diff}_2 - \textbf{diff}_1^T\textbf{diff}_2 \right\vert + \left\vert \textbf{diff}_1^T\textbf{diff}_2 - \left\Vert \Delta_{n,k}(s,a)\right\Vert_2^2 \right\vert\\
    \
    &\leq \left\Vert \textbf{diff}_1 \right\Vert_2\left\Vert\textbf{diff}_2 - \Tilde{\textbf{V}}_{s,a}\Tilde{\textbf{V}}_{s,a}^T\textbf{diff}_2 \right\Vert_2 + \left\Vert \textbf{diff}_1 - \Delta_{n,k}(s,a)\right\Vert_2\left\Vert \textbf{diff}_2 \right\Vert_2\\
    &\hspace{5ex} + \left\Vert \textbf{diff}_1 \right\Vert_2\left\Vert \textbf{diff}_2 - \Delta_{n,k}(s,a)\right\Vert_2\\
    \
    &\leq \left\Vert \textbf{diff}_1 \right\Vert_1\left\Vert\textbf{diff}_2 - \Tilde{\textbf{V}}_{s,a}\Tilde{\textbf{V}}_{s,a}^T\textbf{diff}_2 \right\Vert_2 + \left\Vert \textbf{diff}_1 - \Delta_{n,k}(s,a)\right\Vert_2\left\Vert \textbf{diff}_2 \right\Vert_1\\
    &\hspace{5ex} + \left\Vert \textbf{diff}_1 \right\Vert_1\left\Vert \textbf{diff}_2 - \Delta_{n,k}(s,a)\right\Vert_2\\
    \
    &\leq 2\left\Vert\textbf{diff}_2 - \Tilde{\textbf{V}}_{s,a}\Tilde{\textbf{V}}_{s,a}^T\textbf{diff}_2 \right\Vert_2+2\left\Vert \textbf{diff}_1 - \Delta_{n,k}(s,a)\right\Vert_2+2\left\Vert \textbf{diff}_2 - \Delta_{n,k}(s,a)\right\Vert_2\\
    \
    &\leq 2 \left\Vert \prob_{k_n}(\cdot \mid s,a) - \Tilde{\textbf{V}}_{s,a}\Tilde{\textbf{V}}_{s,a}^T\prob_{k_n}(\cdot \mid s,a) \right\Vert_2 \\
    &\hspace{5ex} + 2\left\Vert \prob_{k}(\cdot \mid s,a) - \Tilde{\textbf{V}}_{s,a}\Tilde{\textbf{V}}_{s,a}^T\prob_{k}(\cdot \mid s,a) \right\Vert_2\\
    &\hspace{5ex} + 2\ind_{c_{n,1} = 0} \left\Vert\prob_{k_n}(\cdot \mid s,a)\right\Vert_2 + 2\ind_{c_{n,2} = 0} \left\Vert\prob_{k_n}(\cdot \mid s,a)\right\Vert_2\\
    \
    &\leq 4C_2\sqrt{\frac{K\epsilon_{mod}(\delta)}{f_{min}\alpha}} + 4\left(\max_{i} \ind_{c_{n,i} = 0} \right)
\end{align*}

Notice that $4C_2\sqrt{\frac{K\epsilon_{mod}(\delta)}{f_{min}\alpha}} \geq 4\epsilon_{mod}(\delta)$ since $\epsilon_{mod}(\delta) \leq 2$, $C_2 \geq 2$, $K \geq 1$, $f_{min}, \alpha \leq 1$. Combining this and the computation above with inequality~\ref{eqn:dist-expectation-inequality}, we have the following final bound.
\begin{equation}
    \left\vert \dist_{1, (s,a)} - \left\Vert \Delta_{n,k}(s,a) \right\Vert_2^2 \right\vert\leq \sum_{i=1}^2 2\left\Vert\hat{\prob}_{n,i}(\cdot \mid s,a) - \E_{Q_i}[\hat{\prob}_{n,i}(\cdot \mid s,a) \mid \G]\right\Vert_2 + 8C_2\sqrt{\frac{K\epsilon_{mod}(\delta)}{f_{min}\alpha}} + 4\left(\max_{i} \ind_{c_{n,i} = 0} \right) 
\end{equation}
where we remind the reader that $c_{n,i} = N(n,i,s,a)$.
\end{proof}

\subsubsection{\textbf{Bounding the concentration-type term}} \label{sssec:classification-concentration-type-term}

We bound the first term in the decomposition lemma (Lemma~\ref{lem:classification-stat-decomp}) with high probability.

\begin{lemma}\label{lem:classification-concentration-type-term-bound}
With probability at least $1-\delta$, when $T_n \geq \Omega\left(Gt_{mix}\log\left( \frac{G}{\delta}\log(1/\alpha)\right)\right)$ and $G \geq \Omega\left( \frac{\log(1/\delta)}{\alpha^2}\right)$, we have the following bound.
\begin{align*}
   \left\Vert\hat{\prob}_{n,i}(\cdot \mid s,a) - \E_{Q_i}[\hat{\prob}_{n,i}(\cdot \mid s,a) \mid \G]\right\Vert_2 \leq O\left(\sqrt{\frac{K+\log(1/\delta)}{G\alpha}} \right)
\end{align*}
\end{lemma}

\begin{proof}
The proof of this lemma is verbatim the proof of Lemma~\ref{lem:clustering-concentration-type-term-bound} after the first inequality.
\end{proof}

\subsubsection{Combining the bounds}

We reuse Corollary~\ref{cor:clustering-not-observing-s-a-bound} along with Lemma~\ref{lem:classification-concentration-type-term-bound} applied to Lemma~\ref{lem:classification-stat-decomp} to get the following bound with probability at least $1-\delta$,
\begin{align*}
    \left\vert \dist_{1, (s,a)} - \left\Vert \Delta_{n,k}(s,a) \right\Vert_2^2 \right\vert \leq O\left(\sqrt{\frac{K+\log(1/\delta)}{G\alpha}} \right) + 8C_2\sqrt{\frac{K\epsilon_{mod}(\delta)}{f_{min}\alpha}}
\end{align*}
whenever $T_n \geq \Omega\left( Gt_{mix}\log(G/\delta)\log(1/\alpha) \right)$ and $G \geq \Omega\left( \frac{\log(1/\delta)}{\alpha^2}\right)$. 

\end{proof}

\end{document}